\documentclass[reqno]{amsart}

\usepackage[T1]{fontenc}
\usepackage[utf8]{inputenc}
\usepackage{lmodern}
\usepackage{mathrsfs}  
\usepackage{hyperref}
\usepackage{graphicx}
\usepackage[english]{babel}
\usepackage{placeins}
\usepackage{amsmath}
\usepackage{amssymb}
\usepackage{graphicx}
\usepackage{dsfont}
\usepackage{nicefrac}
\usepackage{multirow}
\usepackage{float}
\usepackage{lmodern}
\usepackage{mathrsfs}  
\usepackage{mathtools}
\usepackage{emptypage}
\usepackage{amsfonts}
\usepackage{ltablex}
\usepackage{cleveref}
\usepackage{changes}
\definechangesauthor[name=Diego, color=red]{DM}
\usepackage{comment}
\usepackage[noend]{algorithmic}
\usepackage{algorithm}
\usepackage{setspace}
\usepackage{rotating}
\usepackage{tikz}
\usepackage{lineno}
\usepackage{xcolor}
\usetikzlibrary{decorations.pathreplacing}
\definecolor{color1bg}{HTML}{FA8072}
\definecolor{bblue}{HTML}{00BFFF}

\makeindex
\allowdisplaybreaks
\usetikzlibrary{arrows,calc}
\tikzset{
	>=stealth',
	help lines/.style={dashed, thick},
	axis/.style={<->},
	important line/.style={thick},
	connection/.style={thick, dotted},
}

\usetikzlibrary{shadows}
\tikzset{
	diagonal fill/.style 2 args={fill=#2, path picture={
			\fill[#1, sharp corners] (path picture bounding box.south west) -|
			(path picture bounding box.north east) -- cycle;}},
	reversed diagonal fill/.style 2 args={fill=#2, path picture={
			\fill[#1, sharp corners] (path picture bounding box.north west) |- 
			(path picture bounding box.south east) -- cycle;}}
}

\DeclareMathOperator*{\argminA}{arg\,min}
\DeclareMathOperator*{\argmaxA}{arg\,max}
\DeclareMathOperator*{\arginfA}{arg\,min}

\newtheorem{theorem}{Theorem}
\newtheorem{proposition}{Proposition}
\newtheorem{corollary}{Corollary}
\newtheorem{remark}{Remark}
\newtheorem{definition}{Definition}
\newtheorem{example}{Example}

\title[Learning the hypotheses space from data]{Learning the hypotheses space from data through a U-curve algorithm}

\author{Diego Marcondes, Adilson Simonis and Junior Barrera}

\address{Institute of Mathematics and Statistics\\
	University of S\~ao Paulo\\
	S\~ao Paulo, SP 05508-090, Brazil.
	E-mail: dmarcondes@ime.usp.br}

\begin{document}	
	
	\begin{abstract}
	This paper proposes a data-driven systematic, consistent and non-exhaustive approach to Model Selection, that is an extension of the classical agnostic PAC learning model. In this approach, learning problems are modeled not only by a hypothesis space $\mathcal{H}$, but also by a Learning Space $\mathbb{L}(\mathcal{H})$, a poset of subspaces of $\mathcal{H}$, which covers $\mathcal{H}$ and satisfies a property regarding the VC dimension of related subspaces, that is a suitable algebraic search space for Model Selection algorithms. Our main contributions are a data-driven general learning algorithm to perform implicitly regularized Model Selection on $\mathbb{L}(\mathcal{H})$ and a framework under which one can, theoretically, better estimate a target hypothesis with a given sample size by properly modeling $\mathbb{L}(\mathcal{H})$ and employing high computational power. A remarkable consequence of this approach are conditions under which a non-exhaustive search of $\mathbb{L}(\mathcal{H})$ can return an optimal solution. The results of this paper lead to a practical property of Machine Learning, that the lack of experimental data may be mitigated by a high computational capacity. In a context of continuous popularization of computational power, this property may help understand why Machine Learning has become so important, even where data is expensive and hard to get.\\
		
	\end{abstract}

	\maketitle
	
	\section{Introduction}
	\label{Introduction}
	
	The state-of-the-art in Machine Learning is a model $(\mathcal{H},\mathbb{A},\mathcal{D}_{N})$ composed by a set $\mathcal{H}$ of hypotheses $h$, which are functions from $\mathcal{X} \subset \mathbb{R}^{d}, d \geq 1$, to $\mathcal{Y} = \{0,1\}$, called hypotheses space, and a learning algorithm $\mathbb{A}(\mathcal{H},\mathcal{D}_{N})$, which processes a training sample $\mathcal{D}_{N} = \{(X_{1},Y_{1}),\dots,(X_{N},Y_{N})\}$ of a random vector $(X,Y)$, with range $\mathcal{X} \times \mathcal{Y}$ and joint probability distribution $P$, and searches $\mathcal{H}$ seeking to minimize an error measure that assesses how good each $h \in \mathcal{H}$ predicts the values of $Y$ from instances of $X$. 
	
	Let $\ell: \mathcal{Y} \times \mathcal{Y} \mapsto \mathbb{R}_{+}$ be a loss function. The error of a hypothesis $h \in \mathcal{H}$ is the expected value of the local measures $\ell(h(x),y), (x,y) \in \mathcal{X} \times \mathcal{Y}$, and can be of at least two types: the in-sample error $L_{\mathcal{D}_{N}}(h)$, which is the sample mean of $\ell(h(x),y)$ under $\mathcal{D}_{N}$, and the out-of-sample error $L(h)$, the expectation of $\ell(h(X),Y)$ under the joint distribution $P$. In this context, a target hypothesis $h^{\star} \in \mathcal{H}$ is such that its out-of-sample error is minimum in $\mathcal{H}$, i.e., $L(h^{\star}) \leq L(h), \forall h \in \mathcal{H}$. The algorithm returns $\hat{h}(\mathbb{A}) \in \mathcal{H}$ seeking to approximate a target hypothesis $h^{\star} \in \mathcal{H}$. Such learning model has an important parameter that is problem-specific: the hypotheses space $\mathcal{H}$, which has a strong impact on the generalization quality of the model, that is, the correctness of the classification of non observed instances.
	
	The fundamental result in Machine Learning is the Vapnik-Chervonenkis (VC) Theory \cite{vapnik1971uniform,vapnik1974theory,vapnik1974ordered,vapnik1974ordered2,vapnik1998,vapnik2000}, which implies that a hypotheses space $\mathcal{H}$ is PAC learnable \cite{valiant1984} if, and only if, it has finite VC dimension ($d_{VC}(\mathcal{H}) < \infty$) \cite[Theorem~6.7]{shalev2014}. This means that for any joint distribution $P$, $L_{\mathcal{D}_{N}}$ generalizes, that is, $L$ may be uniformly approximated by $L_{\mathcal{D}_{N}}$ in $\mathcal{H}$ with a given precision and great confidence if $N$ is sufficiently large. Therefore, it is possible to learn hypotheses with a finite sample, with precision and confidence dependent on the training sample size and the VC dimension, i.e., complexity, of $\mathcal{H}$. 
	
	The VC theory is general, has a structural importance to the field, and is a useful guide for modeling practical problems. However, since $N(d_{VC}(\mathcal{H}),\epsilon,\delta)$, the least $N$, given $\mathcal{H}, \epsilon$, and $\delta$, under VC theory bounds such that
	\begin{equation}
	\label{typeI}
	\mathbb{P}\Big(\sup\limits_{h \in \mathcal{H}} \Big|L(h) - L_{\mathcal{D}_{N}}(h)\Big| > \epsilon\Big) \leq \delta,
	\end{equation}
	is not a tight bound, as it is distribution-free, it is usually a meaningless quantity in real application problems. In fact, the sample size $N$ depends on data availability, which may be conditioned upon several factors such as technical difficulties and costs. Thus, parameters $(N,\epsilon,\delta)$ are usually predetermined, so the only free component to be adjusted on VC theory's bounds is the hypotheses space $\mathcal{H}$ or, more precisely, its VC dimension. This adjustment is usually performed via a data-driven selection of $\mathcal{H}$, known in the literature of the field as Model Selection (see \cite{guyon2010,raschka2018,ding2018,massart2007} for a review of Model Selection techniques). 
	
	To select $\mathcal{H}$ based on data, one might apply a combinatorial algorithm, which searches a family of candidate hypotheses spaces seeking to minimize an estimator of the out-of-sample error of the optimal hypothesis of each space. In other words, given a family $\{\mathcal{M}_{1},\dots,\mathcal{M}_{n}\}$ of hypotheses spaces, a sample of size $N$ and a consistent estimator $\hat{L}$ of the out-of-sample error, often given by an independent validation sample or cross-validation, such algorithm returns a space $\mathcal{M}$, whose estimated optimal hypothesis is somewhat the best estimator for a target hypothesis. 
	
	The VC dimension of the selected hypotheses space will be adequate, in some sense, to learn with a sample of size $N$ for this particular unknown distribution $P$. If the candidate models are nested, i.e., $\mathcal{M}_{1} \subset \cdots \subset \mathcal{M}_{n}$, then a method based on the Structured Risk Minimization (SRM) Inductive Principle may be applied to solve this problem (see \cite[Chapter~4]{vapnik2000} for more details and \cite{anguita2012} for an example). Furthermore, this problem may also be tackled by penalizing the in-sample error according to the complexity of each model, in both nested and non-nested frameworks (see \cite{massart2007} for a in-depth presentation of Model Selection by penalization and \cite{koltchinskii2001,koltchinskii2011,agarwal2011,arlot2011,bartlett2008} for results in more specific learning frameworks). Moreover, the classical problem of feature selection \cite{guyon2003,john1994} constitutes another framework for Model Selection, in which a family of partially ordered constrained hypotheses spaces is generated through elimination of features. 
	
	A limitation of feature selection is that the family of candidate hypotheses spaces is too constrained, so it may not be sharp enough for some problems of interest, while the limitation of other common methods are the restriction to a nested family of candidate models. In this paper, we propose a family of hypotheses spaces, called Learning Space, which can be designed with adequate constraints for each class of problems, and has properties which one can take advantage of to implement Model Selection algorithms more efficient than an exhaustive search of the candidate models. In this context, feature selection, SRM and penalization methods become particular cases.
	
	We propose an extension of the classical learning model, defining the model $(\mathcal{H}, \mathbb{L}(\mathcal{H}),\mathbb{A},\mathcal{D}_{N})$ composed by a hypotheses space $\mathcal{H}$; a Learning Space $\mathbb{L}(\mathcal{H})$, a poset of subspaces of $\mathcal{H}$, which covers $\mathcal{H}$, and satisfies a property regarding the VC dimension of related subspaces; and a learning algorithm $\mathbb{A}(\mathbb{L}(\mathcal{H}),\mathcal{D}_{N})$, which processes $\mathbb{L}(\mathcal{H})$ and a training sample $\mathcal{D}_{N}$, and returns $\hat{\mathcal{M}} \in \mathbb{L}(\mathcal{H})$, a subspace of $\mathcal{H}$ with nice properties, and $\hat{h}_{\hat{\mathcal{M}}}(\mathbb{A}) \in \hat{\mathcal{M}}$, a hypothesis that seeks to approximate the target $h^{\star}$ of $\mathcal{H}$. Under this framework, the learning of hypotheses is performed in two consecutive steps: one first learn a hypotheses space $\hat{\mathcal{M}}$ among the candidates in $\mathbb{L}(\mathcal{H})$, and then learn a hypothesis $\hat{h}(\mathbb{A}) \in \hat{\mathcal{M}}$ on it.
	
	Our framework is based on properties of VC dimension and can be applied to a great range of problems, including various kinds of hypotheses algebraic representations and learning algorithms, e.g., feature selection, neural networks \cite{neural1994,aggarwal2018} and lattice-based methods \cite{reis2018}. The classical learning model requires that a particular hypotheses space $\mathcal{H}$, and maybe a family of candidate models which usually does not have an algebraic structure more complex than a single chain of nested models, is selected \textit{a priori}, while this new approach has the feature of searching for target hypotheses in a family of hypotheses spaces given by $\mathbb{L}(\mathcal{H})$, a Learning Space of $\mathcal{H}$, which has a much richer structure than a single chain, as is usually a complete lattice, a semi-lattice or a poset. 
	
	An interesting feature of Model Selection via Learning Spaces is an implicit regularization. When we consider candidate models with distinct complexities and optimize over them an error measure based on a validation method, we have an implicit regularization in the sense of avoiding selecting too complex models that would lead to overfitting, and favoring simpler models which may better capture the patterns of the data, hence better generalize. In this sense, Model Selection via Learning Spaces may be considered an \textit{implicit complexity regularizer}, where the regularization is due to considering a great family of candidate models with distinct complexities and optimizing an error measure which tries to protect against overfitting. This differs from typical regularization procedures which often penalize an error measure by the complexity of each hypothesis and optimize the penalized error over $\mathcal{H}$, usually learning directly a hypothesis without searching for a subspace of $\mathcal{H}$ first \cite{bickel2006,micchelli2005,ng2004,oneto2016,massart2007}. 
	
	As opposed to influential and important methods for Model Selection based on penalization of loss functions (as those in \cite{massart2007,koltchinskii2001,koltchinskii2011,arlot2011,agarwal2011,bartlett2008} and many others), the spirit of the framework proposed here is to \textit{not} penalize the loss function, but rather obtain a general and consistent scheme for Model Selection based on resampling techniques such as cross-validation. Hence, in this scenario, the obtained regularization is implicit, since it is not explicitly considered in the loss function by penalizing it. However, although not in the spirit of the paper, penalized methods also fit into the framework. We discuss why this is the case in Section \ref{disPen} and leave a further study of penalization methods under the Learning Space framework for future researches.
	
	When selecting a hypotheses space, one should mind the estimation errors of learning on a given space. At principle, when one learns on $\mathcal{H}$, disregarding any hypothesis not in it, he commits two errors
	\begin{align}
	\label{intro_type1}
	\sup\limits_{h \in \mathcal{H}} \big|L_{\mathcal{D}_{N}}(h) - L(h) \big| & & \text{ and } & & L(\hat{h}(\mathbb{A})) - L(h^{\star}),
	\end{align}
	which we call type I and type II estimation errors, respectively. If type I estimation error is small, then we can estimate the out-of-sample error of any hypothesis in $\mathcal{H}$ by the empirical error with great precision. If type II estimation error is small, then the hypothesis $\hat{h}(\mathbb{A})$, estimated by the algorithm $\mathbb{A}$, well approximates a target $h^{\star}$. Since, fixed $\epsilon$ and the sample size $N$, the VC bounds for the tail probabilities (cf. \eqref{typeI}) of both errors in (\ref{intro_type1}) are usually increasing functions of VC dimension \cite{vapnik1998,devroye1996}, the smaller the hypotheses space is, in the VC dimension sense, the lesser are the estimation errors on it, with high probability, so that we may regulate the VC dimension of the hypotheses space to better generalize. 
	
	This may be accomplished by selecting a proper subset $\mathcal{M} \subset \mathcal{H}$ on which to learn. However, when we restrict the learning to such subspace, we commit another two errors, which we call types III and IV estimation errors, that are, respectively,
	\begin{align*}
	L(h^{\star}_{\mathcal{M}}) - L(h^{\star}) & & \text{ and } & & L(\hat{h}_{\mathcal{M}}(\mathbb{A})) - L(h^{\star}),
	\end{align*}
	in which $h^{\star}_{\mathcal{M}}$ is a target hypothesis of $\mathcal{M}$ and $\hat{h}_{\mathcal{M}}(\mathbb{A})$ is the hypothesis of $\mathcal{M}$ estimated by the algorithm $\mathbb{A}$. If type III estimation error is small, then a target hypothesis $h^{\star}_{\mathcal{M}}$ of $\mathcal{M}$ well approximates a target hypothesis $h^{\star}$ of $\mathcal{H}$. If type IV estimation error is small, then the estimated hypothesis $\hat{h}_{\mathcal{M}}(\mathbb{A})$ of $\mathcal{M}$ well approximates a target of $\mathcal{H}$. If both types III and IV estimation errors are small, then it is feasible to learn on $\mathcal{M}$. 
	
	If there is no prior information about the target $h^{\star}$ which allows us to consider a subset of $\mathcal{H}$ such that $\mathcal{M} \ni h^{\star}$, so type III estimation error is zero and type IV reduces to type II\footnote{In this case, type II estimation error is $L(\hat{h}_{\mathcal{M}}(\mathbb{A})) - L(h^{\star}_{\mathcal{M}})$, which is equal to type IV if $h^{\star} \in \mathcal{M}$.}, it may not be possible to restrict $\mathcal{H}$ beforehand and still estimate a good hypothesis relatively to $h^{\star}$. However, we may learn on a random subset $\hat{\mathcal{M}} \subset \mathcal{H}$ in a manner such that all four estimation errors are asymptotically zero, i.e., tend in probability to zero as the sample size increases. Such a subset is random, for it depends on sample $\mathcal{D}_{N}$: $\hat{\mathcal{M}}$ is learned from data.
	
	We show that the framework for Model Selection based on Learning Spaces learn a subspace $\hat{\mathcal{M}}$ which is such that types I and II estimation errors tend to be smaller than on $\mathcal{H}$ and types III and IV estimation errors are asymptotically zero. We show that, in the proposed approach, all estimation errors converge to zero when the sample size tends to infinity, and that $\hat{\mathcal{M}}$ converges with probability one to a target subspace $\mathcal{M}^{\star}$ of $\mathcal{H}$, which is the subspace in $\mathbb{L}(\mathcal{H})$ with the least VC dimension that contains a target hypothesis $h^{\star}$ (cf. Figure \ref{paradigms}). We say that $\hat{\mathcal{M}}$ is statistically consistent if it satisfies these two properties. Our approach does not demand the specification of a hypotheses space $\mathcal{M}$ \textit{a priori}, but rather introduces the learning of a hypotheses space from data among those in $\mathbb{L}(\mathcal{H})$ as a mean to better learn hypotheses, so prior information is all embed in $\mathbb{L}(\mathcal{H})$.
	
	The target hypotheses space is central in our approach. As is the case in all optimization problems, there must be an optimal solution to the Model Selection problem which satisfies certain desired conditions. In the proposed framework, the optimal solution is the subspace in $\mathbb{L}(\mathcal{H})$ with the least VC dimension which contains a target hypothesis (see Figure \ref{paradigms}). From the perspective of estimation errors, this subspace provides the best circumstances in $\mathbb{L}(\mathcal{H})$ to learn hypotheses with a fixed sample of size $N$. On the one hand, type III estimation error is zero and type IV reduces to type II. On the other hand, the VC dimension is minimal under these constraints, so the bounds for the tail probabilities of types I and II estimation errors are tightest. 
	
	The concept of target hypotheses space brings a new learning paradigm, under which one seeks to estimate this hypotheses space to better estimate a target hypothesis with a fixed sample size. Throughout this paper, we present the main results from the perspective of this paradigm, offering a guide on how one can, theoretically, better estimate, without having to increase the sample size, by incorporating prior knowledge about the problem at hand into $\mathbb{L}(\mathcal{H})$ and employing high computational power.
	
	Indeed, apart from the statistical consistency of the method, it is important to take into account the computational aspects of learning hypotheses via Learning Spaces. At principle, to select a model from $\mathbb{L}(\mathcal{H})$, one would have to apply a combinatorial algorithm that performs an exhaustive search of it, looking for the model which minimizes some error measure. If the cardinality of $\mathbb{L}(\mathcal{H})$ is too great, which will be often the case, this exhaustive search cannot be performed, so computing $\hat{\mathcal{M}}$ is not feasible. Nevertheless, due to the structure of some Learning Spaces, there may exist non-exhaustive algorithms to compute $\hat{\mathcal{M}}$, so, even if these algorithms are highly complex, they may still be employed to solve practical problems when high computational power is available.
	
	In this paper, we define the U-curve property that, when satisfied, allows a non-exhaustive calculation of $\hat{\mathcal{M}}$ via a U-curve algorithm. The property is a rigorous mathematical definition of a phenomenon intuitively related to the bias-variance trade-off or Occam's razor \cite{belkin2019,friedman1997}, in which the estimated error of a model decreases with its complexity up to a point when there is an inflection point, and the error starts increasing with the complexity, forming a U-shaped curve. This heuristic behavior, which is supported by empirical evidence, but often does not have a mathematical proof, is rigorously defined here for models organized in a lattice and proved to be satisfied in certain instances.
	
	We show that a specific Learning Space satisfies the U-curve property and establish a sufficient condition for it that is closely related to convexity, but under a lattice algebra. This is, to our knowledge, the first rigorous result in the literature asserting general conditions under which a non-exhaustive search of candidate models for the purpose of Model Selection returns an optimal solution. We then briefly discuss how one may take advantage of this property to develop U-curve algorithms that return $\hat{\mathcal{M}}$ if the U-curve property is satisfied, but may also be employed efficiently to obtain suboptimal solutions when the property is not satisfied, as has been done for feature selection, where the candidate models form a Boolean lattice (see \cite{u-curve3,featsel,reis2018,u-curve1,ucurveParallel} for more details). 
	
	Following this Introduction, Section \ref{GE} presents the formal definition of Learning Space, as well as examples. Section \ref{LearningFramework} presents the framework of learning hypotheses via Learning Spaces, while Section \ref{SecConsistency} shows the statistical consistency of such framework. Section \ref{SecUproperty} presents the U-curve property, which enables a non-exhaustive search of $\mathbb{L}(\mathcal{H})$ to estimate $\hat{\mathcal{M}}$ via a U-curve algorithm, further discussed in Section \ref{U-curve}. In the Discussion, we review the main contributions of this paper and future research perspectives. Proofs of the results are presented on the Appendix.
	
	\section{Learning Spaces}
	\label{GE}
	
	In this section, we first present the main concepts related to the learning of hypotheses, and discuss how to estimate the error of a hypotheses space. Then, we define the Learning Spaces and present an intuitive manner of building them, as well as examples. At last, we define the minimums of a Learning Space, and present the central notion of target hypotheses space.
	
	\subsection{Learning of hypotheses}
	
	Let $X$ be a random vector and $Y$ a random variable defined on a same probability space $(\Omega,\mathcal{S},\mathbb{P})$, with ranges $\mathcal{X} \subset \mathbb{R}^{d}, d \geq 1$, and $\mathcal{Y} = \{0,1\}$, respectively. We denote $P(x,y) \coloneqq \mathbb{P}(X \leq x,Y \leq y)$ as the joint probability distribution of $(X,Y)$ at point $(x,y) \in \mathcal{X} \times \mathcal{Y}$. We define a sample $\mathcal{D}_{N} = \{(X_{1},Y_{1}), \dots, (X_{N},Y_{N})\}, N \geq 1,$ as a sequence of independent and identically distributed random vectors defined on $(\Omega,\mathcal{S},\mathbb{P})$ with joint distribution $P$. Throughout this paper, we assume that distribution $P$ is unknown, but fixed.
	
	Let $\mathcal{H}$ be a general hypotheses space, whose typical element $h$ is a function $h: \mathcal{X} \mapsto \mathcal{Y}$. We only make a mild measurability assumption about $\mathcal{H}$: it is an arbitrary set of functions $h$ with domain $\mathcal{X}$ and image $\mathcal{Y}$ such that $h \circ X$ is $(\Omega,\mathcal{S})$-measurable. We denote subsets of $\mathcal{H}$ as $\mathcal{M}_{i}$ indexed by the positive integers, i.e., $i \in \mathbb{Z}_{+}$. We may also denote a subset of $\mathcal{H}$ by $\mathcal{M}$ to ease notation. Whenever we say \textit{model} throughout the paper, we mean a subset of $\mathcal{H}$.
	
	For each hypothesis in $\mathcal{H}$, we assign a value indicating the loss incurred employing it as a predictor for $Y$ from the values of $X$. Let $\ell: \mathcal{Y} \times \mathcal{Y} \mapsto \mathbb{R}_{+}$ be a positive measurable loss function. The out-of-sample error, also known in the literature as risk or loss, of a hypothesis $h \in \mathcal{H}$ is defined as
	\begin{equation*}
	L(h) \coloneqq \mathbb{E}[\ell(h(X),Y)] = \int_{\mathcal{X} \times \mathcal{Y}} \ell(h(x),y) \ dP(x,y),
	\end{equation*}
	in which $\mathbb{E}$ means expectation under $\mathbb{P}$. We also consider the in-sample or empirical error of a hypothesis $h$ under sample $\mathcal{D}_{N}$, defined as
	\begin{equation*}
	L_{\mathcal{D}_{N}}(h) \coloneqq \frac{1}{N} \sum_{l=1}^{N} \ell(h(X_{l}),Y_{l}),
	\end{equation*}
	that is the mean of $\ell(h(x),y)$ on sample $\mathcal{D}_{N}$. Throughout this paper, we consider $\ell$ to be the simple loss function $\ell(h(x),y) = \mathds{1}\{h(x) \neq y\}$, so that $L(h)$ is the classification error $\mathbb{P}(h(X) \neq Y)$.
	
	The main goal of Machine Learning is to approximate target hypotheses, that are hypotheses in $\mathcal{H}$ which minimize the out-of-sample error. These hypotheses are in the set
	\begin{equation*}
	h^{\star} \coloneqq \arginfA\limits_{h \in \mathcal{H}} L(h).
	\end{equation*}
	Furthermore, we will also be interested in target hypotheses of subsets of $\mathcal{H}$ which are in
	\begin{align*}
	h^{\star}_{i} \coloneqq \arginfA\limits_{h \in \mathcal{M}_{i}} L(h) & & h^{\star}_{\mathcal{M}} \coloneqq \arginfA\limits_{h \in \mathcal{M}} L(h),
	\end{align*}
	depending on the subset. Since $\mathcal{H}$ may be a proper subset of the space of all functions with domain $\mathcal{X}$ and image $\mathcal{Y}$, there may exist a $f: \mathcal{X} \mapsto \mathcal{Y}, f \notin \mathcal{H}$, with $L(f) < L(h^{\star})$. However, in this paper, we focus on approximating the best hypotheses in $\mathcal{H}$, so we disregard all hypotheses not in it. 
	
	Under the Empirical Risk Minimization (ERM) paradigm \cite{vapnik2000} with sample $\mathcal{D}_{N}$, the target hypotheses are estimated by the ones in 
	\begin{align*}
	\hat{h}(\mathcal{D}_{N}) \coloneqq \arginfA\limits_{h \in \mathcal{H}} L_{\mathcal{D}_{N}}(h),
	\end{align*}	
	while the estimated optimal hypotheses of models are in
	\begin{align*}
	\hat{h}_{i}(\mathcal{D}_{N}) \coloneqq \arginfA\limits_{h \in \mathcal{M}_{i}} L_{\mathcal{D}_{N}}(h) & & \hat{h}_{\mathcal{M}}(\mathcal{D}_{N}) \coloneqq \arginfA\limits_{h \in \mathcal{M}} L_{\mathcal{D}_{N}}(h).
	\end{align*}
	We assume the minimum of $L$ and $L_{\mathcal{D}_{N}}$ is achieved in $\mathcal{H}$ and in all subspaces of it that we consider throughout this paper, so the sets above are not empty.
	
	\subsubsection{VC dimension}
	
	The framework proposed in this paper relies on the complexity of hypotheses spaces, that may be measured by their VC dimension, which is as follows.
	
	\begin{definition}
		\label{VCdimension} \normalfont
		The $k$-th shatter coefficient of a hypotheses space $\mathcal{H}$ is defined as
		\begin{equation*}
		S(\mathcal{H},k) = \max\limits_{(x_{1},\dots,x_{k}) \in \mathcal{X}^{k}} \Big|\big\{\big(h(x_{1}),\dots,h(x_{k})\big): h \in \mathcal{H}\big\}\Big|,
		\end{equation*}
		in which $|\cdot|$ is the cardinality of a set. The VC dimension of a hypotheses space $\mathcal{H}$ is the greatest integer $k \geq 1$ such that $S(\mathcal{H},k) = 2^{k}$ and is denoted by $d_{VC}(\mathcal{H})$. If $S(\mathcal{H},k) = 2^{k}$ for all integer $k \geq 1$, we denote $d_{VC}(\mathcal{H}) = \infty$.
	\end{definition} 
	
	The VC dimension assesses the capability of the functions in $\mathcal{H}$ in classifying instances with values in $\mathcal{X}$ into the categories $\mathcal{Y}$. The VC dimension plays a central role in Machine Learning Theory, as it is a measure of the complexity of a hypotheses space, which may be employed to bound estimation errors involved in the learning of hypotheses. Furthermore, the VC dimension allows us to define a Learning Space as a decomposition of $\mathcal{H}$ based on it: the VC dimension rules the complexity of subspaces, enabling the decomposition of $\mathcal{H}$ into subspaces of different complexities.
	
	\subsection{Model error estimation}
	\label{esti_Lhat}
	
	In order to compare subspaces of $\mathcal{H}$ for Model Selection, one needs to estimate the error of them, which is defined as
	\begin{equation*}
	L(\mathcal{M}) \coloneqq  \inf\limits_{h \in \mathcal{M}} L(h) = L(h^{\star}_{\mathcal{M}}),
	\end{equation*}
	for any $\mathcal{M} \subset \mathcal{H}$. Even though $\hat{h}_{\mathcal{M}}(\mathcal{D}_{N})$ is a consistent estimator of $h^{\star}_{\mathcal{M}}$, the substitution error $L_{\mathcal{D}_{N}}(\hat{h}_{\mathcal{M}}(\mathcal{D}_{N}))$ is generally an optimistically biased estimator of $L(\mathcal{M})$, specially if the sample size is relatively small \cite[Section 2.4]{estimation}. Furthermore, since $L_{\mathcal{D}_{N}}(\hat{h}_{\mathcal{M}_{1}}(\mathcal{D}_{N})) \geq L_{\mathcal{D}_{N}}(\hat{h}_{\mathcal{M}_{2}}(\mathcal{D}_{N}))$ if $\mathcal{M}_{1} \subset \mathcal{M}_{2}$, selecting models based on this error is susceptible to \textit{overfitting}, as minimizing it leads to the selection of more complex models, which may explain the sample very well, but do not generalize well to non-observed data. 
	
	The framework proposed in this paper is not dependent on any specific estimator for $L(\mathcal{M})$, since it may be carried out employing many types of estimators. To illustrate the method, we consider two common estimators for $L(\mathcal{M})$ based on an independent validation sample and cross-validation, which we define below. Other estimators, for instance based on a Bootstrap technique \cite{estimation,efron1979,efron1983,efron1997,kohavi1995}, could also be employed. When there is no need to specify which estimator of $L(\mathcal{M})$ we are referring, we denote simply $\hat{L}(\mathcal{M})$ to mean an arbitrary estimator.
	
	\subsubsection{Validation Sample}
	
	Fix a sequence $\{V_{N}: N \geq 1\}$ such that $\lim_{N \rightarrow \infty} V_{N} = \lim\limits_{N \to \infty} N - V_{N} = \infty$, and let $\mathcal{D}_{N}^{(\text{train})} = \{(X_{l},Y_{l}): 1 \leq l \leq N - V_{N}\}$ and $\mathcal{D}_{N}^{(\text{val})} = \{(X_{l},Y_{l}): N - V_{N} < l \leq N\}$ be a split of $\mathcal{D}_{N}$ into a training and validation sample. Proceeding in this manner, we have two samples $\mathcal{D}_{N}^{(\text{train})}$ and $\mathcal{D}_{N}^{(\text{val})}$ which are independent. The estimator under the validation sample is given by
	\begin{linenomath}
		\begin{equation}
		\label{VALLhat}
		\hat{L}_{\text{val}}(\mathcal{M}) \coloneqq L_{\mathcal{D}_{N}^{(\text{val})}}(\hat{h}_{\mathcal{M}}^{(\text{train})}) =  \frac{1}{V_{N}} \sum_{N- V_{N} < l \leq N}  \ell\big(\hat{h}_{\mathcal{M}}^{(\text{train})}(X_{l}),Y_{l}\big),
		\end{equation}
	\end{linenomath}
	in which $\hat{h}_{\mathcal{M}}^{(\text{train})} = \arginfA_{h \in \mathcal{M}} L_{\mathcal{D}_{N}^{(\text{train})}}(h)$ is the ERM hypothesis of $\mathcal{M}$ under $\mathcal{D}_{N}^{(\text{train})}$.
	Instances in the validation sample are not in the training data, hence may provide less biased information about the generalization quality of $\hat{h}_{\mathcal{M}}^{(\text{train})}$. This estimator is specially useful when there is a great sample available, so one can divide it in training and validation samples with great size themselves. However, when there is little data available, a method based on resampling may perform better \cite{molinaro2005}.
	
	\subsubsection{k-fold cross-validation}
	
	Fix $k \in \mathbb{Z}_{+}$ and assume $N \coloneqq kn$, for a $n \in \mathbb{Z}_{+}$. Then, let $\mathcal{D}_{N}^{(j)} \coloneqq \{(X_{l},Y_{l}): (j-1)n < l \leq jn\}, j = 1,\dots,k$, be a partition of $\mathcal{D}_{N}$: $\mathcal{D}_{N} = \bigcup_{j=1}^{k} \mathcal{D}_{N}^{(j)}$ and $\mathcal{D}_{N}^{(j)} \cap \mathcal{D}_{N}^{(j^{\prime})} = \emptyset$ if $j \neq j^{\prime}$. We define
	\begin{equation*}
	\hat{h}^{(j)}_{\mathcal{M}} \coloneqq \arginfA_{h \in \mathcal{M}} \ L_{\mathcal{D}_{N}\setminus\mathcal{D}_{N}^{(j)}}(h) = \arginfA_{h \in \mathcal{M}} \ \frac{1}{(k-1)n} \sum_{\substack{l \leq (j-1)n \\ \cup \ l > jn}} \ell(h(X_{l}),Y_{l})
	\end{equation*}
	as the ERM hypotheses of the sample $\mathcal{D}_{N}\setminus\mathcal{D}_{N}^{(j)}$, that is the sample composed by all folds but the $j$-th, and
	\begin{equation*}
	\hat{L}_{\text{cv(k)}}^{(j)}(\mathcal{M}) \coloneqq L_{\mathcal{D}_{N}^{(j)}}(\hat{h}^{(j)}_{\mathcal{M}}) =  \frac{1}{n} \sum_{(j-1)n < l \leq jn} \ell(\hat{h}^{(j)}_{\mathcal{M}}(X_{l}),Y_{l})
	\end{equation*}
	as the validation error of the $j$-th fold.
	
	The k-fold cross-validation estimator of $L(\mathcal{M})$ is then given by
	\begin{equation}
	\label{CVLhat}
	\hat{L}_{\text{cv(k)}}(\mathcal{M}) \coloneqq \frac{1}{k} \ \sum_{j=1}^{k} \hat{L}_{\text{cv(k)}}^{(j)}(\mathcal{M}),
	\end{equation}
	that is the average validation error over the folds. Estimator \eqref{CVLhat} seeks to diminish the bias of \eqref{VALLhat} by applying a resampling strategy and averaging the validation error over these samples \cite[Section~2.5]{estimation}. Although we focus on the k-fold, other cross validation methods \cite{arlot2010,stone1974} may also be employed in the framework for Model Selection proposed in this paper.
	
	\subsection{Learning Spaces}
	
	Let $\mathcal{H}$ be a general hypotheses space with $d_{VC}(\mathcal{H}) < \infty$ and $\mathbb{L}(\mathcal{H}) \coloneqq \{\mathcal{M}_{i}: i \in \mathcal{J} \subset \mathbb{Z}_{+}\}$ be a finite subset of the power set of $\mathcal{H}$, i.e., $\mathbb{L}(\mathcal{H}) \subset \mathcal{P}(\mathcal{H})$ and $|\mathcal{J}| < \infty$. We say that the poset $(\mathbb{L}(\mathcal{H}),\subset)$ is a Learning Space if
	
	\vspace{0.1cm}
	\begin{itemize}
		\item[(i)] $\bigcup\limits_{i \in \mathcal{J}} \mathcal{M}_{i} = \mathcal{H}$ 
		\item[(ii)] $\mathcal{M}_{1}, \mathcal{M}_{2} \in \mathbb{L}(\mathcal{H})$ and $\mathcal{M}_{1} \subset \mathcal{M}_{2}$ implies $d_{VC}(\mathcal{M}_{1}) < d_{VC}(\mathcal{M}_{2})$.
	\end{itemize}
	\vspace{0.1cm}
	We define the VC dimension of $\mathbb{L}(\mathcal{H})$ as
	\begin{equation*}
	d_{VC}(\mathbb{L}(\mathcal{H})) \coloneqq \max\limits_{i \in \mathcal{J}} d_{VC}(\mathcal{M}_{i}),
	\end{equation*}
	for which an upper bound is $d_{VC}(\mathcal{H})$.
	
	On the one hand, for $\mathbb{L}(\mathcal{H})$ to be a structuring of $\mathcal{H}$ it should cover $\mathcal{H}$, so the need for (i). On the other hand, condition (ii) implies that any element $\mathcal{M} \in \mathbb{L}(\mathcal{H})$ is maximal in the sense that there does not exist $\mathcal{M}' \in \mathbb{L}(\mathcal{H})$ such that $d_{VC}(\mathcal{M}') = d_{VC}(\mathcal{M})$ and $\mathcal{M}' \subset \mathcal{M}$, so it guarantees that if $\mathcal{M}_{1} \subset \mathcal{M}_{2}$ then the complexity of $\mathcal{M}_{2}$ is greater than that of $\mathcal{M}_{1}$. We note that one could choose $\{\mathcal{M}_{1},\dots,\mathcal{M}_{n}\}$ without thinking of it as a decomposition of a hypotheses space $\mathcal{H}$. Nevertheless, if condition (ii) is satisfied, then it would be a Learning Space of $\mathcal{H} = \cup_{i} \mathcal{M}_{i}$, so taking $\mathcal{H}$ as this union, the only non-trivial condition is (ii).
	
	As $\mathbb{L}(\mathcal{H})$ covers $\mathcal{H}$, when one searches for a hypotheses space in $\mathbb{L}(\mathcal{H})$ on which to learn, no hypothesis of $\mathcal{H}$ is lost beforehand, as there is no prior constraint which exclude hypotheses from it. Indeed, all hypotheses in the candidate models, hence all hypotheses in $\mathcal{H}$, are available to be estimated, since it is enough that a model which contains it is selected, and then it is learned from it. A constraint is added to $\mathcal{H}$ a \textit{posteriori}, and based on data, as the method to be proposed seeks to select, based solely on data, i.e., learn from data, the hypotheses space in $\mathbb{L}(\mathcal{H})$ on which to learn, that can be a constrained subspace $\mathcal{M} \subsetneq \mathcal{H}$.
	
	Although $\mathbb{L}(\mathcal{H})$ is not unique, i.e., there are multiple subsets of $\mathcal{P}(\mathcal{H})$ which are Learning Spaces, there are classes of Learning Spaces that have some properties which allow a non-exhaustive search of $\mathbb{L}(\mathcal{H})$ when looking for a target model (cf. Section \ref{SecUproperty}). The main class of Learning Spaces are the Lattice Learning Spaces, which have a lattice structure with useful properties that increase the performance of search algorithms (see \cite{davey2002} for a definition of lattice).
	
	\begin{definition}
		\label{lattice_LS} \normalfont
		Let $\mathbb{L}(\mathcal{H})$ be a Learning Space of $\mathcal{H}$. We say that $\mathbb{L}(\mathcal{H})$ is a Lattice Learning Space if $(\mathbb{L}(\mathcal{H}),\subset,\wedge,\vee,\mathcal{O},\mathcal{I})$ is a complete lattice, that is a poset with the least ($\mathcal{O}$) and greatest ($\mathcal{I}$) subset, and two operators defined for all pairs of elements in it: the supremum operator $\vee$ and the infimum operator $\wedge$.
	\end{definition} 
	
	In this paper, we consider only Lattice Learning Spaces, or posets of one, although the abstract framework is quite general and may also be applied to other cases.
	
	\begin{remark}
		\normalfont Although we consider the VC-dimension, other complexity measures of hypotheses spaces could be used to define the Learning Space. We considered the VC-dimension since it is a property of the hypotheses space and does not depend on the data unknown distribution $P$. We note the value of the VC dimension is not of importance to the algebraic aspect of the Learning Space definition, but only the fact that it increases when we consider nested models. Hence, any other complexity measure such that this increase is also observed for the chosen nested models would generate the same Learning Space.
	\end{remark}
	
	\subsection{Building Learning Spaces}
	
	The first step in building a Learning Space is fixing an algebraic parametric representation of the hypotheses in $\mathcal{H}$. Some important families of learning models, such as neural networks and lattice functions, have a particular algebraic structure, with a parametric representation and a corresponding optimization algorithm to estimate a target hypothesis from a given sample. In neural networks, the parameters represent the weight of synapses connections and the minimum bursting threshold. In lattice functions, the parameters are lattice intervals which represent a join or meet minimal representation. In the binary case, the lattice representation derives from Boolean Algebra. In the continuous or integer cases, it derives from lattice algebra. In the discrete case, the representation is finite, since the function domain $\mathcal{X}$ is finite. In the continuous case, there are topological conditions on the parametric representation of hypotheses which guarantee finite representations (see \cite{banon1991}). Therefore, all these learning models are parametric, and hence a Learning Space $\mathbb{L}(\mathcal{H})$ should represent sets of hypotheses in terms of their parameters. 
	
	The algebraic structure of $(\mathbb{L}(\mathcal{H}),\subset)$ may be defined from the learning model and algebraic representation fixed, in the following manner. Let $(\mathcal{F},\leq)$ be a poset, in which $\mathcal{F}$ is an arbitrary set with finite cardinality. Moreover, let $\mathcal{R}: \mathcal{F} \mapsto Im(\mathcal{R}) \subset \mathcal{P}(\mathcal{H})$ be a lattice isomorphism from set $(\mathcal{F},\leq)$ to $(Im(\mathcal{R}),\subset)$, a subset of the power set of $\mathcal{H}$ partially ordered by inclusion. This means that $\mathcal{R}$ is bijective and if $a,b \in \mathcal{F}, a \leq b$, then $\mathcal{R}(a) \subset \mathcal{R}(b)$, so $\mathcal{R}$ preserves the partial order $\leq$ on $\mathcal{F}$ as the partial order on $Im(\mathcal{R})$ given by inclusion. Then, if 
	\vspace{0.1cm}
	\begin{itemize}
		\item[(i)] $\bigcup\limits_{a \in \mathcal{F}} \mathcal{R}(a) = \mathcal{H}$ and
		\item[(ii)] $a,b \in \mathcal{F}, a \leq b,$ implies $d_{VC}(\mathcal{R}(a)) < d_{VC}(\mathcal{R}(b))$,
	\end{itemize}
	\vspace{0.15cm}
	we may define $\mathbb{L}(\mathcal{H}) \coloneqq Im(\mathcal{R})$ as a Learning Space of $\mathcal{H}$. Isomorphisms which satisfy these conditions play a central role in the theory, and hence we formally define them.
	
	\begin{definition}
		\normalfont
		Given a partially ordered set $(\mathcal{F},\leq)$, a Lattice isomorphism $\mathcal{R}: (\mathcal{F},\leq) \mapsto (\mathbb{L}(\mathcal{H}),\subset)$, with $\mathbb{L}(\mathcal{H}) \subset \mathcal{P}(\mathcal{H})$, which satisfies (i) and (ii) is called a \textit{Learning Space generator}.
	\end{definition}
	
	A Learning Space is completely defined by a triple $(\mathcal{F},\leq,\mathcal{R})$, in which the elements of $\mathcal{F}$ may be interpreted as sets of parameters which describe a subset of hypotheses, i.e., the hypotheses in $\mathcal{R}(a), a \in \mathcal{F},$ are represented by the parameters $a$, so that, in particular, $\mathcal{F}$ generates a parametric representation of the functions in $\mathcal{H}$. For this reason, we call $(\mathcal{F},\leq)$ a parametric poset of $\mathcal{H}$. Therefore, in general, to build a Learning Space of $\mathcal{H}$ we apply a generator to a parametric poset of its hypotheses. Furthermore, since the generator $\mathcal{R}$ is an isomorphism, it preserves properties of $(\mathcal{F},\leq)$, hence, for instance, by applying $\mathcal{R}$ to $(\mathcal{F},\leq,\wedge,\vee,\mathcal{O},\mathcal{I})$, a complete lattice, we obtain a Lattice Learning Space.

	\subsection{Examples of Learning Spaces}
	\label{SecExamples}
	
	We present some examples of Learning Spaces completely defined by a triple $(\mathcal{F},\leq,\mathcal{R})$.
	
	\begin{example}[Feature Selection]
		\label{feature_lattice} \normalfont
		Let $\mathcal{H}$ be a space of hypotheses with domain $\mathcal{X} \subset \mathbb{R}^{d}, d > 1$. Let $\mathcal{F} = \mathcal{P}(\{1,\dots,d\})$ be the power set of $\{1,\dots,d\}$ partially ordered by inclusion, so that $(\mathcal{F},\subset,\cap,\cup,\emptyset,\{1,\dots,d\})$ is a complete Boolean lattice. Consider the Learning Space generator $\mathcal{R}: \mathcal{F} \mapsto Im(\mathcal{R}) \subset \mathcal{P}(\mathcal{H})$ given by
		\begin{equation*}
		\mathcal{R}(a) = \Big\{h \in \mathcal{H}: h(x) = h(z), \text{ if } x \equiv_{a} z\Big\},
		\end{equation*} 
		in which $a = \{a_{1},\dots,a_{j}\} \in \mathcal{F}$ and $x = (x_{1},\dots,x_{d}) \equiv_{a} z = (z_{1},\dots,z_{d})$ if, and only if, $x_{a_{i}} = z_{a_{i}}$ for $i = 1,\dots,j$, so $\mathcal{R}(a)$ contains the hypotheses which depend solely on features in $a$. The lattice isomorphism $\mathcal{R}$ satisfies condition (i), and often satisfies (ii), as in many applications the VC dimension is an increasing function of the number of features, so $Im(\mathcal{R})$ is often a Learning Space.
		
		\hfill$\square$
	\end{example}
	
	\begin{example}[Partition Lattice]
		\label{partition_lattice} \normalfont	
		The target hypothesis $h^{\star}$ creates equivalence classes in $\mathcal{X}$, partitioning it according to its classification of each input value, that is an ordered partition $\mathcal{X}_{0},\mathcal{X}_{1}$, with $\mathcal{X}_{i} = \{x \in \mathcal{X}: h^{\star}(x) = i\}, i = 0,1$, in which each element is in exactly one of these sets, the one which represents its classification according to $h^{\star}$. Actually, every hypothesis $h$ generates an ordered partition according to its classification of the input values. Hence, there is a duality between hypotheses and the ordered partitions of $\mathcal{X}$ with two parts. 
		
		An example of learning method based on the partition-hypothesis duality are the pyramid-based learning models which specify the hypotheses space through the specification of properties of the desired hypotheses, e.g., increasing or decreasing hypotheses, and some sampling pyramid, which constrains the set of hypotheses considered, generating enormous equivalence classes (partition) in the hypotheses' domain (see \cite{dougherty2001,grauman2006,junior2002} for more details).
		
		The partition-hypothesis duality brings upon the paradigm of learning a hypothesis through a partition. This task can be performed in two manners, either by first choosing explicitly an unordered partition and learning an ordination of it, which generates a hypothesis, or learning the partition implicitly while learning the hypothesis. 
		
		In both manners, it is not necessary to have a partition consisting of exactly two parts, as one could rather consider more sets that form a partition with the constraint that elements within a same set should be classified in the same output. These sets represent an equivalence relation on $\mathcal{X}$, with equivalence between elements in a same set. This may ease the estimation process since the original partition sets are broken into more simple ones, which may have better topological features and be easier to estimate.
		
		Assume we know a partition $\mathcal{X}_{1}, \dots, \mathcal{X}_{k}$ of $k$ greater than two parts such that there is a hypothesis which respects it that well-approximate a target one. The set of hypotheses that respect a partition is composed by the ones that classify elements within a same part in the same output. Once we fix a partition $\mathcal{X}_{1}, \dots, \mathcal{X}_{k}$, the learning is performed considering only hypotheses that respect it, that is a constrained hypotheses space on which the learning may be \textit{better} than on all of $\mathcal{H}$, since this constrained space (a) has lesser VC dimension (that is equal to $k$) and (b) contains a hypothesis that \textit{well} approximate a target hypothesis.
		
		The hypotheses that respect each partition form the models in the Partition Lattice Learning Space (cf. Figure \ref{partitionL}), so this Learning Space is a natural family of candidate models under the partition-hypothesis duality when $|\mathcal{X}|$ is finite. We now formally define this Learning Space.
		
		Let $\mathcal{X}$ be such that $|\mathcal{X}| < \infty$, and let $\mathcal{H}$ contain all functions with domain $\mathcal{X}$ and image $\{0,1\}$. A partition of $\mathcal{X}$ is a set $\pi$ of non-empty subsets of $\mathcal{X}$, called blocks or parts, such that every element $x \in \mathcal{X}$ is in exactly one of these blocks. A partition $\pi$ generates an equivalence relation on $\mathcal{X}$ in the sense that $x$ and $z$ in $\mathcal{X}$ are $\pi$-equivalent, i.e., $x \equiv_{\pi} z$, if, and only if, they are in the same block of partition $\pi$.
		
		Define $\mathcal{F} \coloneqq \{\pi: \pi \text{ is a partition of } \mathcal{X}\}$ as the set of all partitions of $\mathcal{X}$, partially ordered by $\leq$ defined as
		\begin{equation*}
		\pi_{1} \leq \pi_{2} \text{ if, and only if, } x \equiv_{\pi_{2}} z \text{ implies } x \equiv_{\pi_{1}} z
		\end{equation*}
		for $\pi_{1}, \pi_{2} \in \mathcal{F}$, which is a complete lattice $(\mathcal{F},\leq,\wedge,\vee,\{\mathcal{X}\},\mathcal{X})$. See Figure \ref{partitionL} for an example of Partition Lattice. By applying the generator $\mathcal{R}: \mathcal{F} \mapsto Im(\mathcal{R}) \subset \mathcal{P}(\mathcal{H})$ given by
		\begin{equation*}
		\mathcal{R}(\pi) \coloneqq \mathcal{H}|_{\pi} = \Big\{h \in \mathcal{H}: h(x) = h(z) \text{ if } x \equiv_{\pi} z \Big\},
		\end{equation*}
		for $\pi \in \mathcal{F}$, we obtain a Lattice Learning Space $\mathbb{L}(\mathcal{H}) \coloneqq Im(\mathcal{R})$. The set $\mathcal{R}(\pi)$ is formed by all hypotheses which classify the points inside a block of $\pi$ in a same category, that are the hypotheses which respect $\pi$. This lattice is indeed a Leaning Space since $\mathcal{H} \in \mathbb{L}(\mathcal{H})$ and $d_{VC}(\mathcal{H}|_{\pi}) = |\pi|$. In this case, the parameters of the functions $h \in \mathcal{H}$ are the elements in their domain $\mathcal{X}$, in contrast, for example, to the features they depend on as in Example \ref{feature_lattice}. If $\mathcal{H}$ is given by the set of Boolean functions $h: \{0,1\}^{d} \mapsto \{0,1\}$ we have the special case of a Boolean Partition Lattice, which is studied in \cite{edu}.
		
		This Learning Space is quite general and has subsets which are themselves useful Learning Spaces, illustrating how one may drop subsets of $\mathbb{L}(\mathcal{H})$ to obtain other Learning Spaces according to the needs of the application at hand. For example, the Feature Selection Learning Space, when $\mathcal{X} \subset \mathbb{R}^{d}, |\mathcal{X}| < \infty$, is a sub-lattice of the Partition Lattice Learning Space and is represented in orange in Figure \ref{partitionL} when $\mathcal{X} = \{0,1\}^2$. This Learning Space can also be obtained by applying the generator of Example \ref{feature_lattice}.
		
		Apart from dropping nodes, one can also obtain Learning Spaces for subsets $\mathcal{M} \subset \mathcal{H}$ by taking $\mathbb{L}(\mathcal{M}) = \mathbb{L}(\mathcal{H}) \cap \mathcal{M} \coloneqq \{\mathcal{M}^{\prime} \cap \mathcal{M}: \mathcal{M}^{\prime} \in \mathbb{L}(\mathcal{H})\}$ which is often a Learning Space of $\mathcal{M}$. For example, by taking $\mathcal{M}$ as the non-decreasing hypotheses in $\mathcal{H}$ when $\mathcal{X} \subset \mathbb{R}$, i.e., $\mathcal{M} = \{h \in \mathcal{H}: h(x_{1}) \leq h(x_{2}), \text{ if } x_{1} < x_{2}\}$, we obtain $\mathbb{L}(\mathcal{M})$ composed by the dashed nodes in Figure \ref{partitionL} whose hypotheses are the ones in $\mathcal{R}(\pi) \cap \mathcal{M}$. Some nodes may be dropped in order for $\mathbb{L}(\mathcal{H}) \cap \mathcal{M}$ to satisfy (ii).
		
		These are some examples of how one can incorporate prior knowledge about the problem at hand into the Learning Space to remodel it accordingly. If one believes that a target hypothesis does not depend on all features, he may use the Feature Selection Learning Space, while if one believes a target hypothesis is non-decreasing, he may consider the respective subset of the Partition Lattice Learning Space. In both cases, the Learning Space is defined by constraining the Partition Lattice Learning Space according to prior information, so multipurpose learning algorithms based on the Partition Lattice may be applied on several instances when distinct levels of prior information is available.
		
		\hfill$\square$
		
		\begin{figure*}[ht]
			\begin{center}
				\begin{tikzpicture}[scale=0.476, transform shape]
				\tikzstyle{hs} = [rectangle,draw=black, rounded corners, minimum height=3em, minimum width=3.5em, node distance=2.5cm, line width=1pt]
				\tikzstyle{hs2} = [rectangle,fill=orange,draw=black, rounded corners, minimum height=3em, minimum width=3.5em, node distance=2.5cm, line width=1pt]
				\tikzstyle{hs3} = [rectangle,draw=black, rounded corners, minimum height=3em, minimum width=3.5em, node distance=2.5cm, line width=1pt,dashed]
				\tikzstyle{hs4} = [rectangle,fill=orange,draw=black, rounded corners, minimum height=3em, minimum width=3.5em, node distance=2.5cm, line width=1pt,dashed]
				\tikzstyle{branch}=[fill,shape=circle,minimum size=4pt,inner sep=0pt]
				\tikzstyle{simple}=[rectangle,draw=black,minimum height=2em,minimum width=3.5em, node distance=2cm,line width=1pt]
				\tikzstyle{dot}=[node distance=2.5cm, line width=1pt, minimum width=1em]
				\node[hs2] at (0, 2) (a1234) {$\{\{1\},\{2\},\{3\},\{4\}\}$};
				\node[hs3] at (-15, -5) (2a1) {$\{\{1\},\{2,3,4\}\}$};
				\node[hs] at (-10, -5) (2a2) {$\{\{1,3,4\},\{2\}\}$};
				\node[hs] at (-5, -5) (2a3) {$\{\{1,2,4\},\{3\}\}$};
				\node[hs3] at (0, -5) (2a4) {$\{\{1,2,3\},\{4\}\}$};
				\node[hs4] at (5, -5) (2a12) {$\{\{1,2\},\{3,4\}\}$};
				\node[hs2] at (10, -5) (2a13) {$\{\{1,3\},\{2,4\}\}$};
				\node[hs] at (15, -5) (2a14) {$\{\{1,4\},\{2,3\}\}$};
				\node[hs] at (-12.5, -1.5) (3a12) {$\{\{1\},\{2\},\{3,4\}\}$};
				\node[hs] at (-7.5, -1.5) (3a13) {$\{\{1\},\{2,4\},\{3\}\}$};
				\node[hs] at (-2.5, -1.5) (3a14) {$\{\{1\},\{2,3\},\{4\}\}$};
				\node[hs] at (2.5, -1.5) (3a23) {$\{\{1,4\},\{2\},\{3\}\}$};
				\node[hs] at (7.5, -1.5) (3a24) {$\{\{1,3\},\{2\},\{4\}\}$};
				\node[hs] at (12.5, -1.5) (3a34) {$\{\{1,2\},\{3\},\{4\}\}$};
				\node[hs4] at (0, -8.5) (empty) {$\{\{1,2,3,4\}\}$};
				\node at (-12.5,1.5) (tab1)  {$\mathcal{M}_{1} = $ 
					\begin{tabular}{|c|cccccccc|}
					\hline
					$\mathcal{X}$ & $h_{1}$ & $h_{2}$ & $h_{3}$ & $h_{4}$ & $h_{5}$ & $h_{6}$ & $h_{7}$ & $h_{8}$\\
					\hline 
					1 & 1 & 0 & 0 & 1 & 1 & 0 & 1 & 0\\
					2 & 1 & 0 & 1 & 0 & 0 & 0 & 1 & 1\\
					3 & 1 & 0 & 1 & 0 & 1 & 1 & 0 & 0\\
					4 & 1 & 0 & 1 & 0 & 1 & 1 & 0 & 0\\
					\hline
					\end{tabular}};
				\node at (-14.25,-8) (tab2)  {$\mathcal{M}_{2} = $
					\begin{tabular}{|c|cccc|}
					\hline
					$\mathcal{X}$ & $h_{1}$ & $h_{2}$ & $h_{3}$ & $h_{4}$ \\
					\hline 
					1 & 1 & 0 & 0 & 1 \\
					2 & 1 & 0 & 1 & 0 \\
					3 & 1 & 0 & 1 & 0 \\
					4 & 1 & 0 & 1 & 0 \\
					\hline
					\end{tabular}};

				\begin{scope}[line width=1pt]
				\draw[->] (tab1) -- (3a12);
				\draw[->] (tab2) -- (2a1);
				\draw[-] (empty) -- (2a1);
				\draw[-] (empty) -- (2a2);
				\draw[-] (empty) -- (2a3);
				\draw[-] (empty) -- (2a4);
				\draw[-] (empty) -- (2a12);
				\draw[-] (empty) -- (2a13);
				\draw[-] (empty) -- (2a14);
				\draw[-] (2a1) -- (3a12);
				\draw[-] (2a1) -- (3a13);
				\draw[-] (2a1) -- (3a14);
				\draw[-] (2a2) -- (3a12);
				\draw[-] (2a2) -- (3a23);
				\draw[-] (2a2) -- (3a24);
				\draw[-] (2a3) -- (3a13);
				\draw[-] (2a3) -- (3a23);
				\draw[-] (2a3) -- (3a34);
				\draw[-] (2a4) -- (3a14);
				\draw[-] (2a4) -- (3a24);
				\draw[-] (2a4) -- (3a34);
				\draw[-] (2a12) -- (3a12);
				\draw[-] (2a12) -- (3a34);
				\draw[-] (2a13) -- (3a13);
				\draw[-] (2a13) -- (3a24);
				\draw[-] (2a14) -- (3a14);
				\draw[-] (2a14) -- (3a23);
				\draw[-] (3a12) -- (a1234);
				\draw[-] (3a13) -- (a1234);
				\draw[-] (3a14) -- (a1234);
				\draw[-] (3a23) -- (a1234);
				\draw[-] (3a24) -- (a1234);
				\draw[-] (3a34) -- (a1234);
				\end{scope}
				
				\end{tikzpicture}
			\end{center}
			\caption{\footnotesize Partition Lattice for Linear Classifiers with $d = 4$ or for $\mathcal{X} = \{1,2,3,4\}$. The tables present the hypotheses in selected subspaces $\mathcal{M}_{1}, \mathcal{M}_{2}$ of the Partition Lattice Learning Space for $\mathcal{X} = \{1,2,3,4\}$. The orange nodes represent the Boolean Lattice of feature selection when $\mathcal{X} = \{0,1\}^{2}$ so its points are $1 = (0,0), 2 = (0,1), 3 = (1,0)$ and $4 = (1,1)$. The dashed nodes are the ones in $\mathbb{L}(\mathcal{H}) \cap \mathcal{M}$ in which $\mathcal{M}$ is composed by the non-decreasing hypotheses.}
			\label{partitionL}
		\end{figure*}
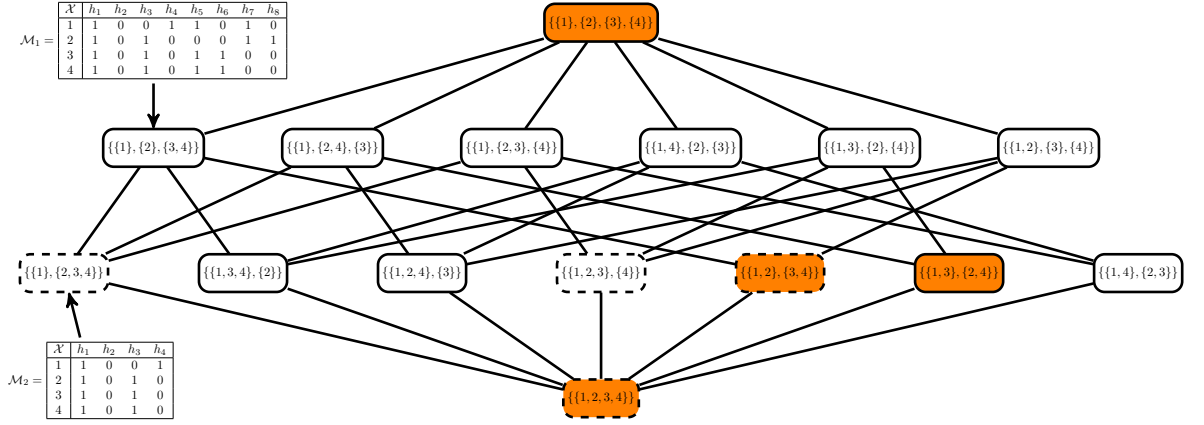
	\end{example}
	
	\begin{example}[Linear Classifiers]
		\label{parametric_lattice} \normalfont
		Let $\mathcal{H}$ be given by the linear classifiers in $\mathbb{R}^{d}, d \geq 1$:
		\begin{equation*}
		\mathcal{H} = \Bigg\{h_{a}(x) = \frac{1}{2} \ \text{sgn}\Big\{a_{0} + \sum_{i=1}^{d} a_{i}x_{i}\Big\} + \frac{1}{2}: a_{i} \in \mathbb{R}\Bigg\},
		\end{equation*}
		in which $x = (x_{1},\dots,x_{d}) \in \mathbb{R}^{d}$ and $h_{a}$ is the function indexed by its parameters $a = (a_{0},\dots,a_{d}) \in \mathbb{R}^{d+1}$. Denoting $\mathcal{A} = \{1,\dots,d\}$, we consider two distinct Learning Spaces generators: from the Boolean Lattice $(\mathcal{P}(\mathcal{A}),\subset,\cap,\cup,\{\emptyset\},\mathcal{A})$ and from the Partition Lattice $(\Pi_{\mathcal{A}},\leq,\wedge,\vee,\{\mathcal{A}\},\mathcal{A})$ of $\mathcal{A}$, in which $\mathcal{P}(\mathcal{A})$ is the power set of $\mathcal{A}$ and $\Pi_{\mathcal{A}}$ is the set of all partitions of $\mathcal{A}$. The Partition Lattice is represented in Figure \ref{partitionL} for $d = 4$.
		
		Define $\mathcal{R}_{1}: \mathcal{P}(\mathcal{A}) \mapsto \mathcal{P}(\mathcal{H})$ as
		\begin{equation*}
		\mathcal{R}_{1}(A) = \Big\{h_{a} \in \mathcal{H}: a_{j} = 0  \text{ if } j \notin A \cup \{0\}\Big\},
		\end{equation*}
		for $A \in \mathcal{P}(\mathcal{A})$ as a feature selection generator, and define $\mathcal{R}_{2}: \Pi_{\mathcal{A}} \mapsto \mathcal{P}(\mathcal{H})$ as
		\begin{equation*}
		\mathcal{R}_{2}(\pi) = \Big\{h_{a} \in \mathcal{H}: a_{j} = a_{k} \text{ if } j \equiv_{\pi} k\},
		\end{equation*}
		for $\pi \in \Pi_{\mathcal{A}}$ as a generator which equal parameters, i.e., create equivalence classes in $\mathcal{A}$. Both $\mathcal{R}_{1}, \mathcal{R}_{2}$ clearly satisfies (i) and
		\begin{align*}
		d_{VC}(\mathcal{R}_{1}(A)) = |A| + 1 & & d_{VC}(\mathcal{R}_{2}(\pi)) = |\pi| + 1,
		\end{align*}
		so they also satisfy (ii). Therefore, these lattice isomorphisms generate two distinct Lattice Learning Spaces for a same hypotheses space $\mathcal{H}$ and the application at hand will dictate which one is the more suitable to solve the problem. For example, if it is believed that $h^{\star}$ does not depend on all variables, then the Boolean Lattice Learning Space may be preferable; otherwise, one would rather choose the Partition Lattice Learning Space, a subset of it or the intersection of it with a subset $\mathcal{M} \subset \mathcal{H}$ if one believes that the target linear classifier satisfies some property.
		
		\hfill$\square$
	\end{example}
	
	\begin{example}[Deep Neural Networks]
		\label{example_NN} \normalfont
		For each $\theta \coloneqq (\theta_{0},\dots,\theta_{m+1}) \in \Theta \subset \mathbb{R}^{t_{0}} \times \cdots \times \mathbb{R}^{t_{m+1}}$, let $f_{0}^{\theta_{0}}, \dots, f_{m+1}^{\theta_{m+1}}, m \geq 2$, be a sequence of measurable functions 
		\begin{linenomath}
			\begin{align}
			\label{condLayer}
			f_{0}^{\theta_{0}}: \mathcal{X} \mapsto \mathbb{R}^{d_{1}} & & \cdots & &  f_{k}^{\theta_{k}}: \mathbb{R}^{d_{k}} \mapsto \mathbb{R}^{d_{k+1}} & & \cdots & & f_{m+1}^{\theta_{m}+1}: \mathbb{R}^{d_{m+1}} \mapsto \mathcal{Y}
			\end{align} 
		\end{linenomath}
		for $1 \leq k \leq m$, that are completely determined by parameters $\theta$, in which $1 \leq d_{k}, t_{k^\prime} < \infty$ for all $1 \leq k \leq m + 1$ and $0 \leq k^\prime \leq m+1$. When we say \textit{measurable}, we mean Borel measurable in case the function domain is a Euclidean space, and $\left(\mathcal{X},\sigma(\mathcal{X})\right)$-measurable when its domain is $\mathcal{X}$, where $\sigma(\mathcal{X})$ is a $\sigma$-algebra of $\mathcal{X}$. 
		
		Assume that $m \geq 2$ is fixed and a class
		\begin{linenomath}
			\begin{equation*}
			\mathcal{A} = \Big\{\{f_{0}^{\theta_{0}}, \dots, f_{m+1}^{\theta_{m+1}}\}: \theta \in \Theta \Big\}
			\end{equation*}
		\end{linenomath}
		satisfying (\ref{condLayer}) is given. Then, for each $\theta \in \Theta$ define $g_{\theta}: \mathcal{X} \mapsto \mathcal{Y}$ as
		\begin{linenomath}
			\begin{equation*}
			g_{\theta}(x) \coloneqq f_{m+1}^{\theta_{m+1}} \circ f_{m}^{\theta_{m}} \circ \cdots \circ f_{0}^{\theta_{0}}(x),
			\end{equation*}
		\end{linenomath}
		for $x \in \mathcal{X}$. We call $\mathcal{A}$ a \textit{Deep Neural Network} (DNN) architecture with $m$ hidden layers which can represent the classifiers in set
		\begin{linenomath}
			\begin{equation*}
			\mathcal{R}(\mathcal{A}) \coloneqq \Big\{g_{\theta}: \theta \in \Theta\Big\},
			\end{equation*} 
		\end{linenomath}
		which is a collection of functions with domain $\mathcal{X}$ and image $\mathcal{Y}$.
		
		Our definition of DNN does not seek to contemplate all kinds of DNNs used nowadays, but rather focus on lesser sophisticated architectures, which includes Fully Connected DNNs, when $f^{\theta_{k}}_{k}$ is given by applying a multidimensional linear transformation to the input followed by an activation function coordinate-wise.
		
		Since an architecture $\mathcal{A}$ generates a hypotheses space $\mathcal{H} \coloneqq \mathcal{R}(\mathcal{A})$, the selection of an architecture, in what is known in the literature as Neural Architecture Search \cite{elsken2019}, could be performed via the selection of a hypotheses space among the candidates in $\mathbb{L}(\mathcal{H})$, generated by distinct DNN architectures. The first step in this task is to define $\mathbb{L}(\mathcal{H})$, which would be as follows.
		
		Let $\mathscr{A} = \{\mathcal{A}_{1}, \dots, \mathcal{A}_{n}\}$ be a collection of architectures and $(\mathscr{A},\leq)$ a poset such that $\mathcal{H} = \bigcup_{i = 1}^{n} \mathcal{R}(\mathcal{A}_{i})$ and
		\begin{equation}
		\label{cond_DNN}
		\mathcal{A}_{i} \leq \mathcal{A}_{j}, i \neq j \implies \mathcal{R}(\mathcal{A}_{i}) \subset \mathcal{R}(\mathcal{A}_{j}) \text{ and } d_{VC}(\mathcal{R}(\mathcal{A}_{i})) < d_{VC}(\mathcal{R}(\mathcal{A}_{j})),
		\end{equation}
		so that $\mathbb{L}(\mathcal{H}) \coloneqq \{\mathcal{R}(\mathcal{A}_{i}): i = 1, \dots, n\}$ is a Learning Space of $\mathcal{H}$. In this framework, selecting a subspace $\mathcal{M}_{i} \coloneqq \mathcal{R}(\mathcal{A}_{i})$ from $\mathbb{L}(\mathcal{H})$ would imply selecting an architecture from $\mathscr{A}$.
		
		Although easy to define, the construction of a Learning Space for DNNs is not a straightforward task. On the one hand, it is not clear which classifiers are in $\mathcal{R}(\mathcal{A}_{i})$ for a given architecture $\mathcal{A}_{i}$. On the other hand, since there is a lot of redundancy among the parameters $\theta$ \cite{denil2013,neyshabur2018}, one usually do not know if a constraint on the parametric space $\Theta$, e.g., by setting some parameters to zero, actually generates a constraint on the hypotheses space $\mathcal{H}$. Furthermore, the VC dimension of $\mathcal{R}(\mathcal{A}_{i})$ is often not exactly known, and is not necessarily proportional to either the dimensionality of $\theta$, that is the architecture number of parameters, or its number of layers \cite{sontag1998,harvey2017,bartlett2019,karpinski1997}. As opposite, in the examples above it is clear that constraints on the parametric space, such as dropping variables (feature selection and linear classifiers), equating parameters (liner classifiers) and creating equivalence classes on the classifier domain (partition lattice), not only represent constraints on the hypotheses space, but are also isomorphisms which preserve the partial ordering on the parametric space and satisfy the Learning Space property regarding the VC dimension of related models.
		
		Nevertheless, a Learning Space for DNNs could be built from the top-down or from the bottom-up. The former would be performed by starting from the greatest architecture $\mathcal{A}$ and then adding constraints to its parametric space $\Theta$ successively in various ways and ordinations, generating the chains of $\mathbb{L}(\mathcal{H})$ from the top. The latter would mean starting from concise building blocks, e.g., nodes of layers, full layers or compositions of layers $f_{j}^{\theta_{j}}$, which can be combined in many ways to form more complex architectures creating chains of $\mathbb{L}(\mathcal{H})$ from the bottom. This task may be accomplished with or without explicitly knowing the subspace $\mathcal{R}(\mathcal{A}_{i})$ for all $i$, since it is enough to guarantee that \eqref{cond_DNN} is in force.
		
		\hfill$\square$
	\end{example}
	
	\subsection{Minimums of Learning Spaces}
	
	Model Selection via Learning Spaces relies on the concept of minimums of a $\mathbb{L}(\mathcal{H})$, since it will select a minimizer of error $\hat{L}$, that is a global minimum, and also a local minimum, of $\mathbb{L}(\mathcal{H})$. We start with the definition of continuous chain, which are basically chains with no \textit{jumps} over models in $\mathbb{L}(\mathcal{H})$.
	
	\begin{definition}
		\label{chain} \normalfont
		A sequence $\mathcal{M}_{i_{1}} \subset \mathcal{M}_{i_{2}} \subset \cdots \subset \mathcal{M}_{i_{k}}$ is called a continuous chain of $\mathbb{L}(\mathcal{H})$ if, and only if, $d(\mathcal{M}_{i_{j}},\mathcal{M}_{i_{j+1}}) = 1$ for all $j \in \{1,\dots,k-1\}$ in which $d$ means distance on the Hasse diagram $(\mathbb{L}(\mathcal{H}),\subset)$.
	\end{definition}
	
	We now define the minimums of $\mathbb{L}(\mathcal{H})$.
	
	\begin{definition}
		\label{local_min} \normalfont
		The model $\mathcal{M}_{i_{j\star}}$ is:
		\begin{itemize}
			\setlength\itemsep{0.5em}
			\item a \textbf{weak local minimum} of a continuous chain $\mathcal{M}_{i_{1}} \subset \mathcal{M}_{i_{2}} \subset \cdots \subset \mathcal{M}_{i_{k}}$ of $\mathbb{L}(\mathcal{H})$ if $\hat{L}(\mathcal{M}_{i_{j\star}}) \leq \min\big(\hat{L}(\mathcal{M}_{i_{j\star-1}}),\hat{L}(\mathcal{M}_{i_{j\star+1}})\big)$, in which $\hat{L}(\mathcal{M}_{i_{0}}) \equiv \hat{L}(\mathcal{M}_{i_{k+1}}) \equiv + \infty$;
			\item a \textbf{strong local minimum} of $\mathbb{L}(\mathcal{H})$ if it is a weak local minimum of all continuous chains of $\mathbb{L}(\mathcal{H})$ which contain it;
			\item a \textbf{global minimum of a continuous chain} if $\hat{L}(\mathcal{M}_{i_{j\star}}) = \min\limits_{1 \leq j \leq k} \hat{L}(\mathcal{M}_{i_{j}})$;
			\item a \textbf{global minimum} of $\mathbb{L}(\mathcal{H})$ if $\hat{L}(\mathcal{M}_{i_{j\star}}) = \min\limits_{i \in \mathcal{J}} \hat{L}(\mathcal{M}_{i})$.
		\end{itemize}
	\end{definition}
	
	\begin{remark}
		\normalfont
		The concept of strong local minimum may be defined in another manner, by considering the distance on the Hasse diagram $(\mathbb{L}(\mathcal{H}),\subset)$. A model $\mathcal{M}$ is a strong local minimum of $\mathbb{L}(\mathcal{H})$ if, and only if,
		\begin{equation*}
		\hat{L}(\mathcal{M}) \leq \min\left\{\hat{L}(\mathcal{M}_{i}): \mathcal{M}_{i} \in \mathbb{L}(\mathcal{H}), \ d(\mathcal{M}_{i},\mathcal{M}) = 1\right\},
		\end{equation*}
		that is, if its estimated error is lesser or equal to the estimated error of all its immediate neighbors, that are models at a distance one from it.
	\end{remark}
	
	In Definition \ref{local_min}, the error of a model $\mathcal{M}$ is estimated by a fixed estimator $\hat{L}(\mathcal{M})$, so the concept of minimums is dependent upon the choice of estimator. Since $\mathcal{M}_{i_{j}} \subset \mathcal{M}_{i_{j+1}}, j = 1, \dots, k-1$, the sequence $\{L_{\mathcal{D}_{N}}(\hat{h}_{i_{j}}(\mathcal{D}_{N})): j = 1, \dots, k\}$ is non-increasing, hence if we estimated this error simply by the resubstitution $L_{\mathcal{D}_{N}}(\hat{h}_{\mathcal{M}}(\mathcal{D}_{N}))$, the minimum would always be the greatest model of the chain so that we would be \textit{doomed to overfit}. On the other hand, employing an estimator involving validation samples, for example, as those in Section \ref{esti_Lhat}, is a manner of avoiding \textit{overfitting}. We discuss these questions in more details in Section \ref{SecUproperty}.
	
	\subsection{The target hypotheses space}
	
	The concept of target hypotheses space is central in our approach. As is the case in all optimization problems, there must be an optimal solution to the Model Selection problem which satisfies certain desired conditions. In the proposed framework, the optimal is the model in $\mathbb{L}(\mathcal{H})$ with the least VC dimension which contains a target hypothesis. To be exact in this definition, we need to consider equivalence classes of models, as it is not possible to differentiate some models with the concepts of the theory.
	
	Define in $\mathbb{L}(\mathcal{H})$ the equivalence relation given by
	\begin{linenomath}
		\begin{align}
		\label{equiv_class}
		\mathcal{M}_{i} \sim \mathcal{M}_{j} \text{ if, and only if, } d_{VC}(\mathcal{M}_{i}) = d_{VC}(\mathcal{M}_{j}) \text{ and } L(\mathcal{M}_{i}) = L(\mathcal{M}_{j}),
		\end{align}
	\end{linenomath}
	for $\mathcal{M}_{i}, \mathcal{M}_{j} \in \mathbb{L}(\mathcal{H})$: two models in $\mathbb{L}(\mathcal{H})$ are equivalent if they have the same VC dimension and error. Let 
	\begin{linenomath}
		\begin{equation*}
		\mathcal{L}^{\star} = \argminA\limits_{\mathcal{M} \in \ \nicefrac{\mathbb{L}(\mathcal{H})}{\sim}} L(\mathcal{M})
		\end{equation*}
	\end{linenomath}
	be the equivalence classes which contain a target hypothesis of $\mathcal{H}$ so that their error is minimum. We define the target model $\mathcal{M}^{\star} \in \nicefrac{\mathbb{L}(\mathcal{H})}{\sim}$ as
	\begin{linenomath}
		\begin{equation*}
		\mathcal{M}^{\star} = \argminA\limits_{\mathcal{M} \in \mathcal{L}^{\star}} d_{VC}(\mathcal{M}),
		\end{equation*}
	\end{linenomath}
	which is the class of the smallest models in $\mathbb{L}(\mathcal{H})$, in the VC dimension sense, that are not disjoint with $h^{\star}$. The target model has the lowest complexity among the unbiased models in $\mathbb{L}(\mathcal{H})$, that are models which contain a target hypothesis.
	
	The intuition of the paradigm proposed by this paper is presented in Figure \ref{paradigms}, in which the ellipses represent some models in $\mathbb{L}(\mathcal{H})$ and their area is proportional to their VC dimension. Assume that $\mathcal{H}$ is all we have to learn on, i.e., we are not willing to consider any hypothesis outside $\mathcal{H}$. Then, if we could choose, we would like to learn on $\mathcal{M}^{\star}$: \textit{the model in $\mathbb{L}(\mathcal{H})$ with the least VC dimension which contains a $h^{\star}$}. Of course, this model is dependent on both $\mathbb{L}(\mathcal{H})$ and $P$, i.e., it is not distribution-free, and thus we cannot establish beforehand, without looking at data, on which model of $\mathbb{L}(\mathcal{H})$ we should learn. Moreover, even if we looked at data, it could not be possible to search $\mathbb{L}(\mathcal{H})$ exhaustively to properly estimate $\mathcal{M}^{\star}$ by a $\hat{\mathcal{M}}$ learned from data and, in the general case, there is nothing guaranteeing that it is feasible to estimate $\mathcal{M}^{\star}$ anyhow.
	
	\begin{figure}[ht]
		\centering
		\includegraphics[scale = 0.5]{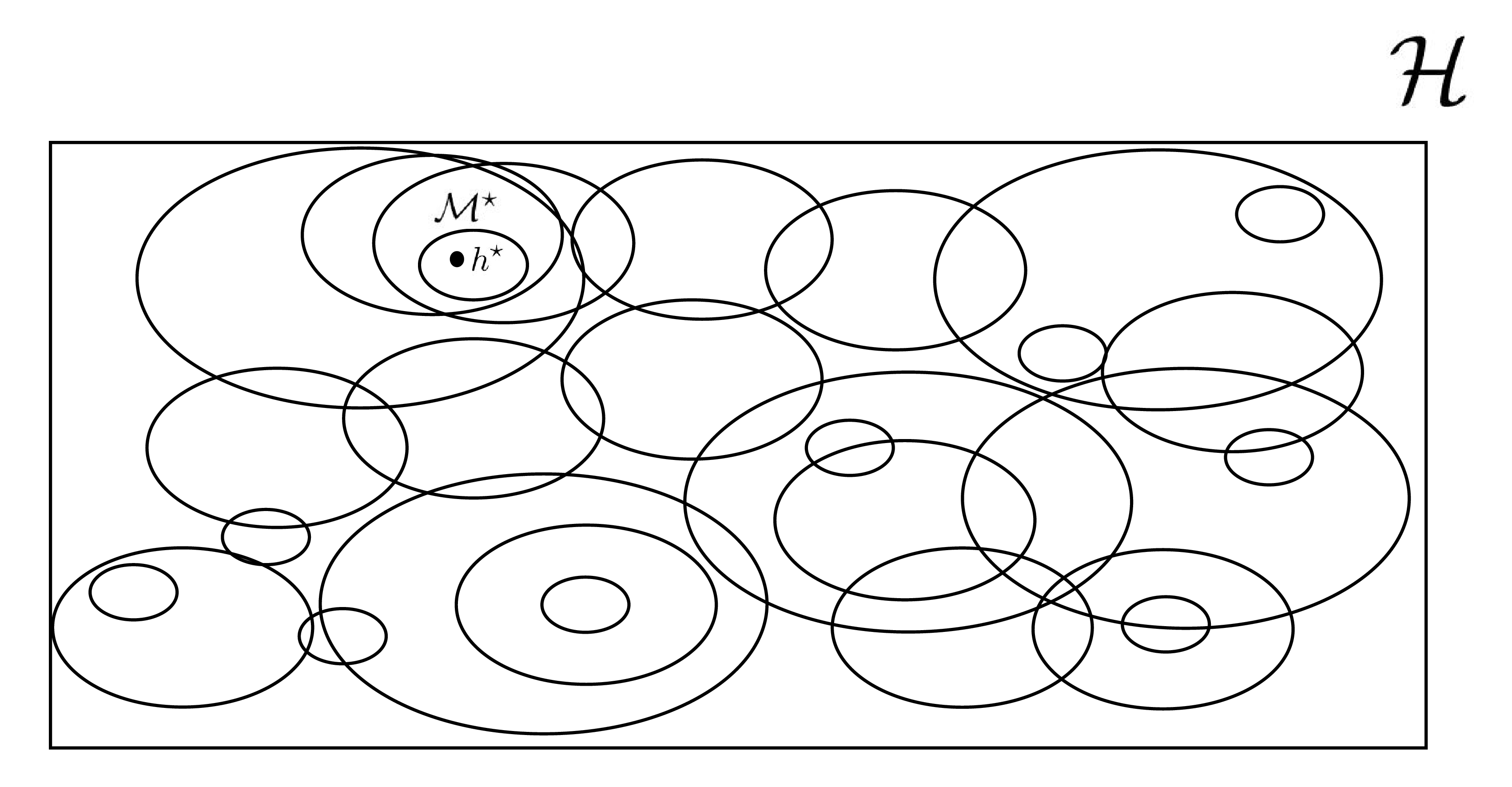}
		\caption{\footnotesize Decomposition of $\mathcal{H}$ by a $\mathbb{L}(\mathcal{H})$. We omitted some models for a better visualization, since $\mathbb{L}(\mathcal{H})$ should cover $\mathcal{H}$.}
		\label{paradigms}
	\end{figure}
	
	In fact, when learning on a hypotheses space $\hat{\mathcal{M}} \in \mathbb{L}(\mathcal{H})$, that is random, since is learned (estimated) from data, we commit three errors which are
	\begin{align}
	\label{ee23}
	\textbf{(II)} \ \ L(\hat{h}_{\hat{\mathcal{M}}}(\mathbb{A})) - L(h^{\star}_{\hat{\mathcal{M}}}) & & \textbf{(III)} \ \ L(h^{\star}_{\hat{\mathcal{M}}}) - L(h^{\star}) & & \textbf{(IV)} \ \ L(\hat{h}_{\hat{\mathcal{M}}}(\mathbb{A})) - L(h^{\star}),
	\end{align}
	that we call types II, III, and IV estimation errors.\footnote{We define type I estimation error in Section \ref{SecLearnOn} (cf. \eqref{typeIe}).} In a broad sense, type III error would represent the bias of learning on $\hat{\mathcal{M}}$, while type II would represent the variance within $\hat{\mathcal{M}}$, and type IV would be the error, with respect to $\mathcal{H}$, committed when learning on $\hat{\mathcal{M}}$ with algorithm $\mathbb{A}$. Indeed, type III error compares a target hypothesis of $\hat{\mathcal{M}}$ with a target hypothesis of $\mathcal{H}$, hence any difference between them would be a systematic bias of learning on $\hat{\mathcal{M}}$ when compared to learning on $\mathcal{H}$. Type II error compares the loss of the estimated hypothesis of $\hat{\mathcal{M}}$ and the loss of its target, assessing how much the estimated hypothesis varies from a target. These estimation errors are illustrated in Figure \ref{fig_error}.
	
	\begin{figure}[ht]
		\centering
		\includegraphics[scale = 0.1]{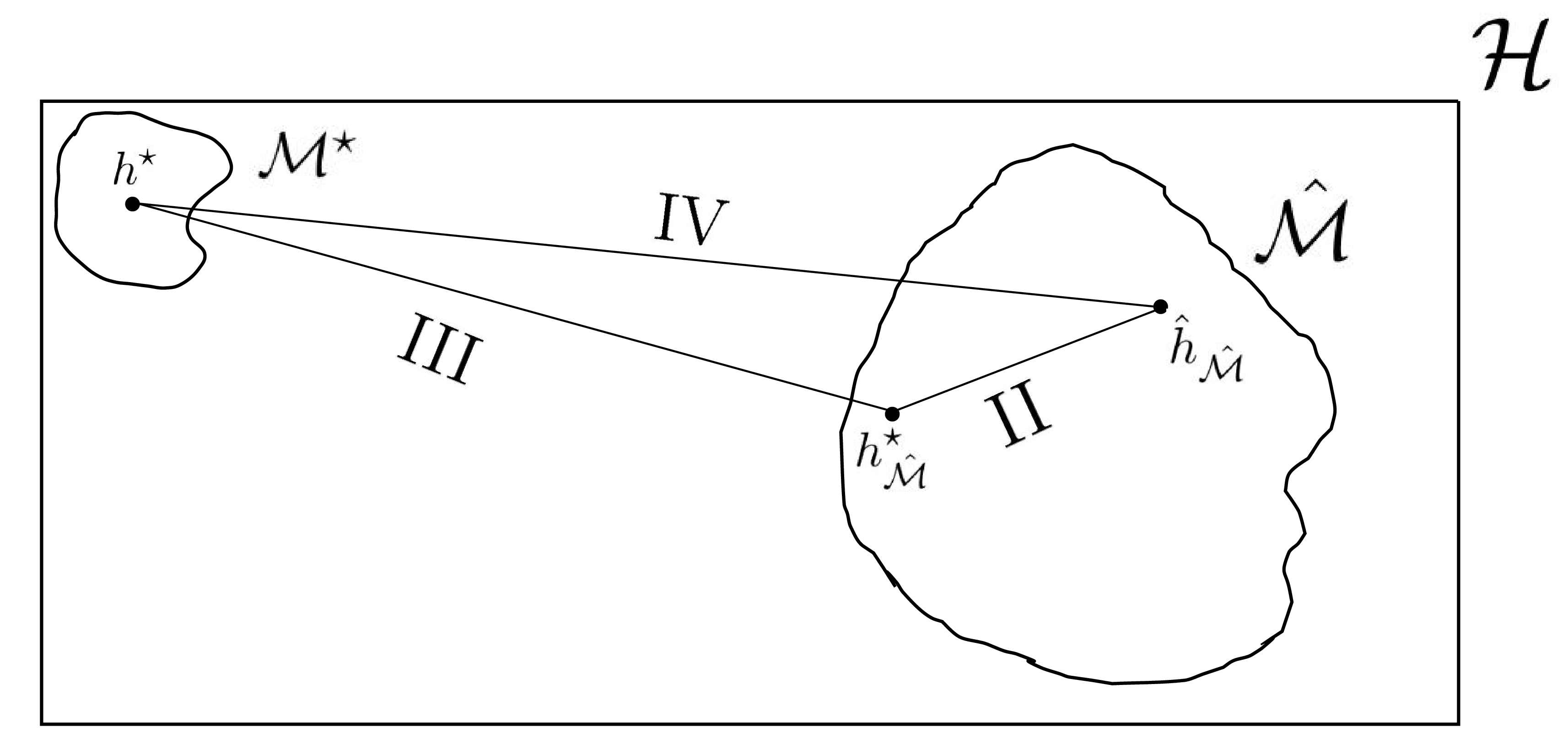}
		\caption{\footnotesize Types II, III, and IV estimation errors when learning on $\hat{\mathcal{M}}$, in which $\hat{h}_{\hat{\mathcal{M}}} \equiv \hat{h}_{\hat{\mathcal{M}}}(\mathbb{A})$.}
		\label{fig_error}
	\end{figure}
	
	As is often the case, there will be a bias-variance trade-off that should be minded when learning (on) $\hat{\mathcal{M}}$. Hence, the feasibility of a Model Selection approach depends on algorithms which (a) are more efficient than an exhaustive search of $\mathbb{L}(\mathcal{H})$ and (b) guarantee that, when the sample size increases, types II and III, and consequently type IV, estimation errors converge to zero, so there is no systematic bias when learning on $\hat{\mathcal{M}}$, and it is statistically consistent to do so.
	
	This paper presents an approach to Model Selection satisfying (a) and (b). In the following sections, we rigorously show, by extending the VC theory, that under mild conditions it is feasible to estimate $\mathcal{M}^{\star}$ by a $\hat{\mathcal{M}}$, a random model learned from data, which converges to $\mathcal{M}^{\star}$ with probability 1 when the sample size tends to infinity. Furthermore, we show bounds to the tail probabilities of the estimation errors of learning on $\hat{\mathcal{M}}$ which may be tighter than those we have when learning on $\mathcal{H}$, i.e., by introducing a bias III, which converges to zero, we may decrease the variance II of the learning process, so it is more efficient to learn on a model learned from data.
	
	Then, we discuss the U-curve phenomenon \cite{u-curve1} and show how it can be explored, via the solution of a U-curve problem \cite{reis2018} through a U-curve algorithm \cite{u-curve3,ucurveParallel}, to estimate a target hypothesis by first learning a hypotheses space $\hat{\mathcal{M}}$ and then a hypothesis in it, without exhaustively searching $\mathbb{L}(\mathcal{H})$. 
	
	In summary, in this paper we show that it is \textit{theoretically feasible}, and may be computationally possible, to learn the hypotheses space from data, via a Learning Space and a U-curve algorithm, even though this algorithm is NP-hard \cite{u-curve2,ucurveParallel}, as it is usually significantly more efficient than an exhaustive search, and may be efficiently employed to obtain suboptimal solutions.
	
	\section{The learning of hypotheses via Learning Spaces}
	\label{LearningFramework}
	
	The Learning Space framework for learning hypotheses is composed of two steps: first learn a model $\hat{\mathcal{M}}$ from $\mathbb{L}(\mathcal{H})$ and then learn a hypothesis on $\hat{\mathcal{M}}$. In this section, we define $\hat{\mathcal{M}}$ and present two ways of learning hypotheses on it under the Learning Space framework.
	
	\subsection{Learning hypotheses spaces}
	
	Model Selection via Learning Spaces is performed by applying a $(\Omega,\mathcal{S})$-measurable learning machine $\mathbb{M}_{\mathbb{L}(\mathcal{H})}$, dependent on the Learning Space $\mathbb{L}(\mathcal{H})$, satisfying
	\begin{linenomath}
		\begin{equation}
		\label{diagram}
		\omega \in \Omega \xrightarrow{(\mathcal{D}_{N},\hat{L})} (\mathcal{D}_{N}(\omega),\hat{L}(\omega)) \xrightarrow{\ \ \mathbb{M}_{\mathbb{L}(\mathcal{H})} \ \ } \mathcal{\hat{M}}(\omega) \in \mathbb{L}(\mathcal{H}),
		\end{equation}
	\end{linenomath}
	which is such that, given $\mathcal{D}_{N}$ and an estimator $\hat{L}$ of the error of each candidate model, learns a $\mathcal{\hat{M}} \in \mathbb{L}(\mathcal{H})$. Note from (\ref{diagram}) that $\mathcal{\hat{M}}$ is a $(\Omega,\mathcal{S})$-measurable $\mathbb{L}(\mathcal{H})$-valued function as it is the composition of measurable functions, i.e., $\mathcal{\hat{M}} \coloneqq \mathcal{\hat{M}}_{\mathcal{D}_{N},\hat{L},\mathbb{L}(\mathcal{H})} = \mathbb{M}_{\mathbb{L}(\mathcal{H})}\big(\mathcal{D}_{N},\hat{L}\big)$. Even though $\mathcal{\hat{M}}$ depends on $\mathcal{D}_{N}, \hat{L}$ and $\mathbb{L}(\mathcal{H})$, we drop the subscripts to ease notation. 
	
	The main feature of $\mathbb{M}_{\mathbb{L}(\mathcal{H})}$ which allows the learning of models is that type III estimation error converges in probability to zero:
	\begin{linenomath}
		\begin{equation}
		\label{bound_ghat}
		\mathbb{P}\Big(L(h^{\star}_{\hat{\mathcal{M}}}) - L(h^{\star}) > \epsilon\Big) \xrightarrow{N \rightarrow \infty} 0,
		\end{equation}
	\end{linenomath}
	for all $\epsilon > 0$ which is equivalent to
	\begin{linenomath}
		\begin{equation*}
		\mathbb{P}\Big(\hat{\mathcal{M}} \cap h^{\star} = \emptyset\Big) \xrightarrow{N \rightarrow \infty} 0,
		\end{equation*}
	\end{linenomath}
	since $|\mathbb{L}(\mathcal{H})| < \infty$. In fact, it is desired the model learned by $\mathbb{M}_{\mathbb{L}(\mathcal{H})}$ to be as simple as it can be under the restriction that it converges to the target model. Therefore, we will develop a learning machine $\mathbb{M}_{\mathbb{L}(\mathcal{H})}$ such that
	\begin{linenomath}
		\begin{equation}
		\label{consistent_LM}
		\mathcal{\hat{M}} = \mathbb{M}_{\mathbb{L}(\mathcal{H})}\big(\mathcal{D}_{N},\hat{L}\big) \xrightarrow{N \rightarrow \infty} \mathcal{M}^{\star} \text{ with probability 1},
		\end{equation}
	\end{linenomath}
	which implies (\ref{bound_ghat}).
	
	A learning machine $\mathbb{M}_{\mathbb{L}(\mathcal{H})}$ which satisfies (\ref{consistent_LM}) may be defined by mimicking the definition of $\mathcal{M}^{\star}$, but employing the estimated error $\hat{L}$ instead of the out-of-sample error $L$. Define in $\mathbb{L}(\mathcal{H})$ the equivalence relation given by
	\begin{linenomath}
		\begin{align*}
		\mathcal{M}_{i} \hat{\sim} \mathcal{M}_{j} \text{ if, and only if, } d_{VC}(\mathcal{M}_{i}) = d_{VC}(\mathcal{M}_{j}) \text{ and } \hat{L}(\mathcal{M}_{i}) = \hat{L}(\mathcal{M}_{j})
		\end{align*}
	\end{linenomath}
	for $\mathcal{M}_{i}, \mathcal{M}_{j} \in \mathbb{L}(\mathcal{H})$, which is a random $(\Omega,\mathcal{S})$-measurable equivalence relation. Let
	\begin{linenomath}
		\begin{equation*}
		\hat{\mathcal{L}} = \argminA\limits_{\mathcal{M} \in \ \nicefrac{\mathbb{L}(\mathcal{H})}{\hat{\sim}}} \hat{L}(\mathcal{M})
		\end{equation*}
	\end{linenomath}
	be the classes in $\nicefrac{\mathbb{L}(\mathcal{H})}{\hat{\sim}}$ which are global minimums of $\mathbb{L}(\mathcal{H})$ (cf. Definition \ref{local_min}). Then, $\mathbb{M}_{\mathbb{L}(\mathcal{H})}$ selects
	\begin{equation}
	\label{Ghat}
	\hat{\mathcal{M}} = \argminA\limits_{\mathcal{M} \in \mathcal{\hat{L}}} d_{VC}(\mathcal{M}),
	\end{equation} 
	the simplest global minimum of $\mathbb{L}(\mathcal{H})$. By selecting $\mathcal{\hat{M}}$ this way, we get to learn on relatively simple models, what is in accordance with the paradigm of selecting the simplest model that properly express reality, which in this case is represented by the fact that $\mathcal{\hat{M}}$ is the simplest global minimum. In Section \ref{SecConvTM} we show that \eqref{consistent_LM} holds when we define $\hat{\mathcal{M}}$ as \eqref{Ghat}.
	
	\subsection{Learning Hypotheses on $\hat{\mathcal{M}}$}
	\label{SecLearnOn}
	
	Once $\hat{\mathcal{M}}$ is selected, we need to learn a hypothesis from it that will be employed to solve the practical problem at hand. We propose two manners of performing such learning, that are characterized, respectively, by resubstitution on the sample $\mathcal{D}_{N}$ and considering an independent sample, as follows.
	
	A straightforward manner of learning on $\hat{\mathcal{M}}$ is to simply consider 
	\begin{align*}
	\hat{h}_{\hat{\mathcal{M}}}(\mathcal{D}_{N}) \coloneqq \arginfA\limits_{h \in \hat{\mathcal{M}}} L_{\mathcal{D}_{N}}(h),
	\end{align*}
	that are the hypotheses which minimize the empirical error under $\mathcal{D}_{N}$ on $\hat{\mathcal{M}}$. Since $\mathcal{D}_{N}$ was employed on the selection of $\hat{\mathcal{M}}$, through $\hat{L}$, estimator $\hat{h}_{\hat{\mathcal{M}}}(\mathcal{D}_{N})$ may be biased, so type IV estimation error $L(\hat{h}_{\hat{\mathcal{M}}}(\mathcal{D}_{N})) - L(h^{\star})$ may be great with high probability if $N$ is not large enough. We call this framework \textit{learning by reusing}.
	
	Another manner of learning on $\hat{\mathcal{M}}$ is to consider a sample $\tilde{\mathcal{D}}_{M} = \{(\tilde{X}_{l},\tilde{Y}_{l}): 1 \leq l \leq M\}$ of $M$ independent and identically distributed random vectors with distribution $P$, independent of $\mathcal{D}_{N}$, and consider 
	\begin{align*}
	\hat{h}_{\hat{\mathcal{M}}}(\tilde{\mathcal{D}}_{M}) \coloneqq \arginfA\limits_{h \in \hat{\mathcal{M}}} L_{\tilde{\mathcal{D}}_{M}}(h),
	\end{align*}
	that are the hypotheses which minimize the empirical error under $\tilde{\mathcal{D}}_{M}$ on $\hat{\mathcal{M}}$. Since the sample $\tilde{\mathcal{D}}_{M}$ is independent of $\mathcal{D}_{N}$ it may provide a less biased estimator of $h^{\star}_{\hat{\mathcal{M}}}$. We call this framework \textit{learning with independent sample}. The two frameworks for learning on $\hat{\mathcal{M}}$ are depicted in Figure \ref{learn_hyp}.
	
	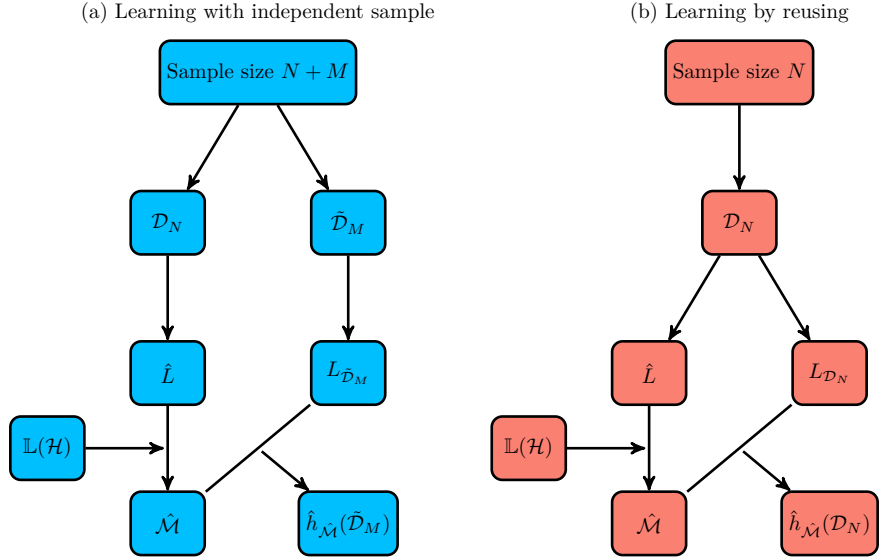
\begin{figure*}[ht]
		\begin{center}
			\begin{tikzpicture}[scale=0.4, transform shape]
			\tikzstyle{hs} = [rectangle,draw=black,fill = color1bg,rounded corners, minimum height=3em, minimum width=3.5em, node distance=2.5cm, line width=1pt]
			\tikzstyle{hs2} = [rectangle,draw=black,fill = bblue,rounded corners, minimum height=3em, minimum width=3.5em, node distance=2.5cm, line width=1pt]
			\tikzstyle{ns} = [rectangle,draw=white,rounded corners, minimum height=3em, minimum width=3.5em, node distance=2.5cm, line width=1pt]
			
			\node[ns,scale=2] at (8, 2) (name1) {(b) Learning by reusing};
			\node[hs,scale=2] at (8, 0) (s1) {Sample size $N$};
			\node[hs,scale=2] at (8, -5) (DN1) {$\mathcal{D}_{N}$};
			\node[hs,scale=2] at (5, -10) (hL1) {$\hat{L}$};
			\node[hs,scale=2] at (11, -10) (LDN1) {$L_{\mathcal{D}_{N}}$};
			\node[hs,scale=2] at (5, -15) (hM1) {$\hat{\mathcal{M}}$};
			\node[hs,scale=2] at (11, -15) (hhM1) {$\hat{h}_{\hat{\mathcal{M}}}(\mathcal{D}_{N})$};
			\node[hs,scale=2] at (1, -12.5) (LH2) {$\mathbb{L}(\mathcal{H})$};
			
			\node[ns,scale=2] at (-8, 2) (name2) {(a) Learning with independent sample};
			\node[hs2,scale=2] at (-8, 0) (s2) {Sample size $N+M$};
			\node[hs2,scale=2] at (-11, -5) (DN2) {$\mathcal{D}_{N}$};
			\node[hs2,scale=2] at (-5, -5) (DM2) {$\tilde{\mathcal{D}}_{M}$};
			\node[hs2,scale=2] at (-11, -10) (hL2) {$\hat{L}$};
			\node[hs2,scale=2] at (-5, -10) (LDM2) {$L_{\tilde{\mathcal{D}}_{M}}$};
			\node[hs2,scale=2] at (-11, -15) (hM2) {$\hat{\mathcal{M}}$};
			\node[hs2,scale=2] at (-5, -15) (hhM2) {$\hat{h}_{\hat{\mathcal{M}}}(\tilde{\mathcal{D}}_{M})$};
			\node[hs2,scale=2] at (-15, -12.5) (LH1) {$\mathbb{L}(\mathcal{H})$};
			
			\node at (-8,-12.5) (m2) {};
			\node at (8,-12.5) (m1) {};
			\node at (5,-12.5) (m22) {};
			\node at (-11,-12.5) (m11) {};

			\begin{scope}[line width=1pt]
			\draw[->] (s1) -- (DN1);
			\draw[->] (DN1) -- (LDN1);
			\draw[->] (DN1) -- (hL1);
			\draw[->] (hL1) -- (hM1);
			\draw[-] (hM1) -- (LDN1);
			\draw[->] (m1) -- (hhM1);
			\draw[->] (s2) -- (DN2);
			\draw[->] (s2) -- (DM2);
			\draw[->] (DN2) -- (hL2);
			\draw[->] (DM2) -- (LDM2);
			\draw[->] (hL2) -- (hM2);
			\draw[-] (hM2) -- (LDM2);
			\draw[->] (m2) -- (hhM2);
			\draw[->] (LH1) -- (m11);
			\draw[->] (LH2) -- (m22);
			\end{scope}
			
			\end{tikzpicture}
		\end{center}
		\caption{\footnotesize Two frameworks for learning hypotheses via Learning Spaces. (a) A sample of size $N+M$ is split into two, one of size $N$ that is used to estimate $\hat{\mathcal{M}}$ by minimization of $\hat{L}$ on $\mathbb{L}(\mathcal{H})$, and another of size $M$ to learn a hypothesis from $\hat{\mathcal{M}}$ by ERM. (b) The whole sample of size $N$ is used for estimating $\hat{\mathcal{M}}$ via the minimization of $\hat{L}$ on $\mathbb{L}(\mathcal{H})$ and to estimate hypotheses on $\hat{\mathcal{M}}$ via ERM.}
		\label{learn_hyp}
	\end{figure*}
	
	On the one hand, if the available sample is \textit{great enough}, then we may split it into $\mathcal{D}_{N}$ and $\tilde{\mathcal{D}}_{M}$, with \textit{great size} themselves, and learn with independent sample. On the other hand, if few samples are available, it could be better to learn by reusing, even if such framework is biased, since dividing the sample into two would generate even smaller samples. In the next section, we discuss what a sample \textit{great enough} means in these cases and better quantify the qualities and pitfalls of each framework.
	
	Besides type II and IV (cf. \eqref{ee23}), there is another estimation error that depends on the algorithm one chooses to learn on $\hat{\mathcal{M}}$. The type I estimation error is defined as
	\begin{align}
	\label{typeIe}
	\textbf{(I)} \begin{cases}
	\sup\limits_{h \in \hat{\mathcal{M}}} \left|L_{\tilde{\mathcal{D}}_{M}}(h) - L(h)\right| & \text{if learning with independent sample}\\
	\sup\limits_{h \in \hat{\mathcal{M}}} \left|L_{\mathcal{D}_{N}}(h) - L(h)\right| & \text{if learning by reusing}
	\end{cases},
	\end{align}
	which represents how well one can estimate the loss uniformly on $\hat{\mathcal{M}}$ by the empirical error under $\tilde{\mathcal{D}}_{M}$ or $\mathcal{D}_{N}$. In a posterior step, after $\hat{\mathcal{M}}$ is selected, one may wish to estimate the loss of the hypotheses in it, and how well this task is accomplished is measured by type I estimation error depending on how it is performed, by either reusing the same sample employed to obtain $\hat{\mathcal{M}}$ or by using an independent sample. 
	
	\section{Consistency of learning hypotheses spaces}
	\label{SecConsistency}
	
	The desired properties of a Model Selection framework are asymptotic zero estimation errors of learning on the selected model and convergence to the target model with probability 1. These properties are what we call consistency.
	
	\begin{definition} \normalfont
		\textbf{(Consistency)} A Model Selection framework is consistent if it returns a random model $\hat{\mathcal{M}}$ and an estimated hypothesis $\hat{h}(\mathbb{A}) \in \hat{\mathcal{M}}$ such that types I, II, III and IV estimation errors of learning on it converge in probability to zero, and $\hat{\mathcal{M}}$ converges to $\mathcal{M}^{\star}$ with probability one, when the sample size tends to infinity.
	\end{definition}
	
	\begin{remark}
		\normalfont
		The convergence of $\hat{\mathcal{M}}$ to $\mathcal{M}^{\star}$ implies the convergence to zero of type III estimation error. Hence, as type IV estimation error reduces to type II when $\hat{\mathcal{M}} = \mathcal{M}^{\star}$, the non-trivial conditions for consistency are convergence in probability to zero of types I and II estimation errors, and convergence of $\hat{\mathcal{M}}$ to $\mathcal{M}^{\star}$ with probability one. In some cases, depending on how the algorithm $\mathbb{A}$ is chosen to learn on $\hat{\mathcal{M}}$, convergence of type I estimation error implies convergence of type II, so convergence of type II estimation error may also be trivial (see \cite[Lemma~8.2]{devroye1996}).
	\end{remark}
	
	To show the consistency of Model Selection via Learning Spaces one should bound the tail probability of types I, II, III and IV estimation errors, implying their convergence to zero, and assert the convergence of $\hat{\mathcal{M}}$ to $\mathcal{M}^{\star}$. We start by showing the convergence of $\hat{\mathcal{M}}$ to $\mathcal{M}^{\star}$ with probability one.
	
	\subsection{Convergence to the target hypotheses space}
	\label{SecConvTM}
	
	In order to have $L(\hat{\mathcal{M}}) = L(\mathcal{M}^{\star})$, one does not need to know exactly $L(\mathcal{M})$ for all $\mathcal{M} \in \mathbb{L}(\mathcal{H})$, i.e., one does not need $\hat{L}(\mathcal{M}) = L(\mathcal{M}), \forall \mathcal{M} \in \mathbb{L}(\mathcal{H})$. We argue that it suffices to have $\hat{L}(\mathcal{M})$ close enough to $L(\mathcal{M})$ for all $\mathcal{M} \in \mathbb{L}(\mathcal{H})$ so the global minimums of $\mathbb{L}(\mathcal{H})$ will have the same error as $\mathcal{M}^{\star}$, even if it is not possible to properly estimate their error. This ``close enough'' depends on $P$, hence is not distribution-free, and is given by the \textit{maximum discrimination error} (MDE) of $\mathbb{L}(\mathcal{H})$ under $P$ defined as
	\begin{linenomath}
		\begin{equation*}
		\epsilon^{\star} = \epsilon^{\star}(\mathbb{L}(\mathcal{H}),P) \coloneqq \min\limits_{\substack{\mathcal{M} \in \mathbb{L}(\mathcal{H})\\L(\mathcal{M}) > L(\mathcal{M}^{\star}) }} L(\mathcal{M}) - L(\mathcal{M}^{\star}).
		\end{equation*}
	\end{linenomath}
	The MDE is the minimum difference between the out-of-sample error of a global target hypothesis and the best hypothesis in a model which does not contain a global target. In other words, it is the difference between the error of the best model $\mathcal{M}^{\star}$ and the second to best. The meaning of $\epsilon^{\star}$ is depicted in Figure \ref{epsilonstar}.
	
	\begin{figure}[ht]
		\begin{center}
			\begin{tikzpicture}[scale=0.5]
			
			\coordinate (y) at (-3.5,0);
			\coordinate (x) at (21,0);
			\draw[->] (y) -- (0,0) --  (x) node[right]
			{$L$};
			
			\draw[solid] (1,-0.2) -- (1,0.2);
			\draw[solid] (9,-0.2) -- (9,0.2);
			\draw[solid] (16,-0.2) -- (16,0.2);
			
			\draw[solid,red] (3,-0.2) -- (3,0.2);
			\draw[solid,red] (-1,-0.2) -- (-1,0.2);
			\draw[solid,blue] (6,-0.2) -- (6,0.2);
			\draw[solid,orange] (12.5,-0.2) -- (12.5,0.2);
			
			\draw[dashed] (1,0) -- (1,2);
			\draw[dashed] (9,0) -- (9,2);
			
			\draw [decorate,decoration={brace,amplitude=10pt},xshift=0pt,yshift=0pt,line width=1pt]
			(1,2) -- (9,2) node [black,midway,yshift=17.5pt] {$\epsilon^{\star}$};
			
			\node at (1,-0.75) {$L(\mathcal{M}^{\star}) = L(\mathcal{M}_{1})$};
			\node at (9,-0.75) {$L(\mathcal{M}_{2})$};
			\node at (16,-0.75) {$L(\mathcal{M}_{3})$};
			
			\node at (3,1) {\color{red} $\hat{L}(\mathcal{M}^{\star})$};
			\node at (-1,1) {\color{red} $\hat{L}(\mathcal{M}_{1})$};
			\node at (6,1) {\color{blue} $\hat{L}(\mathcal{M}_{2})$};
			\node at (12.5,1) {\color{orange} $\hat{L}(\mathcal{M}_{3})$};
			
			\draw [red,decorate, 
			decoration = {brace},line width=0.5pt] (-3,-0.5) --  (-3,0.5);
			\draw [red,decorate, 
			decoration = {brace,mirror},line width=0.5pt,xshift=-3.5pt] (5,-0.5) --  (5,0.5);
			\draw [blue,decorate, 
			decoration = {brace},line width=0.5pt,xshift=3.5pt] (5,-0.5) --  (5,0.5);
			\draw [blue,decorate, 
			decoration = {brace,mirror},line width=0.5pt] (13,-0.5) --  (13,0.5);
			\draw [orange,decorate, 
			decoration = {brace},line width=0.5pt] (12,-0.5) --  (12,0.5);
			\draw [orange,decorate, 
			decoration = {brace,mirror},line width=0.5pt] (20,-0.5) --  (20,0.5);
			
			\draw[solid,red,line width=0.75pt] (-3.1,0) -- (5,0);
			\draw[solid,blue,line width=0.75pt] (5,0) -- (13.1,0);
			\draw[solid,orange,line width=0.75pt] (13.1,0.0) -- (20.1,0);
			\draw[solid,orange,line width=0.75pt] (11.9,0.1) -- (13.1,0.1);
			\end{tikzpicture}
		\end{center}
		\caption{\footnotesize The errors of the equivalence classes (cf. \eqref{equiv_class}) of $\mathbb{L}(\mathcal{H})$ in ascending order. The MDE $\epsilon^{\star}$ is the difference between the error of the target class $\mathcal{M}^{\star}$, and the second to best $\mathcal{M}_{2}$. The colored intervals represent a distance of at most $\epsilon^{\star}/2$ from the real error of each model, and the colored estimated errors $\hat{L}$ illustrate a case such that the estimated error is within $\epsilon^{\star}/2$ of the real error for all models. The class $\mathcal{M}_{1}$ has the same error as $\mathcal{M}^{\star}$, but has a smaller estimated error, and, by the definition of $\mathcal{M}^{\star}$, greater VC dimension. Note from the representation that if one can estimate $\hat{L}$ within a margin of error of $\epsilon^{\star}/2$, then $\hat{\mathcal{M}}$ will be a model with the same error as $\mathcal{M}^{\star}$, in this case $\mathcal{M}_{1}$ (cf. Proposition \ref{proposition_principal}).} \label{epsilonstar}
	\end{figure}
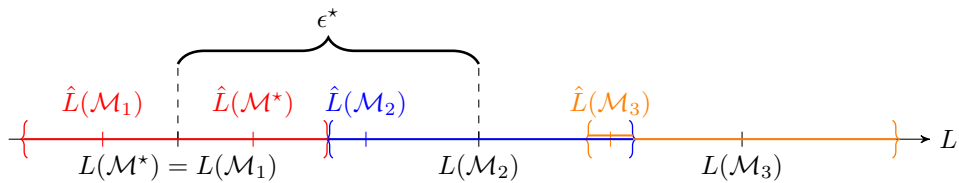
	
	The MDE is greater than zero if there exists at least one $\mathcal{M} \in \mathbb{L}(\mathcal{H})$ such that $h^{\star} \cap \mathcal{M} = \emptyset$, i.e., there is a subset in $\mathbb{L}(\mathcal{H})$ which does not contain a target hypothesis. If  $h^{\star} \cap \mathcal{M} \neq \emptyset$ for all $\mathcal{M} \in \mathbb{L}(\mathcal{H})$, then type III estimation error is zero, and type IV reduces to type II. From this point, we assume $\epsilon^{\star} > 0$.
	
	The terminology MDE is used for we can show that a fraction of $\epsilon^{\star}$ is the greatest error one can commit when estimating $L(\mathcal{M}_{i})$ by $\hat{L}(\mathcal{M}_{i})$ for all $i$ in order for $L(\mathcal{\hat{M}})$ to be equal to $L(\mathcal{M}^{\star})$. This is the result of the next proposition.
	
	\begin{proposition}
		\label{proposition_principal}
		Suppose there exists $\delta > 0$ such that
		\begin{equation}
		\label{cond_prop_principal}
		\mathbb{P}\Big(\sup_{i \in \mathcal{J}} \left|L(\mathcal{M}_{i}) - \hat{L}(\mathcal{M}_{i})\right| < \epsilon^{\star}/2\Big) \geq 1 - \delta.
		\end{equation}
		Then
		\begin{equation}
		\label{prob_equal}
		\mathbb{P}\Big(L(\mathcal{\hat{M}}) = L(\mathcal{M}^{\star})\Big) \geq 1-\delta.
		\end{equation}
	\end{proposition}
	
	\begin{remark}
		\normalfont
		Since there may exist $\mathcal{M} \in \ \nicefrac{\mathbb{L}(\mathcal{H})}{\sim}$ with $L(\mathcal{M}) = L(\mathcal{\mathcal{M}^{\star}})$ and $d_{VC}(\mathcal{M}) > d_{VC}(\mathcal{M}^{\star})$, condition \eqref{cond_prop_principal} guarantees only that the estimated error of both $\mathcal{M}$ and $\mathcal{M}^{\star}$ is lesser than the estimated error of any model with error greater than theirs, but there may happen that $\hat{L}(\mathcal{M}) < \hat{L}(\mathcal{M}^{\star})$ (see Figure \ref{epsilonstar} for an example).
	\end{remark}
	
	From now on, we consider that $\hat{L}$ is of the form
	\begin{equation}
	\label{form_Lhat}
	\hat{L}(\mathcal{M}) = \frac{1}{m} \ \sum_{j=1}^{m} \ \hat{L}^{(j)}(\hat{h}^{(j)}_{\mathcal{M}}),
	\end{equation}
	in which there are $m$ pairs of independent training and validation samples, $\hat{L}^{(j)}$ is the empirical error under the $j$-th validation sample and $\hat{h}^{(j)}_{\mathcal{M}}$ is the ERM hypothesis of $\mathcal{M}$ under the $j$-th training sample, denoted by $\mathcal{D}_{N}^{(j)}$. We assume independence between samples within a pair $j$, but there may exist dependence between samples of distinct pairs $j,j^{\prime}$. The validation sample and k-fold cross validation estimators presented in Section \ref{esti_Lhat} are of the form \eqref{form_Lhat} with $m = 1$ and $m = k$, respectively. In this case, we may obtain a bound for \eqref{prob_equal} depending on $\epsilon^{\star}$, on $d_{VC}(\mathbb{L}(\mathcal{H}))$, and on bounds for tail probabilities of type I estimation error under each validation and training sample.
	
	These bounds also depend on the number of maximal models of $\mathbb{L}(\mathcal{H})$, which are models in
	\begin{align*}
	\text{Max } \mathbb{L}(\mathcal{H}) = \left\{\mathcal{M} \in \mathbb{L}(\mathcal{H}): \mathcal{M} \subset \mathcal{M}^{\prime}, \ \mathcal{M}^{\prime} \in \mathbb{L}(\mathcal{H}) \implies \mathcal{M} = \mathcal{M}^{\prime} \right\},
	\end{align*}
	that are models which are not contained in any element in $\mathbb{L}(\mathcal{H})$ beside themselves. We denote $\mathfrak{m}(\mathbb{L}(\mathcal{H})) = |\text{Max } \mathbb{L}(\mathcal{H})|$. If $\mathbb{L}(\mathcal{H})$ is a complete lattice, then the only maximal element of it is its greatest element, so $\mathfrak{m}(\mathbb{L}(\mathcal{H})) = 1$. We have the following rate of convergence of $L(\hat{\mathcal{M}})$ to $L(\mathcal{M}^{\star})$, and condition for $\hat{\mathcal{M}}$ to converge to $\mathcal{M}^{\star}$ with probability one.
	
	\begin{theorem}
		\label{theorem_principal_convergence}
		For each $\epsilon > 0$ let $\{B_{N,\epsilon}: N \geq 1\}$ and $\{\hat{B}_{N,\epsilon}: N \geq 1\}$ be sequences of positive real-valued increasing functions with domain $\mathbb{Z}_{+}$ satisfying $\lim\limits_{N \to \infty} B_{N,\epsilon}(k) = \lim\limits_{N \to \infty} \hat{B}_{N,\epsilon}(k) = 0$, for all $\epsilon > 0$ and $k \in \mathbb{Z}_{+}$ fixed, and such that
		\begin{linenomath}
			\begin{align}
			\label{maxj}
			\max_{j} \mathbb{P}\Big(\sup\limits_{h \in \mathcal{M}} \big|L_{\mathcal{D}_{N}^{(j)}}(h) - L(h)\big| > \epsilon \Big) \leq B_{N,\epsilon}(d_{VC}(\mathcal{M})) & \text{ and } &\\ \nonumber
			\max_{j} \mathbb{P}\Big(\sup\limits_{h \in \mathcal{M}} \big|\hat{L}^{(j)}(h) - L(h)\big| > \epsilon \Big) \leq \hat{B}_{N,\epsilon}(d_{VC}(\mathcal{M})), & & 
			\end{align}
		\end{linenomath}
		for all $\mathcal{M} \in \mathbb{L}(\mathcal{M})$, recalling that $L_{\mathcal{D}_{N}^{(j)}}$ and $\hat{L}^{(j)}$ represent the empirical error under the $j$-th training and validation samples, respectively. Let $\mathcal{\hat{M}} \in \mathbb{L}(\mathcal{H})$ be a random model learned by $\mathbb{M}_{\mathbb{L}(\mathcal{H})}$. Then,
		\begin{align}
		\label{bound_pM}
		\mathbb{P}\left(L(\hat{\mathcal{M}}) \neq L(\mathcal{M}^{\star})\right) &\leq m \ \mathfrak{m}(\mathbb{L}(\mathcal{H})) \left[\hat{B}_{N,\epsilon^{\star}/4}(d_{VC}(\mathbb{L}(\mathcal{H}))) + B_{N,\epsilon^{\star}/8}(d_{VC}(\mathbb{L}(\mathcal{H})))\right],
		\end{align}
		in which $m$ is the number of pairs considered to calculate \eqref{form_Lhat}. Furthermore, if
		\begin{align}
		\label{as_conv}
		\max_{\mathcal{M} \in \mathbb{L}(\mathcal{H})} \max_{j} \sup\limits_{h \in \mathcal{M}} \big|L_{\mathcal{D}_{N}^{(j)}}(h) - L(h)\big| \xrightarrow{\text{a.s.}} 0 & \text{ and } &\\ \nonumber
		\max_{\mathcal{M} \in \mathbb{L}(\mathcal{H})} \max_{j} \sup\limits_{h \in \mathcal{M}} \big|\hat{L}^{(j)}(h) - L(h)\big| \xrightarrow{\text{a.s.}} 0, & & 
		\end{align}
		then
		\begin{equation*}
		\lim_{N \rightarrow \infty} \mathbb{P}\left(\hat{\mathcal{M}} = \mathcal{M}^{\star}\right) = 1.
		\end{equation*}
	\end{theorem}
	
	\begin{remark}
		\normalfont
		A bound for $\mathbb{P}(L(\hat{\mathcal{M}}) \neq L(\mathcal{M}^{\star}))$ in the case of k-fold cross validation may be obtained from the bounds for \eqref{maxj} given in, for example, \cite[Theorem~4.4]{vapnik1998}, recalling that the sample size in each training and validation sample is $(k-1)n$ and $n$, respectively, with $N = kn$. From these bounds, which are exponential on $N$, and Borel-Cantelli lemma \cite[Theorem~4.3]{billingsley2008}, the almost sure convergences \eqref{as_conv} follow in this case, hence $\hat{\mathcal{M}}$ converges to $\mathcal{M}^{\star}$ with probability one.  Analogously, we may obtain a bound when an independent validation sample is considered.
	\end{remark}
	
	From bound \eqref{bound_pM} we first see that we have to better estimate with the training samples, that require a precision of $\epsilon^{\star}/8$ in contrast to a precision $\epsilon^{\star}/4$ with the validation samples. Hence, as is done in k-fold cross validation, it is better to consider a greater sample size for training rather than for validation. Moreover, from this bound follows that, with a fixed sample size, we can have a tighter bound for $\mathbb{P}(L(\hat{\mathcal{M}}) \neq L(\mathcal{M}^{\star}))$ by choosing a Learning Space with small $d_{VC}(\mathbb{L}(\mathcal{H}))$ and few maximal elements, while attempting to increase $\epsilon^{\star}$. Of course, there is a trade-off between $d_{VC}(\mathbb{L}(\mathcal{H}))$ and the number of maximal elements of $\mathbb{L}(\mathcal{H})$, the only known free quantities in bound $\eqref{bound_pM}$, since the sample size is fixed and $\epsilon^{\star}$ is unknown.
	
	As an illustrative example, let $\mathbb{L}(\mathcal{H})$ be the Partition Lattice Learning Space and
	\begin{align*}
	\mathbb{L}_{2}(\mathcal{H}) \coloneqq \{\mathcal{M} \in \mathbb{L}(\mathcal{H}): d_{VC}(\mathcal{M}) <= 2\}
	\end{align*}
	be its models with VC dimension not greater than two. The collection $\mathbb{L}_{2}(\mathcal{H})$ has $2^{|\mathcal{X}| - 1}$ models, that is the number of partitions of $\mathcal{X}$ with at most two blocks, and is a Learning Space, since condition (ii) is inherited from $\mathbb{L}(\mathcal{H})$, and it covers $\mathcal{H}$: any given $h \in \mathcal{H}$ is in the model generated by partition $\{\{x \in \mathcal{X}: h(x) = 0\},\{x \in \mathcal{X}: h(x) = 1\}\}$ which has at most two blocks. 
	
	On the one hand, since $\mathbb{L}(\mathcal{H})$ is a complete lattice, it has only one maximal element, and $d_{VC}(\mathbb{L}(\mathcal{H})) = d_{VC}(\mathcal{H})$ since its maximal element is $\mathcal{H}$. On the other hand, $\mathfrak{m}(\mathbb{L}_{2}(\mathcal{H})) = 2^{|\mathcal{X}| - 1} - 1$ since the only element in it that is not maximal is the model of the constant hypotheses, and $d_{VC}(\mathbb{L}_{2}(\mathcal{H})) = 2$ by definition. Furthermore, $\mathcal{M}^{\star}$, which has VC dimension at most 2, and $\epsilon^{\star}$, are the same on both spaces.
	
	The form of bound \eqref{bound_pM} may dictate on which of these spaces we can have the tightest bound for $\mathbb{P}(L(\hat{\mathcal{M}}) \neq L(\mathcal{M}^{\star}))$, so may guide the choice of the Learning Space in this scenario. Nevertheless, in practice, it is important to also consider the computational complexity of $\hat{\mathcal{M}}$ in each instance. For this particular case, we discuss in Section \ref{U-curve} a non-exhaustive algorithm to compute $\hat{\mathcal{M}}$ in $\mathbb{L}(\mathcal{H})$, while we need an exhaustive search of $\mathbb{L}_{2}(\mathcal{H})$ for this task. Therefore, the choice of a Learning Space should mind the consistency of $\hat{\mathcal{M}}$ and all prior information about the problem at hand, but also the computational aspect that enables the application of the method.
	
	The bound of Theorem \ref{theorem_principal_convergence} is the first result which supports that by properly modeling the Learning Space one may better learn on $\mathcal{H}$, in this instance by having a greater probability of learning on a model with the same error as $\mathcal{M}^{\star}$, the best model in $\mathbb{L}(\mathcal{H})$. In the next section, we develop bounds for types I, II, III, and IV estimation errors on $\hat{\mathcal{M}}$ which also support this paradigm.
	
	\subsection{Convergence of estimation errors on $\hat{\mathcal{M}}$}
	
	The type III estimation error on $\hat{\mathcal{M}}$ depend solely on $\hat{\mathcal{M}}$, while types I, II and IV depend on $\hat{\mathcal{M}}$, but also on the choice of algorithm $\mathbb{A}$ employed to learn a hypotheses $\hat{h}_{\hat{\mathcal{M}}}(\mathbb{A}) \in \hat{\mathcal{M}}$ (cf. \eqref{ee23} and \eqref{typeIe}). In this section, we consider two possible algorithms, described in Figure \ref{learn_hyp} as learning by reusing, in which we consider the ERM hypothesis of $\hat{\mathcal{M}}$ under sample $\mathcal{D}_{N}$, and learning with independent sample, in which we consider the ERM hypothesis of $\hat{\mathcal{M}}$ under sample $\tilde{\mathcal{D}}_{M}$, independent of $\mathcal{D}_{N}$. We start by discussing in detail the case of an independent sample, and then briefly discuss learning by reusing.
	
	\subsubsection{Learning with independent sample}
	
	Bounds for types I and II estimation errors when learning on a random model with a sample independent of the one employed to compute such random model, may be obtained when there is a bound for them on each $\mathcal{M} \in \mathbb{L}(\mathcal{H})$ under the independent sample. This is the content of Theorem \ref{bound_constant}.
	
	\begin{theorem}
		\label{bound_constant}	
		Assume we are learning with an independent sample $\tilde{\mathcal{D}}_{M}$ and that for each $\epsilon > 0$ there exist sequences $\{B^{I}_{M,\epsilon}: M \geq 1\}$ and $\{B^{II}_{M,\epsilon}: M \geq 1\}$ of positive real-valued increasing functions with domain $\mathbb{Z}_{+}$ satisfying $\lim\limits_{M \to \infty} B^{I}_{M,\epsilon}(k) = \lim\limits_{M \to \infty} B^{II}_{M,\epsilon}(k) = 0$, for all $\epsilon > 0$ and $k \in \mathbb{Z}_{+}$ fixed, such that
		\begin{linenomath}
			\begin{align}
			\label{bound_theoremBC}
			\mathbb{P}\Big(\sup\limits_{h \in \mathcal{M}} \big|L_{\tilde{\mathcal{D}}_{M}}(h) - L(h) \big| > \epsilon \Big) \leq B^{I}_{M,\epsilon}(d_{VC}(\mathcal{M})) & \text{ and } &\\ \nonumber
			\mathbb{P}\Big(L(\hat{h}_{\mathcal{M}}(\tilde{\mathcal{D}}_{M})) - L(h^{\star}_{\mathcal{M}}) > \epsilon \Big) \leq B^{II}_{M,\epsilon}(d_{VC}(\mathcal{M})), & & 
			\end{align}
		\end{linenomath}
		for any $\mathcal{M} \in \mathbb{L}(\mathcal{H})$. Let $\mathcal{\hat{M}} \in \mathbb{L}(\mathcal{H})$ be a random model learned by $\mathbb{M}_{\mathbb{L}(\mathcal{H})}$. Then, for any $\epsilon > 0$,
		\begin{linenomath}
			\begin{align*}
			\textbf{(I)} \ \mathbb{P}\Big(&\sup\limits_{h \in \mathcal{\hat{M}}} \big|L_{\tilde{\mathcal{D}}_{M}}(h) - L(h) \big| > \epsilon \Big) \leq \mathbb{E}\Big[B^{I}_{M,\epsilon}(d_{VC}(\mathcal{\hat{M}}))\Big] \leq B^{I}_{M,\epsilon}\left(d_{VC}(\mathbb{L}(\mathcal{H}))\right)
			\end{align*}
		\end{linenomath}
		and
		\begin{linenomath}
			\begin{align*}
			\textbf{(II)} \ \mathbb{P}\Big(L(\hat{h}_{\mathcal{\hat{M}}}(\tilde{\mathcal{D}}_{M})) - L(h^{\star}_{\mathcal{\hat{M}}}) > \epsilon \Big) \leq \mathbb{E}\Big[B^{II}_{M,\epsilon}(d_{VC}(\mathcal{\hat{M}}))\Big] \leq B^{II}_{M,\epsilon}\left(d_{VC}(\mathbb{L}(\mathcal{H}))\right),
			\end{align*}
		\end{linenomath}
		in which the expectations are over all samples $\mathcal{D}_{N}$, from which $\hat{\mathcal{M}}$ is calculated. Since $d_{VC}(\mathbb{L}(\mathcal{H})) < \infty$, both probabilities above converge to zero when $M \to \infty$.
	\end{theorem}
	
	Our definition of $\mathcal{\hat{M}}$ ensures that it is going to have the smallest VC dimension under the constraint that it is a global minimum of $\mathbb{L}(\mathcal{H})$. As the quantities inside the expectations of Theorem \ref{bound_constant} are increasing functions of VC dimension, fixed $\epsilon$ and $M$, we tend to have smaller expectations, thus tighter bounds for types I and II estimation errors. Furthermore, we conclude from the bounds of Theorem \ref{bound_constant} that the sample complexity needed to learn on $\hat{\mathcal{M}}$ is at most that of $d_{VC}(\mathbb{L}(\mathcal{H}))$. This implies that this complexity is at most that of $\mathcal{H}$, but may be much less if $d_{VC}(\mathbb{L}(\mathcal{H})) \ll d_{VC}(\mathcal{H})$. We conclude that the bounds for the tail probabilities of types I and II estimation errors on $\hat{\mathcal{M}}$ are tighter than that on $\mathcal{H}$ and the sample complexity needed to learn on $\hat{\mathcal{M}}$ is at most that of $\mathbb{L}(\mathcal{H})$, and not of $\mathcal{H}$. 
	
	However, even when these inequalities guarantee the consistency of $\hat{\mathcal{M}}$ regarding types I and II estimation errors, it is still necessary to check that types III and IV estimation errors are small to attest the feasibility of learning on $\hat{\mathcal{M}}$: if $L(h^{\star}_{\hat{\mathcal{M}}})$ is too greater than $L(h^{\star})$, well estimating $h^{\star}_{\hat{\mathcal{M}}}$ (small type II estimation error) is useless, so having small types I and II estimation errors is not enough to properly approximate $h^{\star}$, that is the main objective of the learning problem.
	
	A bound for type III estimation error may be obtained using methods similar to that we employed to prove Theorem \ref{theorem_principal_convergence}. As in that theorem, the bound for type III estimation error depends on $\epsilon^{\star}$, on bounds for type I estimation error under each training and validation sample for every $\mathcal{M} \in \mathbb{L}(\mathcal{H})$, and on $\mathbb{L}(\mathcal{H})$, more specifically, on its VC dimension and number of maximal elements. To ease notation, we denote $\epsilon \vee \epsilon^{\star} \coloneqq \max \{\epsilon,\epsilon^{\star}\}$ for any $\epsilon > 0$.
	
	\begin{theorem}
		\label{theorem_tipeIII}
		Assume the premises of Theorem \ref{theorem_principal_convergence} are in force. Let $\mathcal{\hat{M}} \in \mathbb{L}(\mathcal{H})$ be a random model learned by $\mathbb{M}_{\mathbb{L}(\mathcal{H})}$. Then, for any $\epsilon > 0$,
		\begin{align*}
		\textbf{(III)} \ \mathbb{P}\left(L(h_{\hat{\mathcal{M}}}^{\star}) - L(h^{\star}) > \epsilon\right) \leq m \ \mathfrak{m}(\mathbb{L}(\mathcal{H})) \left[\hat{B}_{N,(\epsilon \vee \epsilon^{\star})/4}(d_{VC}(\mathbb{L}(\mathcal{H}))) + B_{N,(\epsilon \vee \epsilon^{\star})/8}(d_{VC}(\mathbb{L}(\mathcal{H})))\right].
		\end{align*}
		In particular,
		\begin{align*}
		\lim_{N \rightarrow \infty} \mathbb{P}\left(L(h_{\hat{\mathcal{M}}}^{\star}) - L(h^{\star}) > \epsilon\right) = 0,
		\end{align*}
		for any $\epsilon > 0$.
	\end{theorem}
	
	\begin{remark}
		\label{remReuse} \normalfont
		Type III estimation error, and its bound presented in Theorem \ref{theorem_tipeIII}, do not depend on the algorithm $\mathbb{A}$ employed to learn on $\hat{\mathcal{M}}$, hence this theorem is true for both frameworks in Figure \ref{learn_hyp}, holding also when learning by reusing.
	\end{remark}
	
	On the one hand, by definition of $\epsilon^{\star}$, if $\epsilon < \epsilon^{\star}$, then type III estimation error is lesser than $\epsilon$ if, and only if, $L(\hat{\mathcal{M}}) = L(\mathcal{M}^{\star})$, so this error is actually zero, and the result of Theorem \ref{theorem_principal_convergence} is a bound for type III estimation error in this case. On the other hand, if $\epsilon > \epsilon^{\star}$, one way of having type III estimation error lesser than $\epsilon$ is to have the estimated error of each $\mathcal{M}$ at a distance of at most $\epsilon/2$ from its real error and, as can be inferred from the proof of Theorem \ref{theorem_principal_convergence}, this can be accomplished if one has type I estimation error not greater than a fraction of $\epsilon$ under each training and validation sample considered, so a modification of the result of Theorem \ref{theorem_principal_convergence} also applies to this case.
	
	Finally, as the tail probability of type IV estimation error may be bounded by the following inequality, involving the tail probabilities of types II and III estimation errors,
	\begin{linenomath}
		\begin{align}
		\label{triangle}
		\mathbb{P}\Big(&L(\hat{h}_{\mathcal{\hat{M}}}(\tilde{\mathcal{D}}_{M})) - L(h^{\star}) > \epsilon\Big) \leq \mathbb{P}\Big(L(\hat{h}_{\mathcal{\hat{M}}}(\tilde{\mathcal{D}}_{M})) - L(h^{\star}_{\mathcal{\hat{M}}}) > \epsilon/2\Big) + \mathbb{P}\Big(L(h^{\star}_{\mathcal{\hat{M}}}) - L(h^{\star}) > \epsilon/2\Big)
		\end{align}
	\end{linenomath}
	a bound on the rate of convergence of this error to zero is a direct consequence of Theorems \ref{bound_constant} and \ref{theorem_tipeIII}.
	
	\begin{corollary}
		\label{cor_typeIV}
		Assume the premises of Theorem \ref{theorem_principal_convergence} and \ref{bound_constant} are in force. Let $\mathcal{\hat{M}} \in \mathbb{L}(\mathcal{H})$ be a random model learned by $\mathbb{M}_{\mathbb{L}(\mathcal{H})}$. Then, for any $\epsilon > 0$,
		\begin{align*}
		&\mathbb{P}\Big(L(\hat{h}_{\mathcal{\hat{M}}}(\tilde{\mathcal{D}}_{M})) - L(h^{\star}) > \epsilon\Big)\\
		& \leq \mathbb{E}\Big[B^{II}_{M,\epsilon/2}(d_{VC}(\mathcal{\hat{M}}))\Big] + m \ \mathfrak{m}(\mathbb{L}(\mathcal{H})) \left[\hat{B}_{N,(\epsilon/2 \vee \epsilon^{\star})/4}(d_{VC}(\mathbb{L}(\mathcal{H}))) + B_{N,(\epsilon/2 \vee \epsilon^{\star})/8}(d_{VC}(\mathbb{L}(\mathcal{H})))\right]\\
		&\leq  B^{II}_{M,\epsilon/2}(d_{VC}(\mathbb{L}(\mathcal{H}))) + m \ \mathfrak{m}(\mathbb{L}(\mathcal{H})) \left[\hat{B}_{N,(\epsilon/2 \vee \epsilon^{\star})/4}(d_{VC}(\mathbb{L}(\mathcal{H}))) + B_{N,(\epsilon/2 \vee \epsilon^{\star})/8}(d_{VC}(\mathbb{L}(\mathcal{H})))\right].
		\end{align*}
		In particular,
		\begin{align*}
		\lim_{\substack{N \rightarrow \infty \\ M \rightarrow \infty}} \mathbb{P}\Big(L(\hat{h}_{\mathcal{\hat{M}}}(\tilde{\mathcal{D}}_{M})) - L(h^{\star}) > \epsilon\Big) = 0,
		\end{align*}
		for any $\epsilon > 0$.
	\end{corollary}
	
	From Theorems \ref{theorem_principal_convergence}, \ref{bound_constant} and \ref{theorem_tipeIII}, and Corollary \ref{cor_typeIV}, follows the consistency of the Model Selection framework given by selecting $\hat{\mathcal{M}}$ via $\mathbb{M}_{\mathbb{L}(\mathcal{H})}$ and learning on it with an independent sample when we consider $\hat{L}$ given by k-fold cross-validation or an independent validation sample. We state this result in the next corollary.
	
	\begin{corollary}	
		\label{corollary}
		The Model Selection framework given by
		\begin{enumerate}
			\item[(a)] estimating $L(\mathcal{M})$ by k-fold cross validation with a fixed $k$ or by an independent validation sample,
			\item[(b)] selecting $\hat{\mathcal{M}}$ via $\mathbb{M}_{\mathbb{L}(\mathcal{H})}$,
			\item[(c)] and learning with an independent sample on $\hat{\mathcal{M}}$,
		\end{enumerate}
		is consistent.	  
	\end{corollary}
	
	From the results above, follows that the Learning Space plays an important role in the rate of convergence of $\mathbb{P}(L(\hat{\mathcal{M}}) = L(\mathcal{M}^{\star}))$ to 1 and of the estimation errors to zero, through $\epsilon^{\star}$ and $d_{VC}(\mathcal{M}^{\star})$. Moreover, these results also shed light on manners of improving the quality of the learning, i.e., decreasing the estimation errors, specially type IV, when the sample size is fixed. We discuss the main features of the consistent framework presented in this section.
	
	First, since $\hat{\mathcal{M}}$ converges to $\mathcal{M}^{\star}$ with probability 1, 
	\begin{linenomath}
		\begin{equation*}
		\mathbb{E}(F(\mathcal{\hat{M}})) \xrightarrow{N \rightarrow \infty} F(\mathcal{M}^{\star})
		\end{equation*}
	\end{linenomath}
	by the Dominated Convergence Theorem, in which $F: \mathbb{L}(\mathcal{H}) \mapsto \mathbb{R}$ is any $(\Omega,\mathcal{S})$-measurable real-valued function, since the domain of $F$ is finite. The convergence of $\mathbb{E}(F(\mathcal{\hat{M}}))$ ensures that the functions of $\hat{\mathcal{M}}$ on the right-hand side of inequalities of Theorem \ref{bound_constant} and Corollary \ref{cor_typeIV} tend to the same functions evaluated at $\mathcal{M}^{\star}$, when $N$ tends to infinity. Hence, if one was able to isolate $h^{\star}$ within a model $\mathcal{M}^{\star}$ with small VC dimension, the bounds for types I, II and IV estimation errors will tend to be tighter for a sample of a given size $N + M$.
	
	Second, if the MDE of $\mathbb{L}(\mathcal{H})$ under $P$ is great, then we need less precision when estimating $L(\mathcal{M})$ for $L(\hat{\mathcal{M}})$ to be equal to $L(\mathcal{M}^{\star})$, and for types III and IV estimation errors to be lesser than a $\epsilon \ll \epsilon^{\star}$ with high probability, so fewer samples are needed to learn a model as good as $\mathcal{M}^{\star}$ and to have lesser types III and IV estimation errors. Moreover, the sample complexity to learn this model is that of the most complex model in $\mathbb{L}(\mathcal{H})$, what implies that this is at most the complexity of a model with VC dimension $d_{VC}(\mathbb{L}(\mathcal{H}))$, which may be lesser than that of $\mathcal{H}$. 
	
	Therefore, by embedding into $\mathbb{L}(\mathcal{H})$ all prior information about $h^{\star}$ and $P$, seeking to increase $\epsilon^{\star}$ and decrease $d_{VC}(\mathcal{M}^{\star})$, one may, with a given sample of size $N + M$, better learn on $\mathcal{H}$, that is, better approximate $h^{\star}$ by a $\hat{h}_{\hat{\mathcal{M}}}(\tilde{\mathcal{D}}_{M})$ (small type IV estimation error). Hence, the results of this section also support that by properly modeling the Learning Space one may better learn on $\mathcal{H}$, which means having small type IV estimation error.
	
	By the deductions above, under the framework detailed in Corollary \ref{corollary}, it follows that all estimation errors converge in probability to zero and that $\hat{\mathcal{M}}$ tends to $\mathcal{M}^{\star}$ with probability one, when $N$ and $M$ tend to infinity. However, we are not able, by making use of the bounds provided by VC theory and extended to $\hat{\mathcal{M}}$ in this case, to find bounds for these estimation errors which do not depend on $\epsilon^{\star}$, and thus on $P$. In other words, we have established a distribution free consistency of the framework, but not a distribution free rate to the considered convergences. 
	
	Although not distribution-free, the convergences proved above reflect an important property of the approach, which may have been overlooked by other methods: the sample complexity does depend on the target hypotheses, in this case through the target model. If one can isolate a target hypothesis inside a \textit{simple} model (see Figure \ref{paradigms}) which also contains $\{h \in \mathcal{H}: L(h) - L(h^\star) \leq \epsilon^{\star}\}$ for a \textit{great} $\epsilon^{\star} >0$, then \textit{few} samples are needed to \textit{properly} estimate this target, independently of its ``complexity'', as $\mathcal{\hat{M}}$ would be equal to $\mathcal{M}^{\star}$ with high probability for a relatively small sample, as $\epsilon^{\star}$ is large, and types I and II estimation errors on $\mathcal{\hat{M}}$ would probably be \textit{small}, as it is with high probability equal to $\mathcal{M}^{\star}$ which is \textit{simple}. 
	
	Hence, without constraining beforehand the hypotheses space $\mathcal{H}$, which contains all hypotheses one is willing to consider, and without gathering more samples to increase a sample of size $N + M$, one could, theoretically, still estimate $h^{\star}$ by a hypothesis which well generalizes, by properly modeling $\mathbb{L}(\mathcal{H})$. Such a modeling should be done by embedding into $\mathbb{L}(\mathcal{H})$ all prior information about $h^{\star}, P$ and the practical problem at hand.
	
	\subsubsection{Learning by reusing}
	
	When learning by reusing, one is employing the same sample points to both estimate $\hat{\mathcal{M}}$ and learn a hypothesis $\hat{h}(\mathcal{D}_{N}) \in \hat{\mathcal{M}}$ from it, so there is a dependence between type I and II estimation errors and the events $\{\hat{\mathcal{M}} = \mathcal{M}\}, \mathcal{M} \in \mathbb{L}(\mathcal{H})$. Indeed, an equality like \eqref{cond_independence} may not be true in this case, that is, we may have
	\begin{align*}
	\mathbb{P}\Big(\sup\limits_{h \in \hat{\mathcal{M}}} \big|L_{\mathcal{D}_{N}}(h) - L(h) \big| > \epsilon \Big|\mathcal{\hat{M}} = \mathcal{M}\Big) \neq \mathbb{P}\Big(\sup\limits_{h \in \mathcal{M}} \big|L_{\mathcal{D}_{N}}(h) - L(h) \big| > \epsilon\Big),
	\end{align*}
	since, conditioned on $\{\hat{\mathcal{M}} = \mathcal{M}\}$, not only the distribution of each sample point $(X_{l},Y_{l}), l = 1,\dots,N$, changes, but also these points are now dependent: they must be such that $\hat{\mathcal{M}} = \mathcal{M}$, hence, cannot be independent. Therefore, the argument of the proof of Theorem \ref{bound_constant} does not hold in this instance.
	
	Nevertheless, since $\hat{\mathcal{M}}$ converges with probability one to $\mathcal{M}^{\star}$ by Theorem \ref{theorem_principal_convergence}, we may obtain a bound for types I and II estimation errors when learning by reusing which depends on such bounds in $\mathcal{M}^{\star}$ and on the rate of convergence of $\hat{\mathcal{M}}$ to $\mathcal{M}^{\star}$.
	
	\begin{theorem}
		\label{bound_constant_reusing}	
		Assume we are learning by reusing and that for each $\epsilon > 0$ there exist sequences $\{B^{I}_{N,\epsilon}: N \geq 1\}$ and $\{B^{II}_{N,\epsilon}: N \geq 1\}$ of positive real-valued increasing functions with domain $\mathbb{Z}_{+}$ satisfying $\lim\limits_{N \to \infty} B^{I}_{N,\epsilon}(k) = \lim\limits_{N \to \infty} B^{II}_{N,\epsilon}(k) = 0$, for all $\epsilon > 0$ and $k \in \mathbb{Z}_{+}$ fixed, such that
		\begin{linenomath}
			\begin{align*}
			\label{bound_theoremBC2}
			\mathbb{P}\Big(\sup\limits_{h \in \mathcal{M}} \big|L_{\mathcal{D}_{N}}(h) - L(h) \big| > \epsilon \Big) \leq B^{I}_{N,\epsilon}(d_{VC}(\mathcal{M})) & \text{ and } &\\ \nonumber
			\mathbb{P}\Big(L(\hat{h}_{\mathcal{M}}(\mathcal{D}_{N})) - L(h^{\star}_{\mathcal{M}}) > \epsilon \Big) \leq B^{II}_{N,\epsilon}(d_{VC}(\mathcal{M})), & & 
			\end{align*}
		\end{linenomath}
		for any $\mathcal{M} \in \mathbb{L}(\mathcal{H})$. Let $\mathcal{\hat{M}} \in \mathbb{L}(\mathcal{H})$ be a random model learned by $\mathbb{M}_{\mathbb{L}(\mathcal{H})}$. Then, for any $\epsilon > 0$,
		\begin{linenomath}
			\begin{align*}
			\textbf{(I)} \ \mathbb{P}\Big(&\sup\limits_{h \in \mathcal{\hat{M}}} \big|L_{\mathcal{D}_{N}}(h) - L(h) \big| > \epsilon \Big) \leq B^{I}_{N,\epsilon}(d_{VC}(\mathcal{M}^{\star})) + \mathbb{P}\left(\hat{\mathcal{M}} \neq \mathcal{M}^{\star}\right)
			\end{align*}
		\end{linenomath}
		and
		\begin{linenomath}
			\begin{align*}
			\textbf{(II)} \ \mathbb{P}\Big(L(\hat{h}_{\mathcal{\hat{M}}}(\mathcal{D}_{N})) - L(h^{\star}_{\mathcal{\hat{M}}}) > \epsilon \Big) \leq B^{II}_{N,\epsilon}(d_{VC}(\mathcal{M}^{\star})) + \mathbb{P}\left(\hat{\mathcal{M}} \neq \mathcal{M}^{\star}\right).
			\end{align*}
		\end{linenomath}
		If conditions \eqref{as_conv} of Theorem \ref{theorem_principal_convergence} are satisfied, both probabilities above converge to zero when $N \to \infty$.
	\end{theorem}
	
	By triangle inequality \eqref{triangle} and the bound for type III estimation error established in Theorem \ref{theorem_tipeIII}, that is also true when learning by reusing (cf. Remark \ref{remReuse}), we have that the tail probability of type IV estimation error converges to zero as $N$ tends to infinity, a result analogous to Corollary \ref{cor_typeIV}. Hence, it is consistent to learn by reusing.
	
	\begin{corollary}	
		\label{corollary_reuse}
		The Model Selection framework given by
		\begin{enumerate}
			\item[(a)] estimating $L(\mathcal{M})$ by k-fold cross validation with a fixed $k$ or by an independent validation sample,
			\item[(b)] selecting $\hat{\mathcal{M}}$ via $\mathbb{M}_{\mathbb{L}(\mathcal{H})}$,
			\item[(c)] and learning by reusing on $\hat{\mathcal{M}}$,
		\end{enumerate}
		is consistent.	  
	\end{corollary}
	
	The bounds for types I and II estimation errors in Theorem \ref{bound_constant_reusing} outline that, if $\mathcal{M}^{\star}$ has a \textit{small} VC dimension and $\mathbb{M}_{\mathbb{L}(\mathcal{H})}$ is such that $\hat{\mathcal{M}} = \mathcal{M}^{\star}$ with \textit{high} probability, then one can learn by reusing and still \textit{properly} estimate on $\hat{\mathcal{M}}$. This result also supports the paradigm of, by properly modeling $\mathbb{L}(\mathcal{H})$ seeking to have the target $\mathcal{M}^{\star}$ with small VC dimension, one can better estimate with a fixed sample of size $N$. As was also the case for learning with an independent sample, we have established the distribution-free consistency of learning by reusing, but have not obtained a distribution-free rate for the convergence of the estimation errors. 
	
	A drawback of learning by reusing is the selection bias of $\hat{h}(\mathcal{D}_{N})$ which increases the risk of overfitting. This has been very well empirically studied in \cite{cawley2010} for some specific cases. This may or may not be an issue, and further empirical studies are needed to better understand when it will be. Nevertheless, this approach could in theory be better when the sample size available is not \textit{great enough}, since the drawback of learning with an independent sample is the need for large sample sizes $N$ and $M$ (cf. Corollary \ref{cor_typeIV}). 
	
	\begin{remark}
		\normalfont
		We believe that bounds for the expectation of types I, and specially II, estimation errors when learning by reusing may be established for specific Learning Spaces via the so-called oracle inequalities, discussed in \cite{lecue2012,van2006,mitchell2009}. We leave this interesting topic for future researches.
	\end{remark}
	
	\section{U-curve property}
	\label{SecUproperty}
	
	As important as the consistency of a Model Selection framework via Learning Spaces, is the possibility of computing $\hat{\mathcal{M}}$ for problems of interest. As can be noted in the examples of Section \ref{SecExamples}, the cardinality of a Learning Space might be (more than) exponential on the number of parameters representing the hypotheses in $\mathcal{H}$, making an exhaustive search of it unfeasible. In this section, we discuss a property which, when satisfied by a Learning Space, may be used to develop non-exhaustive algorithms (discussed in Section \ref{U-curve}) to compute $\hat{\mathcal{M}}$. We first define the U-curve property, then show that it is satisfied by the Partition Lattice Learning Space, and establish a sufficient condition for it that shed light on what it actually means, by drawing a parallel to convexity.
	
	The Model Selection approach based on Learning Spaces is a solution of optimization problem
	\begin{linenomath}
		\begin{equation}
		\label{ucurve_return}
		\mathcal{\hat{M}} \coloneqq \mathbb{M}_{\mathbb{L}(\mathcal{H})}(\mathcal{D}_{N},\hat{L}) \in \hat{\mathcal{L}} = \argminA\limits_{\mathcal{M} \in \mathbb{L}(\mathcal{M})} \hat{L}(\mathcal{M}),
		\end{equation}
	\end{linenomath}
	that is the solution with the least VC dimension.
	
	The main issue with problem (\ref{ucurve_return}) is that, in principle, it demands a combinatorial algorithm which exhaustively search $\mathbb{L}(\mathcal{H})$ to compute $\mathcal{\hat{L}}$ and then select $\mathcal{\hat{M}}$ as the simplest model in it, making it computationally unfeasible to compute $\mathcal{\hat{M}}$. However, due to properties of $\mathbb{L}(\mathcal{H})$ under a loss function $\ell$, and the fact that we consider an estimator $\hat{L}(\mathcal{M})$ apart from the resubstitution error $L_{\mathcal{D}_{N}}(\hat{h}_{\mathcal{M}}(\mathcal{D}_{N}))$, one may take advantage of a U-curve property to solve the problem without exhaustively searching $\mathbb{L}(\mathcal{H})$. Indeed, to find $\mathcal{\hat{M}}$ one needs only to find all strong (or all weak) local minimums of $\mathbb{L}(\mathcal{H})$ (cf. Definition \ref{local_min}), as each global minimum is one of them, a search which may be performed more efficiently than an exhaustive one if the loss function satisfies either the strong or the weak U-curve property. 
	
	\begin{definition}
		\label{u-curve_property} \normalfont
		A Learning Space $\mathbb{L}(\mathcal{H})$ under loss function $\ell$ and estimator $\hat{L}$ satisfies the:
		\begin{itemize}
			\setlength\itemsep{0.5em}
			\item \textbf{strong U-curve property} if every weak local minimum of a continuous chain of $\mathbb{L}(\mathcal{H})$ is a global minimum of such chain;
			\item \textbf{weak U-curve property} if every strong local minimum is a global minimum of all continuous chains of $\mathbb{L}(\mathcal{H})$ which contain it. 
		\end{itemize} 
		The conditions characterizing the U-curve properties should be true with probability one, holding for all possible samples $\mathcal{D}_{N}$, for any value of $N$.
	\end{definition}
	
	We call the properties above U-curve for the plot of $(i_{j},\hat{L}(\mathcal{M}_{i_{j}})), j = 1, \dots, k$, is \textit{U}-shaped when calculated for any continuous chain $\mathcal{M}_{i_{1}} \subset \cdots \subset \mathcal{M}_{i_{k}}$ if the strong U-curve property is in force. It is straightforward that the strong implies the weak U-curve property. Since the concept of local minimums (cf. Definition \ref{local_min}) depends on $\hat{L}$, so does the U-curve properties, whose conditions, once $\hat{L}$ is fixed, should hold for any possible sample $\mathcal{D}_{N}$. In Figure \ref{ucurvelattice}, we present an example of a lattice which satisfies the weak U-curve property. 
	
	In Figure \ref{minimums} is depicted a strong and a weak local minimum. On the one hand, we see in (a) that the intersection of all curves is a strong local minimum, for it is a local minimum of all continuous chains which contain it. Furthermore, it is also the global minimum of all continuous chains which contain it, presenting the behavior which characterizes the weak U-curve property. On the other hand, in (b) we see that the intersection of all curves is not a strong local minimum, but rather a weak local minimum of some continuous chains which contain it, and no U-curve property is satisfied.
	
	The U-curve phenomenon allows a non-exhaustive search for $\mathcal{\hat{M}}$. If the strong U-curve property is satisfied, to find $\mathcal{\hat{M}}$ we do not need to exhaustively search $\mathbb{L}(\mathcal{H})$: we go through every continuous chain of $\mathbb{L}(\mathcal{H})$ until we find the (only) weak local minimum of it so that we find every weak local minimum, and, therefore, the global minimum. Similarly, if the weak U-curve property is satisfied, to find $\hat{\mathcal{M}}$ we go through every continuous chain of $\mathbb{L}(\mathcal{H})$ until we find the (only) strong local minimum of it so that we find every strong local minimum, and, therefore, the global minimum. 
	
	Either way, $\mathbb{L}(\mathcal{H})$ is not exhaustively searched, as when we find a weak or strong local minimum we do not need to estimate the error of the remaining models of a continuous chain, as the strong or weak U-curve property, respectively, ensure that the found local minimum is a global minimum of the continuous chain. This was first done for feature selection lattices in \cite{u-curve3,featsel,reis2018,u-curve1,u-curve2,ucurveParallel}.
	
	\begin{figure}[ht]
		\begin{center}
			\begin{tikzpicture}[scale=1, transform shape]
			\tikzstyle{hs} = [circle,draw=black, rounded corners,minimum width=2em, node distance=2.5cm, line width=1pt]
			\tikzstyle{branch}=[fill,shape=circle,minimum size=4pt,inner sep=0pt]
			\tikzstyle{simple}=[rectangle,draw=black,minimum height=2em,minimum width=3.5em, node distance=2cm,line width=1pt]
			\tikzstyle{dot}=[node distance=2.5cm, line width=1pt, minimum width=1em]
			\node[hs] at (5, 2) (empty) {\tiny 0,07};
			\node[hs,fill=orange] at (0, 0.5) (a1) {\tiny 0,05};
			\node[hs] at (3.3, 0.5) (a2) {\tiny 0,062};
			\node[hs] at (6.6, 0.5) (a3) {\tiny 0,057};
			\node[hs,fill=orange] at (10, 0.5) (a4) {\tiny 0,051};
			\node[hs] at (0, -1) (a12) {\tiny 0,060};
			\node[hs,fill=green] at (2, -1) (a13) {\tiny 0,042};
			\node[hs] at (4, -1) (a14) {\tiny 0,053};
			\node[hs,fill=green] at (6, -1) (a23) {\tiny 0,041};
			\node[hs,fill=orange] at (8, -1) (a24) {\tiny 0,048};
			\node[hs] at (10, -1) (a34) {\tiny 0,054};
			\node[hs,fill=orange] at (0, -2.5) (a123) {\tiny 0,045};
			\node[hs,fill=green] at (3.3, -2.5) (a124) {\tiny 0,047};
			\node[hs,fill=orange] at (6.6, -2.5) (a134) {\tiny 0,048};
			\node[hs,fill=orange] at (10, -2.5) (a234) {\tiny 0,053};
			\node[hs] at (5, -4) (a1234) {\tiny 0,055};

			\begin{scope}[line width=1pt]
			\draw[-] (empty) -- (a1);
			\draw[-] (empty) -- (a2);
			\draw[-] (empty) -- (a3);
			\draw[-] (empty) -- (a4);
			\draw[-] (a1) -- (a12);
			\draw[-] (a1) -- (a13);
			\draw[-] (a1) -- (a14);
			\draw[-] (a2) -- (a12);
			\draw[-] (a2) -- (a23);
			\draw[-] (a2) -- (a24);
			\draw[-] (a3) -- (a13);
			\draw[-] (a3) -- (a23);
			\draw[-] (a3) -- (a34);
			\draw[-] (a4) -- (a14);
			\draw[-] (a4) -- (a24);
			\draw[-] (a4) -- (a34);
			\draw[-] (a12) -- (a123);
			\draw[-] (a12) -- (a124);
			\draw[-] (a13) -- (a123);
			\draw[-] (a13) -- (a134);
			\draw[-] (a14) -- (a124);
			\draw[-] (a14) -- (a134);
			\draw[-] (a23) -- (a123);
			\draw[-] (a23) -- (a234);
			\draw[-] (a24) -- (a124);
			\draw[-] (a24) -- (a234);
			\draw[-] (a34) -- (a134);
			\draw[-] (a34) -- (a234);
			\draw[-] (a123) -- (a1234);
			\draw[-] (a124) -- (a1234);
			\draw[-] (a134) -- (a1234);
			\draw[-] (a234) -- (a1234);
			\end{scope}
			
			\end{tikzpicture}
		\end{center}
		\caption{\footnotesize Example of a lattice satisfying the weak U-curve property. The number inside each node $\mathcal{M}$ is $\hat{L}(\mathcal{M})$. The strong local minimums are in green and the weak local minimums are in orange. All strong local minimums are global minimums of all continuous chains which contain them, so this is an example of a weak U-curve property configuration.}
		\label{ucurvelattice}
	\end{figure}
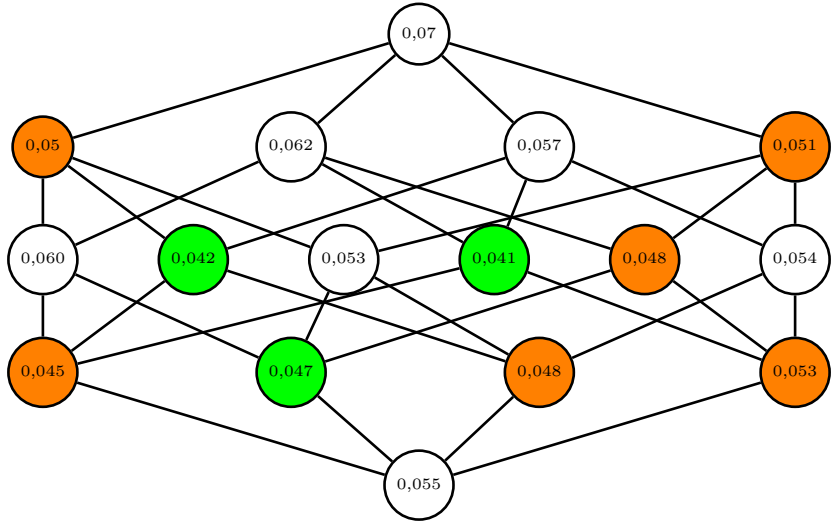
	
	\begin{figure}[ht]
		\centering
		\includegraphics[width=\linewidth]{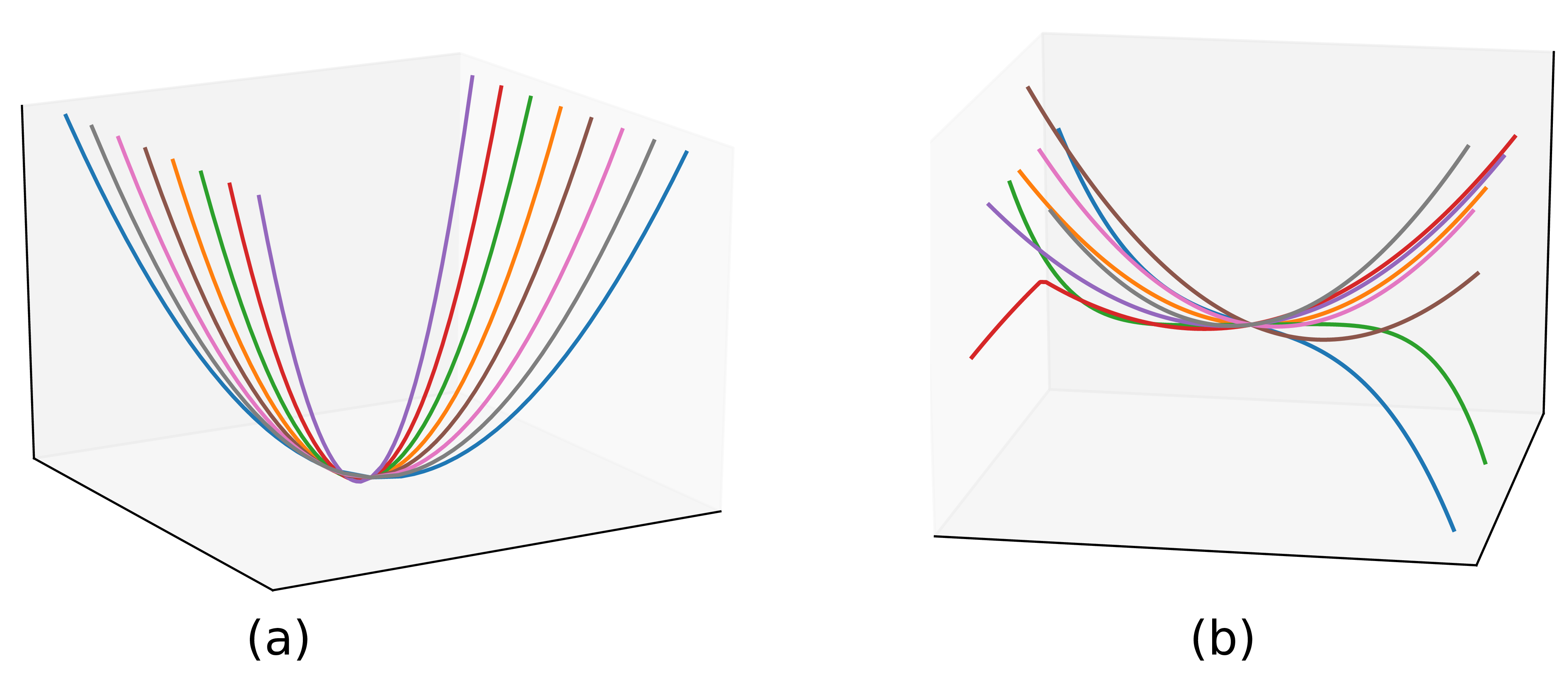}
		\caption{\footnotesize Example of (a) strong and (b) weak local minimums. In (a), we have the typical behavior of the weak U-curve property.} \label{minimums}
	\end{figure}
	
	The U-curve properties are characterized by local features of the estimated error $\hat{L}$, when calculated for chains of a Leaning Space, which imply global properties of such chains. This implication is only possible due to the existence of a structure, given by relations between the models in $\mathbb{L}(\mathcal{H})$. Therefore, not only the structure of the Learning Space is highly related to the rate of convergence to the target hypotheses space, evidenced by the MDE $\epsilon^{\star}$, but it is also what enables the estimation of $\hat{\mathcal{M}}$, or a suboptimal model, via the solution of a U-curve optimization problem, evidencing once more the importance of carefully modeling $\mathbb{L}(\mathcal{H})$ according to the application at hand.
	
	The next step, after defining the U-curve properties, is to point Learning Spaces which satisfy them. A natural way is to establish sufficient conditions for a  U-curve property which can be verified on a given $\mathbb{L}(\mathcal{H})$ or employed to construct Learning Spaces. In Section \ref{Sec1Chap2}, we show that the Partition Lattice Learning Space satisfies the weak U-curve property and in Section \ref{sufcond} we establish a sufficient condition for the weak U-curve property. The condition presented there is not necessary, as we show that the Partition Lattice Learning Space satisfies the weak U-curve property, but not the condition.
	
	\subsection{U-curve on the Partition Lattice Learning Space}
	\label{Sec1Chap2}
	
	When $\mathcal{X}$ is finite, one could theoretically explicitly search the Partition Lattice Learning Space to find a target partition. However, since the cardinality of this space is the $|\mathcal{X}|$-Bell number \cite{becker1948arithmetic,bell1934exponential,bell1938iterated}, which increases more than exponentially with $|\mathcal{X}|$ (see Table \ref{bell} for the 30 first Bell numbers), an exhaustive search of this lattice is not feasible. However, depending on how one defines the loss function $\ell$ and the loss $\hat{L}$ of each partition, a non-exhaustive search may be performed in this space, returning a suitable partition on which to learn a good hypothesis, since the weak U-curve property is satisfied. This is the result of the next proposition.
	
	\begin{proposition}
		\label{partition_has_ucurve} 
		The Partition Lattice Learning Space under the simple loss function and $\hat{L}$ of the form \eqref{form_Lhat} satisfies the weak U-curve property.
	\end{proposition}
	
	Learning hypotheses via the Partition Lattice Learning Space, although demands a lot of computation, has some advantages. First, if one has some prior knowledge about the partition generated by $h^{\star}$ he may search only the partitions which satisfy a given property, or within a partition consider only hypotheses that respect it and satisfy a given condition. Second, once a partition is selected, one may qualitatively analyze it, and the path of the U-curve algorithm until it (cf. Algorithm \ref{A1}), to obtain insights about the learned hypothesis and better understand why it classifies certain inputs in an output.

	\begin{table}[ht]
		\centering
		\resizebox*{\linewidth}{!}{\begin{tabular}{cl|cl|cl}
				\hline
				$|\mathcal{X}|$ & Bell number & $|\mathcal{X}|$ & Bell number & $|\mathcal{X}|$ & Bell number \\ 
				\hline
				1 & 1 & 11 & 678,570 & 21 & 474,869,816,156,751 \\ 
				2 & 2 & 12 & 4,213,597 & 22 & 4,506,715,738,447,323 \\ 
				3 & 5 & 13 & 27,644,437 & 23 & 44,152,005,855,084,344 \\ 
				4 & 15 & 14 & 190,899,322 & 24 & 445,958,869,294,805,312 \\ 
				5 & 52 & 15 & 1,382,958,545 & 25 & 4,638,590,332,229,998,592 \\ 
				6 & 203 & 16 & 10,480,142,147 & 26 & 49,631,246,523,618,762,752 \\ 
				7 & 877 & 17 & 82,864,869,804 & 27 & 545,717,047,936,060,030,976 \\ 
				8 & 4,140 & 18 & 682,076,806,159 & 28 & 6,160,539,404,599,936,679,936 \\ 
				9 & 21,147 & 19 & 5,832,742,205,057 & 29 & 71,339,801,938,860,290,605,056 \\ 
				10 & 115,975 & 20 & 51,724,158,235,372 & 30 & 846,749,014,511,809,254,653,952 \\ 
				\hline
		\end{tabular}}
		\caption{\footnotesize First to 30th Bell number.} \label{bell}
	\end{table}
	
	\subsection{Sufficient condition for U-curve property}
	\label{sufcond}
	
	We start out by establishing notation. For each $\mathcal{M} \in \mathbb{L}(\mathcal{H})$, a Lattice Learning Space, define
	\begin{linenomath}
		\begin{align*}
		\mathcal{C}^{+}(\mathcal{M}) = \Big\{\mathcal{M}_{i} \in \mathbb{L}(\mathcal{H}): \mathcal{M} \subset \mathcal{M}_{i}\Big\} & & \mathcal{C}^{-}(\mathcal{M}) = \Big\{\mathcal{M}_{i} \in \mathbb{L}(\mathcal{H}):  \mathcal{M}_{i} \subset \mathcal{M}\Big\}
		\end{align*} 
	\end{linenomath}
	as the models which contain or are contained in $\mathcal{M}$, respectively. Both $\mathcal{C}^{+}(\mathcal{M})$ and $\mathcal{C}^{-}(\mathcal{M})$ are complete lattices, on which $\mathcal{M}$ is the least and greatest model, respectively. We define for each $\mathcal{M}_{i} \in \mathcal{C}^{+}(\mathcal{M})\setminus\{\mathcal{M}\}$ the lower immediate neighborhood of $\mathcal{M}_{i}$ relative to $\mathcal{M}$ as
	\begin{linenomath}
		\begin{align*}
		N^{+}(\mathcal{M}_{i}) = \Big\{\mathcal{M}_{j} \in \mathcal{C}^{+}(\mathcal{M}): \mathcal{M}_{j} \subset \mathcal{M}_{i}, d(\mathcal{M}_{j},\mathcal{M}_{i}) = 1\Big\}
		\end{align*}
	\end{linenomath}
	and for each $\mathcal{M}_{i} \in \mathcal{C}^{-}(\mathcal{M})\setminus\{\mathcal{M}\}$ the upper immediate neighborhood relative to $\mathcal{M}$ as
	\begin{linenomath}
		\begin{align*}
		N^{-}(\mathcal{M}_{i}) = \Big\{\mathcal{M}_{j} \in \mathcal{C}^{-}(\mathcal{M}): \mathcal{M}_{i} \subset \mathcal{M}_{j}, d(\mathcal{M}_{j},\mathcal{M}_{i}) = 1\Big\}.
		\end{align*}
	\end{linenomath}
	What differs these two sets, both composed by the models in the sub-lattice which has $\mathcal{M}$ as the greatest or least element, which are at a distance one from $\mathcal{M}_{i}$, is if these models contain or are contained in $\mathcal{M}_{i}$.
	
	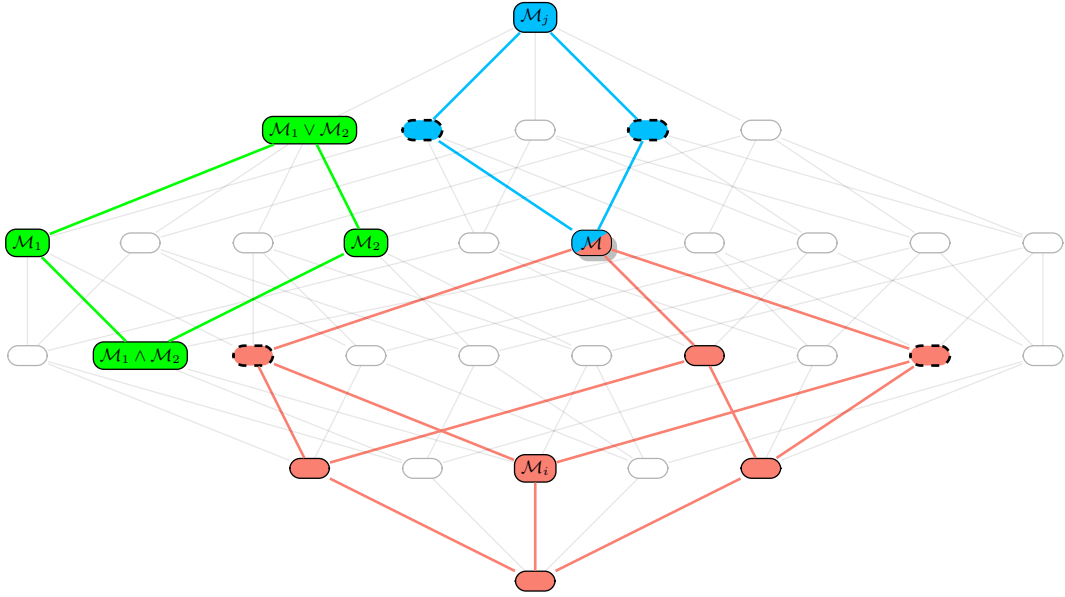
\begin{figure}[ht]
		\begin{center}
			\begin{tikzpicture}[scale=0.75, transform shape]
			\tikzstyle{hs} = [rectangle,draw=black, rounded corners,minimum width=2em,minimum height=1em, node distance=2.5cm, line width=0.5pt,opacity = 0.3]
			\tikzstyle{hs2} = [rectangle,draw=black, rounded corners,minimum width=2em,minimum height=1em, node distance=2.5cm, line width=0.5pt]
			\tikzstyle{hs3} = [rectangle,draw=black, rounded corners,minimum width=2em,minimum height=1em, node distance=2.5cm, line width=0.5pt,fill = bblue]
			\tikzstyle{hs31} = [rectangle,draw=black, rounded corners,minimum width=2em,minimum height=1em, node distance=2.5cm, line width=1pt,fill = bblue,dashed]
			\tikzstyle{hs4} = [rectangle,draw=black, rounded corners,minimum width=2em,minimum height=1em, node distance=2.5cm, line width=0.5pt,fill = color1bg]
			\tikzstyle{hs41} = [rectangle,draw=black, rounded corners,minimum width=2em,minimum height=1em, node distance=2.5cm, line width=1pt,fill = color1bg,dashed]
			\tikzstyle{hs5} = [rectangle,draw=black, rounded corners,minimum width=2em,minimum height=1em, node distance=2.5cm, line width=0.5pt,fill = green]
			
			\node[hs3] at (0,0) (11111) {\small $\mathcal{M}_{j}$};
			
			\node[hs5] at (-4,-2) (01111) {\small $\mathcal{M}_{1} \vee \mathcal{M}_{2}$};
			\node[hs31] at (-2,-2) (10111) {};
			\node[hs] at (0,-2) (11011) {};
			\node[hs31] at (2,-2) (11101) {};
			\node[hs] at (4,-2) (11110) {};
			
			\node[hs5] at (-9,-4) (00111) {\small $\mathcal{M}_{1}$};
			\node[hs] at (-7,-4) (01011) {};
			\node[hs] at (-5,-4) (01101) {};
			\node[hs5] at (-3,-4) (01110) {\small $\mathcal{M}_{2}$};
			\node[hs] at (-1,-4) (10011) {};
			\node[diagonal fill={color1bg}{bblue},
			minimum height=1em,minimum width=2em,
			text centered, rounded corners, draw, drop shadow] at (1,-4) (10101) {\small $\mathcal{M}$};
			\node[hs] at (3,-4) (10110) {};
			\node[hs] at (5,-4) (11001) {};
			\node[hs] at (7,-4) (11010) {};
			\node[hs] at (9,-4) (11100) {};
			
			\node[hs] at (-9,-6) (00011) {};
			\node[hs41] at (-5,-6) (00101) {};
			\node[hs5] at (-7,-6) (00110) {\small $\mathcal{M}_{1} \wedge \mathcal{M}_{2}$};
			\node[hs] at (-3,-6) (01001) {};
			\node[hs] at (-1,-6) (01010) {};
			\node[hs] at (1,-6) (01100) {};
			\node[hs4] at (3,-6) (10001) {};
			\node[hs] at (5,-6) (10010) {};
			\node[hs41] at (7,-6) (10100) {};
			\node[hs] at (9,-6) (11000) {};
			
			\node[hs4] at (4,-8) (10000) {};
			\node[hs] at (2,-8) (01000) {};
			\node[hs4] at (0,-8) (00100) {\small $\mathcal{M}_{i}$};
			\node[hs] at (-2,-8) (00010) {};
			\node[hs4] at (-4,-8) (00001) {};
			
			\node[hs4] at (0,-10) (00000) {};
			

			\begin{scope}[line width=1pt]
			\draw[-,color = bblue] (11111) -- (10111);
			\draw[-,color = bblue] (11111) -- (11101);
			\draw[-,color = bblue] (10111) -- (10101);
			\draw[-,color = bblue] (11101) -- (10101);
			\draw[-,color = color1bg] (10101) -- (00101);
			\draw[-,color = color1bg] (10101) -- (10001);
			\draw[-,color = color1bg] (10101) -- (10100);
			\draw[-,color = color1bg] (10000) -- (10100);
			\draw[-,color = color1bg] (10000) -- (10001);
			\draw[-,color = color1bg] (00100) -- (10100);
			\draw[-,color = color1bg] (00001) -- (10001);
			\draw[-,color = color1bg] (00001) -- (00101);
			\draw[-,color = color1bg] (00100) -- (00101);
			\draw[-,color = color1bg] (00001) -- (00000);
			\draw[-,color = color1bg] (00100) -- (00000);
			\draw[-,color = color1bg] (10000) -- (00000);
			\draw[-,color = green] (01111) -- (00111);
			\draw[-,color = green] (01111) -- (01110);
			\draw[-,color = green] (00111) -- (00110);
			\draw[-,color = green] (01110) -- (00110);
			\end{scope}
			
			\begin{scope}[line width=0.5pt,opacity=0.1]
			\draw[-] (11111) -- (01111);
			\draw[-] (11111) -- (11011);
			\draw[-] (11111) -- (11110);		
			\draw[-] (01111) -- (01011);
			\draw[-] (01111) -- (01101);
			\draw[-] (10111) -- (00111);
			\draw[-] (10111) -- (10011);
			\draw[-] (10111) -- (10110);
			\draw[-] (11011) -- (01011);
			\draw[-] (11011) -- (10011);
			\draw[-] (11011) -- (11001);
			\draw[-] (11011) -- (11010);
			\draw[-] (11101) -- (01101);
			\draw[-] (11101) -- (11001);
			\draw[-] (11101) -- (11100);
			\draw[-] (11110) -- (01110);
			\draw[-] (11110) -- (10110);
			\draw[-] (11110) -- (11010);
			\draw[-] (11110) -- (11100);
			\draw[-] (00111) -- (00011);
			\draw[-] (00111) -- (00101);
			\draw[-] (01011) -- (00011);
			\draw[-] (01011) -- (01001);
			\draw[-] (01011) -- (01010);
			\draw[-] (01101) -- (00101);
			\draw[-] (01101) -- (01001);
			\draw[-] (01101) -- (01100);
			\draw[-] (01110) -- (01010);
			\draw[-] (01110) -- (01100);
			\draw[-] (10011) -- (00011);
			\draw[-] (10011) -- (10001);
			\draw[-] (10011) -- (10010);
			\draw[-] (10110) -- (00110);
			\draw[-] (10110) -- (10010);
			\draw[-] (10110) -- (10100);
			\draw[-] (11001) -- (01001);
			\draw[-] (11001) -- (10001);
			\draw[-] (11001) -- (11000);
			\draw[-] (11010) -- (01010);
			\draw[-] (11010) -- (10010);
			\draw[-] (11010) -- (11000);
			\draw[-] (11100) -- (01100);
			\draw[-] (11100) -- (10100);
			\draw[-] (11100) -- (11000);		
			\draw[-] (10000) -- (11000);
			\draw[-] (10000) -- (10010);
			\draw[-] (01000) -- (11000);
			\draw[-] (01000) -- (01100);
			\draw[-] (01000) -- (01010);
			\draw[-] (01000) -- (01001);
			\draw[-] (00100) -- (01100);
			\draw[-] (00100) -- (00110);
			\draw[-] (00010) -- (10010);
			\draw[-] (00010) -- (01010);
			\draw[-] (00010) -- (00110);
			\draw[-] (00010) -- (00011);
			\draw[-] (00001) -- (01001);
			\draw[-] (00001) -- (00011);
			\draw[-] (00010) -- (00000);
			\draw[-] (01000) -- (00000);				
			\end{scope}
			
			\end{tikzpicture}
		\end{center}
		\caption{\footnotesize A Learning Space isomorphic to a Boolean Lattice, so it is U-curve compatible. The red nodes represent the lattice $C^{-}(\mathcal{M})$ and the blue nodes the lattice $C^{+}(\mathcal{M})$, for a given $\mathcal{M}$. The red dashed nodes are in $N^{-}(\mathcal{M}_{i})$ and the blue dashed nodes are in $N^{+}(\mathcal{M}_{j})$. The green nodes are an example of a pair $\mathcal{M}_{1}, \mathcal{M}_{2}$ for which the condition \eqref{condS} of Theorem \ref{general_theorem} should be satisfied.}
		\label{Blattice}
	\end{figure}
	
	If $\mathcal{M}_{j} \in N^{+}(\mathcal{M}_{i})$ then $\mathcal{M} \subset \mathcal{M}_{j} \subset \mathcal{M}_{i}$ and if $\mathcal{M}_{j} \in N^{-}(\mathcal{M}_{i})$ then $\mathcal{M}_{i} \subset \mathcal{M}_{j} \subset \mathcal{M}$. We say that $\mathbb{L}(\mathcal{H})$ is U-curve compatible if for every $\mathcal{M} \in \mathbb{L}(\mathcal{H})$
	\begin{linenomath}
		\begin{align*}
		\nonumber
		N^{+}(\mathcal{M}_{i}) &= \{\mathcal{M}\} \text{ or } |N^{+}(\mathcal{M}_{i})| \geq 2, & & \forall \ \mathcal{M}_{i} \in \mathcal{C}^{+}(\mathcal{M})\setminus\{\mathcal{M}\} \\
		N^{-}(\mathcal{M}_{i}) &= \{\mathcal{M}\} \text{ or } |N^{-}(\mathcal{M}_{i})| \geq 2, & & \forall \ \mathcal{M}_{i} \in \mathcal{C}^{-}(\mathcal{M})\setminus\{\mathcal{M}\}
		\end{align*}
	\end{linenomath}
	i.e., $\mathbb{L}(\mathcal{H})$ is \textit{U-curve compatible} if for every $\mathcal{M} \in \mathbb{L}(\mathcal{H})$  the lower (upper) immediate neighborhood of all models in $\mathcal{C}^{+}(\mathcal{M})\setminus\{\mathcal{M}\}$ ($\mathcal{C}^{-}(\mathcal{M})\setminus\{\mathcal{M}\}$) is equal to $\mathcal{M}$ or contains at least two distinct models. The sets defined above are illustrated in Figure \ref{Blattice} which presents a U-curve compatible Learning Space. If $\mathbb{L}(\mathcal{H})$ is U-curve compatible, then a simple property of $\hat{L}(\mathcal{M}_{i})$ is sufficient for the weak U-curve property.

	\begin{theorem}
		\label{general_theorem}
		Let $\mathbb{L}(\mathcal{H})$ be a U-curve compatible Lattice Learning Space. If all $\mathcal{M}_{1}, \mathcal{M}_{2} \in \mathbb{L}(\mathcal{H})$ such that $d(\mathcal{M}_{i},\mathcal{M}_{1} \wedge \mathcal{M}_{2}) = d(\mathcal{M}_{i},\mathcal{M}_{1} \vee \mathcal{M}_{2}) = 1, i = 1,2,$ satisfies
		\begin{linenomath}
			\begin{align}
			\label{condS}
			\hat{L}(\mathcal{M}_{1} \vee \mathcal{M}_{2}) \geq \hat{L}(\mathcal{M}_{1}) + \hat{L}(\mathcal{M}_{2}) - \hat{L}(\mathcal{M}_{1} \wedge \mathcal{M}_{2}),
			\end{align}
		\end{linenomath}
		with probability $1$, then the weak U-curve property is in force for $\mathbb{L}(\mathcal{H})$ under $\ell$ and estimator $\hat{L}$.
	\end{theorem}
	
	An example of pair $\mathcal{M}_{1}, \mathcal{M}_{2}$ which should satisfy \eqref{condS} is presented in Figure \ref{Blattice}. Assuming, without loss of generality, that $\hat{L}(\mathcal{M}_{1}) \geq \hat{L}(\mathcal{M}_{2})$ and rewriting \eqref{condS} as
	\begin{equation*}
	\hat{L}(\mathcal{M}_{1} \vee \mathcal{M}_{2}) - \hat{L}(\mathcal{M}_{1}) \geq \hat{L}(\mathcal{M}_{2}) - \hat{L}(\mathcal{M}_{1} \wedge \mathcal{M}_{2})\; ,
	\end{equation*}
	we see that the increase on $\hat{L}$ when we go from $\mathcal{M}_{1}$ to $\mathcal{M}_{1} \vee \mathcal{M}_{2}$ is greater than when we go from $\mathcal{M}_{1} \wedge \mathcal{M}_{2}$ to $\mathcal{M}_{2}$. This feature is analogous to that observed on convex functions, that is, the increment of the function increases when its inputs increase, in which \textit{increase in inputs} in this context is according to relation $\subset$ and the value of $\hat{L}$ when subsets are not related. Hence, we call property \eqref{condS} \textit{Lattice Convexity}. As is the case with convex functions, (strong) local minimums are global minimums, but, since $\mathbb{L}(\mathcal{H})$ is a poset, this is true only for the models which are related to the local minimum.
	
	Our definition of weak U-curve property is not as restrictive as lattice convexity, since the latter implies that the chains are monotone below and above the strong minimum, while in our definition there may be, for example, another weak local minimum in a continuous chain that passes through a strong minimum, although this weak minimum has a loss greater than the strong one.  
	
	In order to prove Proposition \ref{partition_has_ucurve} we have not relied on Theorem \ref{general_theorem} since, even though the Partition Lattice Learning Space is U-curve compatible, the condition \eqref{condS} is not satisfied for all pair of sets $\mathcal{M}_{1}, \mathcal{M}_{2}$ with $d(\mathcal{M}_{i},\mathcal{M}_{1} \wedge \mathcal{M}_{2}) = d(\mathcal{M}_{i},\mathcal{M}_{1} \vee \mathcal{M}_{2}) = 1, i = 1,2$. See Appendix \ref{counter_part} for examples of $\mathcal{M}_{1}, \mathcal{M}_{2}$ which satisfy and which may not satisfy the condition on the Partition Lattice Learning Space.
	
	\section{A generic U-curve algorithm}
	\label{U-curve}
	
	The U-curve algorithm was first proposed by \cite{u-curve1} in the context of feature selection and showed by \cite{u-curve2} to be NP-hard (see also \cite{ucurveParallel}). However, the algorithm avoids an exhaustive search of $\mathbb{L}(\mathcal{H})$ and permits solving problems which could not be solved by exhaustive search, due to limitations on computer power. Indeed, it has been applied to a variety of problems \cite{u-curve3,edu,featsel,reis2018,ucurveParallel}. In this section, we discuss a generic U-curve algorithm to solve \eqref{ucurve_return} when the weak U-curve property is satisfied.
	
	To apply the U-curve algorithm in this context, we assume that an optimizer which, given $\mathcal{M} \in \mathbb{L}(\mathcal{H})$, computes $\hat{L}(\mathcal{M})$, is implemented. Moreover, for each $\mathcal{M} \in \mathbb{L}(\mathcal{H})$ we define
	\begin{linenomath}
		\begin{align*}
		N(\mathcal{M}) = \Big\{\mathcal{M}_{i} \in \mathbb{L}(\mathcal{H}):\mathcal{M} \subset \mathcal{M}_{i}, d(\mathcal{M},\mathcal{M}_{i}) = 1\Big\}
		\end{align*}
	\end{linenomath}
	as the models in $\mathbb{L}(\mathcal{H})$ which immediately contains $\mathcal{M}$ and
	\begin{linenomath}
		\begin{align*}
		n(\mathcal{M}) = \Big\{\mathcal{M}_{i} \in \mathbb{L}(\mathcal{H}): \mathcal{M}_{i} \subset \mathcal{M}, d(\mathcal{M},\mathcal{M}_{i}) = 1\Big\}
		\end{align*}
	\end{linenomath}
	as the models in $\mathbb{L}(\mathcal{H})$ immediately contained in $\mathcal{M}$. We call $N(\mathcal{M})$ and $n(\mathcal{M})$ the upper and lower immediate neighborhood of $\mathcal{M}$, respectively.  
	
	Next, we need an auxiliary algorithm to compute if a given model $\mathcal{M}$ is a strong local minimum of $\mathbb{L}(\mathcal{H})$ by estimating the out-of-sample error of all models in the immediate neighborhoods of $\mathcal{M}$ and comparing them to the estimated error of  $\mathcal{M}$. The \textit{MinimumExhausted} algorithm is presented in Algorithm \ref{minexs} and returns TRUE if $\mathcal{M}$ is a strong local minimum, and FALSE otherwise. 
	
	\begin{algorithm}[ht]
		\caption{\textit{MinimumExhausted} auxiliary algorithm.}
		\label{minexs}
		\begin{spacing}{1.3}
			\hspace*{\algorithmicindent} Input: $\mathcal{M}, \hat{L}(\mathcal{M})$
			\begin{algorithmic}[1]
				\FOR{$\mathcal{M}^{\prime} \in N(\mathcal{M})  \cup n(\mathcal{M})$}
				\IF{$\hat{L}(\mathcal{M}^{\prime}) < \hat{L}(\mathcal{M})$}
				\RETURN FALSE
				\ENDIF
				\ENDFOR
				\RETURN TRUE
			\end{algorithmic}
		\end{spacing}
	\end{algorithm}
	
	Taking advantage of the weak U-curve property, the U-curve algorithm presented in Algorithm \ref{A1} solves (\ref{ucurve_return}) and is as follows. We first select a model $\mathcal{M} \in \mathbb{L}(\mathcal{H})$ and calculate $\hat{L}(\mathcal{M})$. Then, we apply the \textit{MinimumExhausted} algorithm to $\mathcal{M}$ to establish if it is a strong local minimum. If it is a strong local minimum, we store $\mathcal{M}$ and exclude from $\mathbb{L}(\mathcal{H})$ all models which contain or are contained in $\mathcal{M}$, as we know they have a greater estimated out-of-sample error by the weak U-curve property, obtaining a $\mathbb{L} \subset \mathbb{L}(\mathcal{H})$. If $\mathbb{L} \neq \emptyset$, we start the process again by selecting a model $\mathcal{M} \in \mathbb{L}$.
	
	If $\mathcal{M}$ is not a strong local minimum and there exists a $\mathcal{M}^{\prime} \in (N(\mathcal{M}) \cup n(\mathcal{M})) \cap \mathbb{L}$ with $\hat{L}(\mathcal{M}^{\prime}) < \hat{L}(\mathcal{M})$ we exclude $\mathcal{M}$ from $\mathbb{L}$, make $\mathcal{M} = \mathcal{M}^{\prime}$ and start the algorithm again. Otherwise, if there is no such $\mathcal{M}^{\prime}$, we exclude $\mathcal{M}$ from $\mathbb{L}$ and start the algorithm from another $\mathcal{M} \in \mathbb{L}$. We proceed this way until $\mathbb{L} = \emptyset$. Finally, we select the global minimums among the stored strong local minimums. 
	
	The pseudocode presented in Algorithms \ref{minexs} and \ref{A1} do not treat the technical peculiarities of an implementation of the U-curve algorithm, and present only its main ideas, as it is out of the scope of this paper to further study U-curve algorithms, since it has been done elsewhere (see \cite{u-curve3,edu,featsel,u-curve1,u-curve2,ucurveParallel}).
	
	\begin{algorithm}[ht]
		\centering
		\caption{Generic U-curve algorithm when the weak U-curve property is satisfied.}
		\label{A1}
		\begin{algorithmic}[1]
			\ENSURE $\mathcal{M} \in \mathbb{L}(\mathcal{H}), \text{Cost} \gets \hat{L}(\mathcal{M}), \mathbb{L} \gets \mathbb{L}(\mathcal{H}), \text{LocalMinimuns} \gets \emptyset$			
			\WHILE{$\mathbb{L} \neq \emptyset$}
			\IF{$MinimumExhausted(\mathcal{M},Cost)$}
			\STATE $\text{LocalMinimuns} \gets \text{LocalMinimuns} \cup \{\mathcal{M}\}$ 
			\STATE $\mathbb{L} \gets \mathbb{L} \setminus \{\mathcal{M}_{i} \in \mathbb{L}: \mathcal{M}_{i} \subset \mathcal{M} \text{ or } \mathcal{M} \subset \mathcal{M}_{i}\}$
			\IF{$\mathbb{L} \neq \emptyset$}
			\STATE $\mathcal{M} \gets \mathcal{M}_{i}, \mathcal{M}_{i} \in \mathbb{L}$
			\STATE $\text{Cost} \gets \hat{L}(\mathcal{M})$
			\ENDIF
			\ELSE
			\IF{$\exists \mathcal{M}^{\prime} \in (N(\mathcal{M}) \cup n(\mathcal{M})) \cap \mathbb{L} \text{ s.t } \hat{L}(\mathcal{M}^{\prime}) < \hat{L}(\mathcal{M})$}
			\STATE $\mathbb{L} \gets \mathbb{L} \setminus \{\mathcal{M}\}$
			\STATE $\mathcal{M} \gets \mathcal{M}^{\prime}$
			\STATE $\text{Cost} \gets \hat{L}(\mathcal{M}^{\prime})$
			\ELSE
			\STATE $\mathbb{L} \gets \mathbb{L} \setminus \{\mathcal{M}\}$
			\IF{$\mathbb{L} \neq \emptyset$}
			\STATE $\mathcal{M} \gets \mathcal{M}_{i}, \mathcal{M}_{i} \in \mathbb{L}$
			\STATE $\text{Cost} \gets \hat{L}(\mathcal{M})$
			\ENDIF
			\ENDIF 
			\ENDIF
			\ENDWHILE
			\RETURN{$\text{LocalMinimuns}$}
		\end{algorithmic}
	\end{algorithm}
	
	\begin{remark}
		\normalfont
		Due to ties of error $\hat{L}$, the global minimum with the least VC dimension may not be in the set returned by Algorithm \ref{A1}, since another global minimum which contains it may have been returned instead. To force the return of $\hat{\mathcal{M}}$ one has to check if the neighbors of a strong local minimum, with lesser VC dimension and same error, are also strong local minimums. We did not add this feature to Algorithm \ref{A1} to better present its main idea, but it must be regarded when implementing the algorithm for a specific case.
	\end{remark}
	
	\subsection{Improving the U-curve algorithm}
	
	There are a handful of features which could be implemented to the generic Algorithm \ref{A1} to improve its performance. The generic U-curve algorithm will always return a result since every time we exhaust the neighborhood of a model we delete it from $\mathbb{L}$ so, in the worst case, the algorithm would perform an exhaustive search of\footnote{Actually it would exhaust the neighborhood of all models, what is redundant and more complex than an exhaustive search of $\mathbb{L}(\mathcal{H})$.} $\mathbb{L}(\mathcal{H})$. Nevertheless, we can stop the algorithm early and apply an exhaustive search on the models remaining in $\mathbb{L}$, when the cardinality of $\mathbb{L}$ is within the reach of an exhaustive search. This may improve the algorithm since, as the cardinality of $\mathbb{L}$ decreases, it will contain less strong local minimums, hence the prunes of $\mathbb{L}$ due to strong local minimums will get rarer and the resources employed to exhaust the neighborhood of each model in $\mathbb{L}$ could be better employed to search it exhaustively. 
	
	For example, if a model in $\mathbb{L}$ has more immediate neighbors than there are points in $\mathbb{L}$, an exhaustive search of $\mathbb{L}$ could be more efficient than continuing with the U-curve algorithm. Therefore, one manner of improving the efficiency of the U-curve algorithm is to stop it when $|\mathbb{L}| < c$, for a given constant $c$, and then exhaustively search $\mathbb{L}$ for its global minimums, which could then be compared with the strong local minimums found by the U-curve algorithm to find the global minimums of $\mathbb{L}(\mathcal{H})$ and the solution of \eqref{ucurve_return}. 
	
	Another manner of improving the efficiency of the algorithm is to delete from $\mathbb{L}$ all neighbors with greater loss, visited during the \textit{MinimumExhausted} calls. Since these models do not have lesser loss than the exhausted model, they cannot be strong local minimums, hence may be disregarded. In some instances, it could be a good option to not stop the \textit{MinimumExhausted} when a model with less error is found, but rather calculate the error for all neighbors to exclude from $\mathbb{L}$ all of them with greater error and restart the algorithm in the neighbor with the least error when the exhausted model is not a strong local minimum. This algorithm can be performed in parallel by taking advantage of a specific feature of the lattice structure of $\mathbb{L}(\mathcal{H})$, improving the speed of the algorithm (see the results in \cite{ucurveParallel} for feature selection).
	
	Moreover, besides being a sufficient condition to the weak U-curve property, (\ref{condS}) is also a tool for increasing the efficiency of a U-curve algorithm. Suppose that $\mathcal{M}_{1}$ and $\mathcal{M}_{2}$ satisfy (\ref{condS}) and that $\min\big(\hat{L}(\mathcal{M}_{1}),\hat{L}(\mathcal{M}_{2})\big) > \hat{L}(\mathcal{M}_{1} \wedge \mathcal{M}_{2})$. Then
	\begin{linenomath}
		\begin{align*}
		\hat{L}(\mathcal{M}_{1} \vee \mathcal{M}_{2}) \geq \hat{L}(\mathcal{M}_{1}) + \hat{L}(\mathcal{M}_{2}) - \hat{L}(\mathcal{M}_{1} \wedge \mathcal{M}_{2}) > \hat{L}(\mathcal{M}_{1} \wedge \mathcal{M}_{2})
		\end{align*}
	\end{linenomath}
	so after visiting the models $\mathcal{M}_{1} \wedge \mathcal{M}_{2},\mathcal{M}_{1}$ and $\mathcal{M}_{2}$, and noting the increase in the estimated error, one does not need to visit $\mathcal{M}_{1} \vee \mathcal{M}_{2}$ as the estimated error will increase. This fact may be employed to establish search strategies for U-curve algorithms.
	
	The efficiency of the algorithm may be improved if one employs it to find suboptimal solutions.  For example, one may apply a stochastic U-curve algorithm, in which we sample a model in $\mathbb{L}(\mathcal{H})$ every time we restart the algorithm and one does not perform prunes on $\mathbb{L}(\mathcal{H})$, but rather stop the algorithm when a ``sufficient'' number of strong/weak local minimums, or a ``good enough'' local minimum, is found. This algorithm is implemented for the Partition Lattice Learning Space in \cite{edu} where more details can be found. Furthermore, in such a stochastic algorithm, one could sample a certain number of neighbors in Algorithm \ref{minexs} and consider as a strong local minimum any model which has a cost lesser than the sampled ones, what would speed up the algorithm at the cost of considering as strong minimums models which are not.
	
	Apart from features such as the ones discussed above, there are two main theoretical considerations when developing a U-curve algorithm which can impact its efficiency. The first one is how to prune the Learning Space when one finds a local minimum, i.e., find all models which contain or are contained in a local minimum and store this information in some way to not visit these models. This is an important question since $\mathbb{L}(\mathcal{H})$ may be too great to be stored. For example, the Partition Lattice Learning Space with $100$ points in the domain $\mathcal{X}$ has a cardinality of the order $10^{115}$, so it is unfeasible to store it. Hence, one needs a computationally cheap manner of determining if a $\mathcal{M} \in \mathbb{L}(\mathcal{H})$ is in $\mathbb{L}$ or if $\mathbb{L} = \emptyset$ in a given step of the algorithm, apart from storing $\mathbb{L}$ and testing the inclusion $\mathcal{M} \in \mathbb{L}$ or if $\mathbb{L} = \emptyset$. When the Learning Space is a Boolean Lattice, the inclusion and emptiness test may be explicitly computed via complement operations and stored efficiently (see \cite{featsel,u-curve1,u-curve2}). However, it is not always immediate how one can efficiently perform these tasks.
	
	The second consideration is which neighbor of a model to visit first in the \textit{MinimumExhausted} algorithm: can one search $\mathbb{L}(\mathcal{H})$ more efficiently by going down certain paths? We leave both these considerations as open problems, since technical details of the U-curve algorithm are out of the scope of this paper.
	
	\section{Discussion}
	
	In this section, we discuss the consequences of the proposed framework and some perspectives for future researches.
	
	\subsection{Main results and contributions to the field}
	
	For a long time, the communities of Pattern Recognition and Machine Learning have been doing optimization on, mainly Boolean, lattices, to design classifiers or to learn hypotheses. However, we have no knowledge about results establishing lattice and loss function properties which guarantee that the applied procedures really return optimal solutions. In this context, the main contribution of this paper is, to the best of our knowledge, the first rigorous definition of properties which ensure that a non-exhaustive search of a lattice with the purpose of Model Selection can return an optimal and statistically consistent result. We established such properties and showed that they do exist by proving the weak U-curve property is satisfied by the Partition Lattice Learning Space. 
	
	In general lines, the method presented in this paper proposes a data-driven systematic, consistent and non-exhaustive approach to Model Selection. Indeed, given a loss function, a sample, and a Learning Space generator, we proposed a systematic procedure, based on data, to select a model,  i.e., a subset in the Learning Space, and then learn a hypothesis of it seeking to approximate a target hypothesis of $\mathcal{H}$. Moreover, we showed that this method can be more efficient than an exhaustive search on the Learning Space for obtaining optimal solutions, when a U-curve property is satisfied, or suboptimal ones, when such a property is not satisfied. 
	
	Moreover, from a theoretical point of view, we established the statistical consistency of our approach, developing bounds for estimation errors and showing the convergence of the estimated random hypotheses space $\mathcal{\hat{M}}$ to the target hypotheses space $\mathcal{M}^{\star}$ with probability one. We introduced the maximum discrimination error $\epsilon^{\star}$ and evidenced the importance of properly embedding all prior information into the Learning Space under the paradigm, supported by the results of this paper, of better learning with a fixed sample size by properly modeling the Learning Space seeking to (a) have $\mathcal{M}^{\star}$ with small VC dimension and (b) have a great MDE $\epsilon^{\star}$.
	
	The practical application of the abstract theory requires, first, the concrete choice of a parametric representation of the hypotheses. Then, within a suitable algebraic structure of such a representation, one needs to define a Learning Space generator, that is a rule of how to obtain, via constraints on the parameters, a family of subsets of $\mathcal{H}$, partially ordered by inclusion, with the same algebraic structure considered for the parameters. A canonical practical example of this abstract system is the Boolean lattice for feature selection, in which the hypotheses are represented by the features they depend on; the Boolean lattice algebra of the subsets of features is considered; and the Learning Space generator associates each subset of features with the hypotheses which depend solely on features in the set.
	
	Although abstract and general, and even if, at first sight, it may not be clear how to construct Learning Spaces, they emerge naturally on applications on which a \textit{meaningful} parametric representation is employed. By meaningful, we mean that each parameter is identifiable with some concrete concept, that is, the parameters are interpretable. This is the case on Examples \ref{feature_lattice} to \ref{parametric_lattice}, where the parameters represent variables or points of the classifier domain. Besides facilitating the development of Learning Spaces, the interpretability of the parameters has at least another two properties. 
	
	First, since one knows what each parameter means, it is possible to translate prior information into constraints on these parameters in order to (a) define a Learning Space generator, (b) replace $\mathcal{H}$ by a subset of it and (c) drop sets from the Learning Space (see discussion on Example \ref{partition_lattice} for an example). Second, one may re-parametrize the hypotheses, or consider another constraint on the same parameters, to add other kinds of prior information, thus, generating another Learning Space for the same hypotheses space $\mathcal{H}$ (see Example \ref{parametric_lattice}). Hence, on the one hand, the Learning Space is not unique and there is no general abstract formulae to build just one that would work for all applications, what would be a “canonical Learning Space”. However, on the other hand, this flexibility when defining Learning Spaces makes it a customizable method, which can be instantiated for a large family of learning problems. 
	
	\subsection{Implicit complexity regularization}
	
	From the viewpoint of Machine Learning Theory, the approach presented here can also be seen as a regularization method. More specifically, it could be considered an \textit{implicit complexity regularizer}, which implies regularization via constraints on the complexity (VC dimension) of the candidate models. However, we believe that the U-curve algorithm applied to a Learning Space differs from established regularization methods because the optimization occurs in plain sight and is the consequence of a formal model designed to represent an application problem. Some ordinary regularization procedures are heuristics in which the loss function is penalized by the complexity of the model, and the estimated hypothesis is obtained by a constrained optimization problem which gives few or no reasons to why such model has been selected in detriment of others, although we note that some regularization procedures have rigorous mathematical justifications, with results such as those in \cite{massart2007}, and may work very well in practice. 
	
	In the extended learning model proposed here, regularization is an abstract process that results from the optimization in the Learning Space under a finite sample. In this learning approach, the search on the Learning Space itself implies the regularization effect. This oversight of the regularization process comes at the cost of computing power, since the search on a Learning Space is a combinatorial algorithm, even under a U-curve property, which will tend to have higher computational complexity than other regularization procedures. Nevertheless, with the ever-growing computing capabilities, in the near future such an increase in computational complexity could be worth it to have sharper models for real problems. 
	
	\subsection{Learning Spaces and penalized loss functions}
	\label{disPen}
	
	Penalized Model Selection in Statistical Learning, as described in \cite[Chapter~8]{massart2007}, is a special case of the proposed framework given by considering the model error estimator
	\begin{align*}
	\hat{L}(\mathcal{M}) \coloneqq L_{\mathcal{D}_{N}}(\hat{h}_{\mathcal{M}}) + \text{ pen}(\mathcal{M}) & & \mathcal{M} \in \mathbb{L}(\mathcal{H}),
	\end{align*}
	in which $\text{pen}: \mathbb{L}(\mathcal{H}) \to \mathbb{R}_{+}$ is a penalty function and $\hat{h}_{\mathcal{M}}$ minimizes the empirical error $L_{\mathcal{D}_{N}}$ in $\mathcal{M}$. In this case, the complexity regularization is explicit and due to the penalty function, as opposed to an implicit regularization given by a resampling method. The study of penalized Model Selection under the Learning Space framework would consist of proving results such as those in Sections \ref{SecConsistency} and \ref{SecUproperty}, as follows.
	
	First, one would have to establish the statistical consistency of the method, by showing the convergence of $\hat{\mathcal{M}}$ to $\mathcal{M}^{\star}$ with probability one, and the convergence in probability to zero of the estimation errors. This could be performed by applying very well-studied oracle and concentration inequalities for Model Selection with penalized loss functions (see \cite{massart2007,koltchinskii2001,koltchinskii2011,agarwal2011,arlot2011,bartlett2008} and the references therein).
	
	Second, one would have to establish conditions on the loss $\ell$, the Learning Space $\mathbb{L}(\mathcal{H})$ and the penalty function under which a U-curve property is satisfied, through the existence of lattice convexity, for example. Observe that in this case, the lattice convexity condition \eqref{condS} reduces to
	\begin{align*}
	\text{pen}(\mathcal{M}_{1} \vee \mathcal{M}_{1}) + \text{pen}(\mathcal{M}_{1} \wedge \mathcal{M}_{1}) &- \text{pen}(\mathcal{M}_{1}) - \text{pen}(\mathcal{M}_{2}) \geq\\
	&\hat{L}(\hat{h}_{\mathcal{M}_{1}}) + \hat{L}(\hat{h}_{\mathcal{M}_{2}}) - \hat{L}(\hat{h}_{\mathcal{M}_{1} \wedge \mathcal{M}_{2}}) - \hat{L}(\hat{h}_{\mathcal{M}_{1} \vee \mathcal{M}_{2}}), 
	\end{align*} 
	an inequality which associates the variation of the minimum empirical error within the considered models with the variation of the penalty function on these models. The existence of a U-curve property in this scheme could be a tool to enhance the computational efficiency of Model Selection with penalization, hence be a great contribution to this field.
	
	We leave this important case of penalized loss functions for future researches, since this paper aims to present the Learning Space framework not from the viewpoint of penalization, but rather as a general and consistent scheme for Model Selection based on resampling techniques such as cross-validation, so one does not have to \textit{choose} a penalty function.
	
	\subsection{Decreasing the approximation error}
	
	The proposed framework may also be applied to try to reduce the approximation error, which is as follows. Let $\mathcal{H}^{\star}$ be the set of all functions with domain $\mathcal{X}$, the support of $X$, and image $\{0,1\}$, which is possibly a hypotheses space with infinite VC dimension. Denote
	\begin{align*}
	h_{Bayes} = \arginfA_{h \in \mathcal{H}^{\star}} L(h)
	\end{align*}
	so that $L(h_{Bayes})$ is the Bayes error, the least error we commit when classifying instances with values $X$. Note that $h_{Bayes}$ may or may not be in $\mathcal{H}$, and when it is not, we commit the error
	\begin{equation*}
	L(h^{\star}) - L(h_{Bayes})
	\end{equation*} 
	which is called approximation error (see \cite[Chapter~12]{devroye1996}). This error is, in general, not controllable and, to decrease it, one must increase $\mathcal{H}$, which in turn increases the risk of \textit{overfitting} if the sample size is not great enough. This scenario is depicted in Figure \ref{fig_appro} (a).
	
	However, with the method presented in this paper, we may increase $\mathcal{H}$ mitigating the risk of \textit{overfitting}, so we expect to be able to reduce the approximation error. In a perfect scenario, one would expect the scheme presented in Figure \ref{fig_appro} (b): we choose a $\mathcal{H}$ highly complex, so it contains $h_{Bayes}$, but we actually learn on a relatively simple model $\hat{\mathcal{M}}$ which also contains $h_{Bayes}$. Even if $h_{Bayes}$ is not in our complex $\mathcal{H}$, which we actually do not know, we may expect the approximation error to be small, as, theoretically, we chose a $\mathcal{H}$ more complex than we would if we were unable to control \textit{overfitting} nor search $\mathbb{L}(\mathcal{H})$.
	
	\begin{figure}[H]
		\centering
		\includegraphics[scale = 0.35]{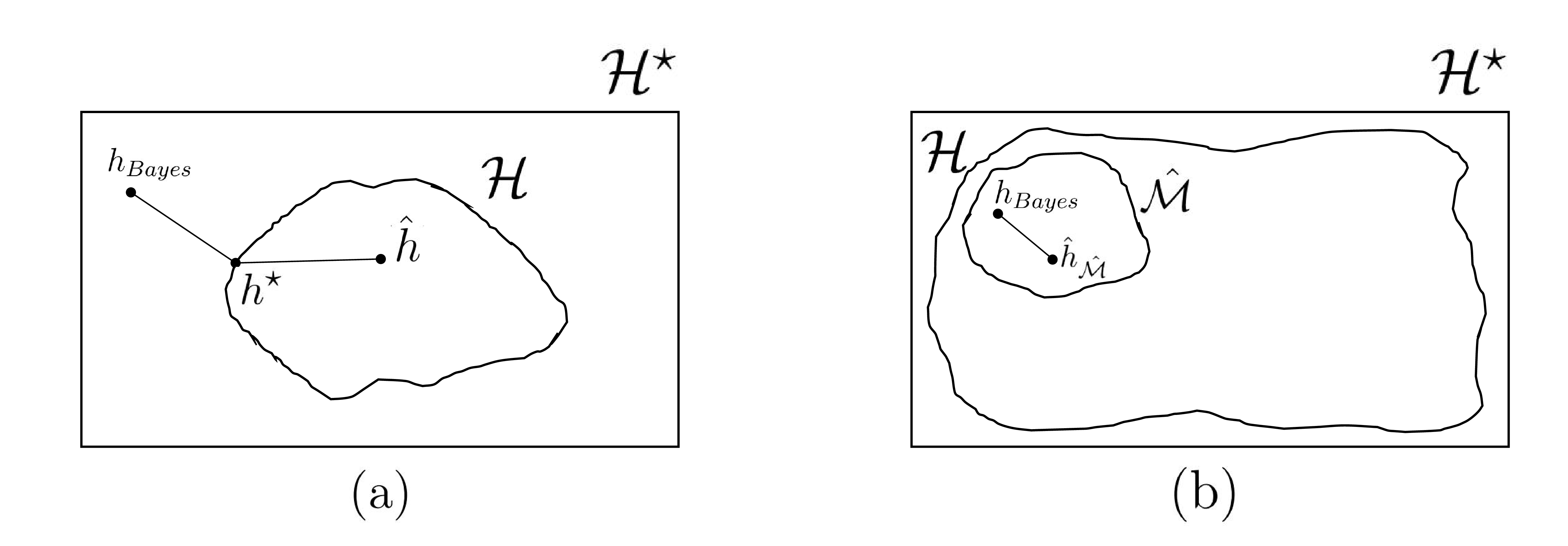}
		\caption{\footnotesize The approximation error of hypotheses spaces.}
		\label{fig_appro}
	\end{figure}
	
	\subsection{Perspectives for future researches}
	
	There are some interesting perspectives of applying the method to Neural Networks. As discussed in Example \ref{example_NN}, a given architecture of a Neural Network generates a hypotheses space $\mathcal{H}$, formed by the hypotheses obtained wandering all over its parameters' domain. Of course, such a $\mathcal{H}$ has subsets, and families of them satisfy the two axioms of Learning Spaces. Now, if one could develop a Learning Space generator which associates constraints on the architecture to subsets in a Learning Space, the abstract method proposed here could be applied to select architectures, since the constraints on the parameters would generate sub-architectures, so the Learning Space would represent a family of architectures, and the learning on it would be the learning of architectures. This instantiation has potential to become a relevant contribution to the field of Neural Architecture Search \cite{elsken2019} since it would constitute a systematic and general approach to architecture learning, which could have major implications to the future of Neural Networks research. However, since each parameter of a Neural Network is not exactly interpretable by itself, due to over-parametrization \cite{neyshabur2018,denil2013}, and the VC dimension of $\mathcal{H}$ is not exactly known in general for Neural Networks \cite{sontag1998}, the instantiation of the method in this case is not immediate and requires a further understanding of the concepts proposed here for Neural Networks. We leave this important study as a relevant topic for future research.
	
	The major perspectives for future researches concern empirical studies of the proposed approach. Although theoretically sound, the overall \textit{quality} of learning via Learning Spaces depends on the (low) complexity of $\mathcal{M}^{\star}$, on the (great) size of $\epsilon^{\star}$ and on the (low) computational complexity of $\hat{\mathcal{M}}$, which are all concepts dependent on the data distribution $P$. Hence, to attest the feasibility of this framework, it is necessary to perform empirical studies to assess how it behaves on specific problems of interest. Even though outside the scope of this paper, that aimed at presenting the theoretical foundations of learning via Learning Spaces, empirical studies are important to better outline what applications can benefit the most from this framework.
	
	Besides empirical studies, there are further perspectives for future researches. On the one hand, from a theoretical standpoint, one could try to find distribution-free bounds for types III and IV estimation errors in some cases of interest, which depend solely on $\mathbb{L}(\mathcal{H})$; study other frameworks for learning on $\hat{\mathcal{M}}$ besides those in Figure \ref{learn_hyp}; and develop sufficient conditions for Learning Spaces to satisfy a U-curve property. On the other hand, regarding U-curve algorithms, there is a lack of a multipurpose general implementation of it, as it is implemented only for specific cases (see \cite{featsel}). Such implementation relies on developing efficient ways of storing, pruning and searching Learning Spaces, which poses many open problems for non-Boolean lattices.
	
	Moreover, there is a lot of ground to break in the direction of developing families of Learning Spaces which solve a class of problems, showing the existence of the U-curve property for other Learning Spaces, and developing optimal and suboptimal U-curve algorithms. In no way, we exhausted this subject, but only scratched its surface, since the extension to the learning model proposed here may be a tool to understand and enhance the always increasing niche of high performance and computing demanding learning applications.
	
	The paradigm presented in this paper also lead to an important practical property of Machine Learning: the lack of experimental data may be mitigated by a high computational capacity, as one can apply computer power to look for the target model compatible with a given small sample under the paradigm in Figure \ref{paradigms}. In a context of continuous increasing and popularization of high computational power, this property may be the key to understanding why Machine Learning has become so important in all branches of Science, even in the ones where data is expensive and hard to get. 
	
	
	\section*{Acknowledgments}
	
	We thank Cl\'audia Peixoto for her careful reading of, and insightful remarks about, earlier versions of the manuscript. Diego Marcondes received financial support from CNPq during this research. Junior Barrera is supported by FAPESP: proc. 14/50937-1 and 15/24485-9.
	
	
	\appendix
	\section{Proof of results}
	
	\begin{proof}[Proof of Proposition \ref{proposition_principal}]
		If
		\begin{equation*}
		\sup_{i \in \mathcal{J}} \left|L(\mathcal{M}_{i}) - \hat{L}(\mathcal{M}_{i})\right| < \epsilon^{\star}/2
		\end{equation*}
		then, for any $i \in \mathcal{J}$ such that $L(\mathcal{M}_{i}) > L(\mathcal{M}^{\star})$, we have
		\begin{linenomath}
			\begin{align}
			\label{ineq11}
			\hat{L}(\mathcal{M}_{i}) - \hat{L}(\mathcal{M}^{\star}) > L(\mathcal{M}_{i}) - L(\mathcal{M}^{\star}) - \epsilon^{\star} \geq 0,
			\end{align}
		\end{linenomath}	
		in which the last inequality follows from the definition of $\epsilon^{\star}$. From \eqref{ineq11} follows that the global minimum of $\nicefrac{\mathbb{L}(\mathcal{H})}{\hat{\sim}}$ with the least VC dimension, that is $\hat{\mathcal{M}}$, is such that $L(\hat{\mathcal{M}}) = L(\mathcal{M}^{\star})$. Indeed, from \eqref{ineq11} follows that $\hat{L}(\mathcal{M}) > \hat{L}(\mathcal{M}^{\star})$ for all $\mathcal{M} \in \mathbb{L}(\mathcal{H})$ such that $L(\mathcal{M}) > L(\mathcal{M}^{\star})$, hence, since $\hat{L}(\hat{\mathcal{M}}) \leq \hat{L}(\mathcal{M}^{\star})$, we must have $L(\hat{\mathcal{M}}) = L(\mathcal{M}^{\star})$. Therefore, we have the inclusion of events
		\begin{equation}
		\label{inclusion_star}
		\left\{\sup_{i \in \mathcal{J}} \left|L(\mathcal{M}_{i}) - \hat{L}(\mathcal{M}_{i})\right| < \epsilon^{\star}/2\right\} \subset \left\{L(\hat{\mathcal{M}}) = L(\mathcal{M}^{\star})\right\},
		\end{equation}
		which proves the result.
	\end{proof}
	
	\begin{proof}[Proof of Theorem \ref{theorem_principal_convergence}]
		We will apply Proposition \ref{proposition_principal}. Denoting $\hat{h}_{i}^{(j)} \coloneqq \hat{h}_{\mathcal{M}_{i}}^{(j)}$,
		\begin{align*}
		\mathbb{P}&\left(\max_{i \in \mathcal{J}} \left|L(\mathcal{M}_{i}) - \hat{L}(\mathcal{M}_{i})\right| \geq \epsilon^{\star}/2\right) \leq \mathbb{P}\left(\max_{i \in \mathcal{J}} \sum_{j=1}^{m} \frac{1}{m} \left|L(\mathcal{M}_{i}) - \hat{L}^{(j)}(\hat{h}^{(j)}_{i})\right| > \epsilon^{\star}/2\right)\\
		&\leq \mathbb{P}\left(\max_{j} \max_{i \in \mathcal{J}} \left|L(\mathcal{M}_{i}) - \hat{L}^{(j)}(\hat{h}^{(j)}_{i})\right| > \epsilon^{\star}/2\right)\\
		&\leq \mathbb{P}\left(\bigcup_{j=1}^{m} \left\{\max_{i \in \mathcal{J}} \left|L(\mathcal{M}_{i}) - \hat{L}^{(j)}(\hat{h}^{(j)}_{i})\right| > \epsilon^{\star}/2\right\}\right)\\
		&\leq \sum_{j=1}^{m} \mathbb{P}\left(\max_{i \in \mathcal{J}} \left|L(\mathcal{M}_{i}) - \hat{L}^{(j)}(\hat{h}^{(j)}_{i})\right| > \epsilon^{\star}/2\right)\\
		&= \sum_{j=1}^{m} \mathbb{P}\left(\max_{i \in \mathcal{J}} \left|L(\mathcal{M}_{i}) - L(\hat{h}^{(j)}_{i}) + L(\hat{h}^{(j)}_{i}) - \hat{L}^{(j)}(\hat{h}^{(j)}_{i})\right| > \epsilon^{\star}/2\right)\\
		&\leq \sum_{j=1}^{m} \mathbb{P}\left(\max_{i \in \mathcal{J}} L(\hat{h}^{(j)}_{i}) - L(\mathcal{M}_{i}) + \max_{i \in \mathcal{J}} \left|L(\hat{h}^{(j)}_{i}) - \hat{L}^{(j)}(\hat{h}^{(j)}_{i})\right| > \epsilon^{\star}/2\right)\\
		&\leq \sum_{j=1}^{m} \mathbb{P}\left(\max_{i \in \mathcal{J}} L(\hat{h}^{(j)}_{i}) - L(\mathcal{M}_{i}) > \epsilon^{\star}/4\right) + \mathbb{P}\left(\max_{i \in \mathcal{J}} \left|L(\hat{h}^{(j)}_{i}) - \hat{L}^{(j)}(\hat{h}^{(j)}_{i})\right| > \epsilon^{\star}/4\right)\\
		&\leq \sum_{j=1}^{m} \mathbb{P}\left(\max_{i \in \mathcal{J}} L(\hat{h}^{(j)}_{i}) - L(\mathcal{M}_{i}) > \epsilon^{\star}/4\right) + \mathbb{P}\left(\max_{i \in \mathcal{J}} \sup_{h \in \mathcal{M}_{i}} \left|\hat{L}^{(j)}(h) - L(h)\right| > \epsilon^{\star}/4\right)
		\end{align*}
		in which in the first inequality we applied the definition of $\hat{L}(\mathcal{M})$. For each $j$, the first probability is equal to
		\begin{align*}
		\mathbb{P}&\left(\max_{i \in \mathcal{J}} L(\hat{h}^{(j)}_{i}) - L_{\mathcal{D}_{N}^{(j)}}(\hat{h}^{(j)}_{i}) + L_{\mathcal{D}_{N}^{(j)}}(\hat{h}^{(j)}_{i}) - L(\mathcal{M}_{i}) > \epsilon^{\star}/4\right)\\
		&\leq \mathbb{P}\left(\max_{i \in \mathcal{J}} L(\hat{h}^{(j)}_{i}) - L_{\mathcal{D}_{N}^{(j)}}(\hat{h}^{(j)}_{i}) + L_{\mathcal{D}_{N}^{(j)}}(h^{\star}_{i}) - L(\mathcal{M}_{i}) > \epsilon^{\star}/4\right)\\
		&\leq \mathbb{P}\left(\left\{\max_{i \in \mathcal{J}} \left|L(\hat{h}^{(j)}_{i}) - L_{\mathcal{D}_{N}^{(j)}}(\hat{h}^{(j)}_{i})\right| > \epsilon^{\star}/8\right\} \bigcup \left\{\max_{i \in \mathcal{J}} \left| L_{\mathcal{D}_{N}^{(j)}}(h^{\star}_{i}) - L(\mathcal{M}_{i}) \right| > \epsilon^{\star}/8\right\}\right)\\
		&\leq \mathbb{P}\left(\max_{i \in \mathcal{J}} \sup_{h \in \mathcal{M}_{i}} \left|L_{\mathcal{D}_{N}^{(j)}}(h) - L(h)\right| > \epsilon^{\star}/8\right)
		\end{align*}
		in which the first inequality follows from the fact that $L_{\mathcal{D}_{N}^{(j)}}(\hat{h}^{(j)}_{i}) \leq L_{\mathcal{D}_{N}^{(j)}}(h^{\star}_{i})$ and the last follows since $L(\mathcal{M}_{i}) = L(h^{\star}_{i})$. We conclude that
		\begin{align*}
		\mathbb{P}&\left(\max_{i \in \mathcal{J}} \left|L(\mathcal{M}_{i}) - \hat{L}(\mathcal{M}_{i})\right| \geq \epsilon^{\star}/2\right) \\
		&\leq \sum_{j=1}^{m} \mathbb{P}\left(\max_{i \in \mathcal{J}} \sup_{h \in \mathcal{M}_{i}} \left|\hat{L}^{(j)}(h) - L(h)\right| > \epsilon^{\star}/4\right) + \mathbb{P}\left(\max_{i \in \mathcal{J}} \sup_{h \in \mathcal{M}_{i}} \left|L_{\mathcal{D}_{N}^{(j)}}(h) - L(h)\right| > \epsilon^{\star}/8\right).
		\end{align*}
		
		If $\mathcal{M}_{1} \subset \mathcal{M}_{2}$ then, for any $\epsilon > 0$ and $j = 1, \dots, m$, we have the following inclusion of events
		\begin{align*}
		&\left\{\sup_{h \in \mathcal{M}_{1}} \left|\hat{L}^{(j)}(h) - L(h)\right| > \epsilon\right\} \subset \left\{\sup_{h \in \mathcal{M}_{2}} \left|\hat{L}^{(j)}(h) - L(h)\right| > \epsilon\right\}\\
		&\left\{\sup_{h \in \mathcal{M}_{1}} \left|L_{\mathcal{D}_{N}^{(j)}}(h) - L(h)\right| > \epsilon\right\} \subset \left\{\sup_{h \in \mathcal{M}_{2}} \left|L_{\mathcal{D}_{N}^{(j)}}(h) - L(h)\right| > \epsilon\right\},
		\end{align*}
		hence it is true that
		\begin{align*}
		&\left\{\max_{i \in \mathcal{J}} \sup_{h \in \mathcal{M}_{i}} \left|\hat{L}^{(j)}(h) - L(h)\right| > \epsilon^{\star}/4\right\} \subset \left\{\max_{\mathcal{M} \in \text{ Max } \mathbb{L}(\mathcal{H})} \sup_{h \in \mathcal{M}} \left|\hat{L}^{(j)}(h) - L(h)\right| > \epsilon^{\star}/4\right\}\\
		&\left\{\max_{i \in \mathcal{J}} \sup_{h \in \mathcal{M}_{i}} \left|L_{\mathcal{D}_{N}^{(j)}}(h) - L(h)\right| > \epsilon^{\star}/8\right\} \subset \left\{\max_{\mathcal{M} \in \text{ Max } \mathbb{L}(\mathcal{H})} \sup_{h \in \mathcal{M}} \left|L_{\mathcal{D}_{N}^{(j)}}(h) - L(h)\right| > \epsilon^{\star}/8\right\},
		\end{align*}
		from what follows
		\begin{align}
		\label{conlusion_cond} \nonumber
		\mathbb{P}&\left(\max_{i \in \mathcal{J}} \left|L(\mathcal{M}_{i}) - \hat{L}(\mathcal{M}_{i})\right| \geq \epsilon^{\star}/2\right) \\  \nonumber
		&\leq \sum_{j=1}^{m} \sum_{\mathcal{M} \in \text{ Max } \mathbb{L}(\mathcal{H})} \mathbb{P}\Big(\sup_{h \in \mathcal{M}} \left|\hat{L}^{(j)}(h) - L(h)\right| > \epsilon^{\star}/4\Big) + \mathbb{P}\Big(\sup_{h \in \mathcal{M}} \left|L_{\mathcal{D}_{N}^{(j)}}(h) - L(h)\right| > \epsilon^{\star}/8\Big)\\ \nonumber
		&\leq m \sum_{\mathcal{M} \in \text{ Max } \mathbb{L}(\mathcal{H})} \left[\hat{B}_{N,\epsilon^{\star}/4}(d_{VC}(\mathcal{M})) + B_{N,\epsilon^{\star}/8}(d_{VC}(\mathcal{M}))\right]\\
		&\leq m \left|\text{Max } \mathbb{L}(\mathcal{H})\right| \left[\hat{B}_{N,\epsilon^{\star}/4}(d_{VC}(\mathbb{L}(\mathcal{H}))) + B_{N,\epsilon^{\star}/8}(d_{VC}(\mathbb{L}(\mathcal{H})))\right],
		\end{align}
		in which the last inequality follows from the fact that both $\hat{B}_{N,\epsilon^{\star}/4}$ and $B_{N,\epsilon^{\star}/8}$ are increasing functions, and $d_{VC}(\mathbb{L}(\mathcal{H})) = \max_{\mathcal{M} \in \mathbb{L}(\mathcal{H})} d_{VC}(\mathcal{M})$. The result follows from Proposition \ref{proposition_principal} since
		\begin{equation*}
		\{L(\hat{\mathcal{M}}) \neq L(\mathcal{M}^{\star})\} \subset \left\{\max_{i \in \mathcal{J}} \left|L(\mathcal{M}_{i}) - \hat{L}(\mathcal{M}_{i})\right| \geq \epsilon^{\star}/2\right\}.
		\end{equation*}	
		
		If the almost sure convergences \eqref{as_conv} holds, then
		\begin{equation}
		\label{as_proof}
		\hat{L}(\mathcal{M}) \xrightarrow{\text{a.s.}} L(\mathcal{M})
		\end{equation}
		for all $\mathcal{M} \in \mathbb{L}(\mathcal{H})$. Observe that
		\begin{align}
		\label{incl}
		\left\{\max_{\mathcal{M} \in \mathbb{L}(\mathcal{H})} \left|L(\mathcal{M}) - \hat{L}(\mathcal{M})\right| = 0\right\} \subset \left\{\hat{\mathcal{M}} = \mathcal{M}^{\star}\right\},
		\end{align}
		since, if the estimated error $\hat{L}$ is equal to the real error $L$, then the definitions of $\hat{\mathcal{M}}$ and $\mathcal{M}^{\star}$ coincide. As the probability of the event on left hand-side of \eqref{incl} converges to one if \eqref{as_proof} is true, we conclude that, if \eqref{as_conv} holds, then $\hat{\mathcal{M}}$ converges to $\mathcal{M}^{\star}$ with probability one.
	\end{proof}
	
	\begin{proof}[Proof of Theorem \ref{bound_constant}]
		We first note that
		\begin{linenomath}
			\begin{align}
			\label{Sum1} \nonumber
			\mathbb{P}\Big(\sup\limits_{h \in \mathcal{\hat{M}}} \big|L_{\tilde{\mathcal{D}}_{M}}(h) - &L(h) \big| > \epsilon \Big)  = \mathbb{E} \Bigg(\mathbb{P}\Big(\sup\limits_{h \in \mathcal{\hat{M}}} \big|L_{\tilde{\mathcal{D}}_{M}}(h) - L(h) \big| > \epsilon \Big|\mathcal{\hat{M}}\Big)\Bigg)\\ \nonumber
			& = \sum_{i \in \mathcal{J}} \mathbb{P}\Big(\sup\limits_{h \in \mathcal{\hat{M}}} \big|L_{\tilde{\mathcal{D}}_{M}}(h) - L(h) \big| > \epsilon \Big|\mathcal{\hat{M}} = \mathcal{M}_{i}\Big) \mathbb{P}(\mathcal{\hat{M}} = \mathcal{M}_{i})\\
			& = \sum_{i \in \mathcal{J}} \mathbb{P}\Big(\sup\limits_{h \in \mathcal{M}_{i}} \big|L_{\tilde{\mathcal{D}}_{M}}(h) - L(h) \big| > \epsilon \Big|\mathcal{\hat{M}} = \mathcal{M}_{i}\Big) \mathbb{P}(\mathcal{\hat{M}} = \mathcal{M}_{i}).
			\end{align}
		\end{linenomath}
		Fix $\mathcal{M} \in \mathbb{L}(\mathcal{H})$ with $\mathbb{P}(\mathcal{\hat{M}} = \mathcal{M}) > 0$. We claim that
		\begin{linenomath}
			\begin{align}
			\label{cond_independence}
			\mathbb{P}\Big(\sup\limits_{h \in \mathcal{M}} \big|L_{\tilde{\mathcal{D}}_{M}}(h) - L(h) \big| > \epsilon \Big|\mathcal{\hat{M}} = \mathcal{M}\Big) = \mathbb{P}\Big(\sup\limits_{h \in \mathcal{M}} \big|L_{\tilde{\mathcal{D}}_{M}}(h) - L(h) \big| > \epsilon\Big).
			\end{align}
		\end{linenomath}
		Indeed, since $\tilde{\mathcal{D}}_{M}$ is independent of $\mathcal{D}_{N}$, the event $\{\sup_{h \in \mathcal{M}} \big|L_{\tilde{\mathcal{D}}_{M}}(h) - L(h) \big| > \epsilon\}$ is independent of $\{\mathcal{\hat{M}} = \mathcal{M}\}$, as the former depends solely on $\tilde{\mathcal{D}}_{M}$, and the latter solely on $\mathcal{D}_{N}$. Hence, by applying the bound $\eqref{bound_theoremBC}$ for each positive probability in the sum \eqref{Sum1} we obtain that
		\begin{linenomath}
			\begin{align*}
			\mathbb{P}\Big(\sup\limits_{h \in \mathcal{\hat{M}}} \big|L_{\tilde{\mathcal{D}}_{M}}(h) - L(h) \big| > \epsilon \Big) & \leq \sum_{i \in \mathcal{J}} B_{M,\epsilon}^{I}(d_{VC}(\mathcal{M}_{i}))  \mathbb{P}(\mathcal{\hat{M}} = \mathcal{M}_{i})\\
			& = \mathbb{E} \Big(B_{M,\epsilon}^{I}(d_{VC}(\mathcal{\hat{M}}))\Big) \leq B_{N,\epsilon}^{I}(d_{VC}(\mathbb{L}(\mathcal{H}))),
			\end{align*}
		\end{linenomath}
		as desired, in which the last inequality follows from the fact that $B_{M,\epsilon}^{I}$ is an increasing function and $d_{VC}(\mathbb{L}(\mathcal{H})) = \max_{\mathcal{M} \in \mathbb{L}(\mathcal{H})} d_{VC}(\mathcal{M})$.
		
		The bound for type II estimation error may be obtained similarly, since
		\begin{linenomath}
			\begin{align*}
			\mathbb{P}\Big(L(\hat{h}_{\mathcal{\hat{M}}}(\tilde{\mathcal{D}}_{M})) - L(h^{\star}_{\mathcal{\hat{M}}}) > \epsilon \Big) &= \mathbb{E} \Bigg(\mathbb{P}\Big(L(\hat{h}_{\mathcal{\hat{M}}}(\tilde{\mathcal{D}}_{M})) - L(h^{\star}_{\mathcal{\hat{M}}}) > \epsilon \Big| \mathcal{\hat{M}} \Big)\Bigg)\\
			& = \sum_{i \in \mathcal{J}} \mathbb{P}\Big(L(\hat{h}_{\mathcal{\hat{M}}}(\tilde{\mathcal{D}}_{M})) - L(h^{\star}_{\mathcal{\hat{M}}}) > \epsilon \Big| \mathcal{\hat{M}} = \mathcal{M}_{i} \Big) \mathbb{P}(\mathcal{\hat{M}} = \mathcal{M}_{i})\\
			& = \sum_{i \in \mathcal{J}} \mathbb{P}\Big(L(\hat{h}_{\mathcal{M}_{i}}(\tilde{\mathcal{D}}_{M})) - L(h^{\star}_{\mathcal{M}_{i}}) > \epsilon \Big| \mathcal{\hat{M}} = \mathcal{M}_{i} \Big) \mathbb{P}(\mathcal{\hat{M}} = \mathcal{M}_{i})\\
			& = \sum_{i \in \mathcal{J}} \mathbb{P}\Big(L(\hat{h}_{\mathcal{M}_{i}}(\tilde{\mathcal{D}}_{M})) - L(h^{\star}_{\mathcal{M}_{i}}) > \epsilon \Big) \mathbb{P}(\mathcal{\hat{M}} = \mathcal{M}_{i})
			\end{align*}
		\end{linenomath}
		and $B^{II}_{M,\epsilon}(d_{VC}(\mathcal{M}_{i}))$ is a bound for the probabilities inside the sum by \eqref{bound_theoremBC}. The assertion that types I and II estimation errors are asymptotically zero when $d_{VC}(\mathbb{L}(\mathcal{H})) < \infty$ is immediate from the established bounds.
	\end{proof}
	
	\begin{proof}[Proof of Theorem \ref{theorem_tipeIII}]
		We first show that
		\begin{align}
		\label{lemma_inside}
		\mathbb{P}\Big(\sup\limits_{i \in \mathcal{J}} \left|\hat{L}(\mathcal{M}_{i}) - L(\mathcal{M}_{i})\right| > (\epsilon \vee \epsilon^{\star})/2 \Big) \geq \mathbb{P}\left(L(h_{\hat{\mathcal{M}}}^{\star}) - L(h^{\star}) > \epsilon\right).
		\end{align}
		If $\epsilon \leq \epsilon^{\star}$ then, by the inclusion of events \eqref{inclusion_star} in the proof of Proposition \ref{proposition_principal} we have that
		\begin{align}
		\label{incl1}
		\left\{\sup\limits_{i \in \mathcal{J}} \left|\hat{L}(\mathcal{M}_{i}) - L(\mathcal{M}_{i})\right| < (\epsilon \vee \epsilon^{\star})/2\right\} \subset \left\{L(\hat{\mathcal{M}}) = L(\mathcal{M}^{\star})\right\} \subset \left\{L(h_{\hat{\mathcal{M}}}^{\star}) - L(h^{\star}) < \epsilon\right\},
		\end{align}
		since $L(h^{\star}_{\hat{\mathcal{M}}}) = L(\hat{\mathcal{M}})$ and $L(h^{\star}_{\mathcal{M}^{\star}}) = L(\mathcal{M}^{\star})$, so \eqref{lemma_inside} follows in this case.
		
		Now, if $\epsilon > \epsilon^{\star}$ and $\sup_{i \in \mathcal{J}} \left|\hat{L}(\mathcal{M}_{i}) - L(\mathcal{M}_{i})\right| < \epsilon/2$, then
		\begin{align*}
		L(\hat{\mathcal{M}}) - L(\mathcal{M}^{\star}) &= [L(\hat{\mathcal{M}}) - \hat{L}(\mathcal{M}^{\star})] - [L(\mathcal{M}^{\star}) - \hat{L}(\mathcal{M}^{\star})]\\
		&\leq [L(\hat{\mathcal{M}}) - \hat{L}(\hat{\mathcal{M}})] - [L(\mathcal{M}^{\star}) - \hat{L}(\mathcal{M}^{\star})]\\
		&\leq \epsilon/2 + \epsilon/2 = \epsilon,
		\end{align*}
		in which the first inequality follows from the fact that the minimum of $\hat{L}$ is attained at $\hat{\mathcal{M}}$, and the last inequality follows from $\sup_{i \in \mathcal{J}} \left|\hat{L}(\mathcal{M}_{i}) - L(\mathcal{M}_{i})\right| < \epsilon/2$. Since $L(\hat{\mathcal{M}}) - L(\mathcal{M}^{\star}) = L(h_{\hat{\mathcal{M}}}^{\star}) - L(h^{\star})$ we also have the inclusion of events
		\begin{align}
		\label{incl2}
		\left\{\sup\limits_{i \in \mathcal{J}} \left|\hat{L}(\mathcal{M}_{i}) - L(\mathcal{M}_{i})\right| < (\epsilon \vee \epsilon^{\star})/2\right\} \subset \left\{L(h_{\hat{\mathcal{M}}}^{\star}) - L(h^{\star}) < \epsilon\right\},
		\end{align}
		when $\epsilon > \epsilon^{\star}$. From \eqref{incl1} and \eqref{incl2} follows \eqref{lemma_inside} as desired.
		
		Substituting $\epsilon^{\star}$ by $\epsilon \vee \epsilon^{\star}$ in \eqref{conlusion_cond} we obtain
		\begin{align}
		\label{conclusion2} \nonumber
		\mathbb{P}\Big(\max_{i \in \mathcal{J}} &\left|L(\mathcal{M}_{i}) - \hat{L}(\mathcal{M}_{i})\right| \geq (\epsilon \vee \epsilon^{\star})/2\Big) \leq\\
		& m \left|\text{Max } \mathbb{L}(\mathcal{H})\right| \left[\hat{B}_{N,(\epsilon \vee \epsilon^{\star})/4}(d_{VC}(\mathbb{L}(\mathcal{H}))) + B_{N,(\epsilon \vee \epsilon^{\star})/8}(d_{VC}(\mathbb{L}(\mathcal{H})))\right].
		\end{align}
		The result follows combining \eqref{lemma_inside} and \eqref{conclusion2}.	
	\end{proof}
	
	\begin{proof}[Proof of Corollary \ref{corollary}]
		If $\hat{L}(\mathcal{M})$ is estimated by a validation sample or via k-fold cross-validation, the result follows from Theorems \ref{theorem_principal_convergence}, \ref{bound_constant} and \ref{theorem_tipeIII}, and Corollary \ref{cor_typeIV}, since the bounds $B_{N,\epsilon}, \hat{B}_{N,\epsilon}, B^{I}_{M,\epsilon}$ and $B^{II}_{M,\epsilon}$ follow from classical VC Theory applied to the independent training and validation samples, and independent sample $\tilde{\mathcal{D}}_{M}$ \cite{devroye1996,vapnik1998,vapnik2000}.
	\end{proof}
	
	\begin{proof}[Proof of Theorem \ref{bound_constant_reusing}]
		The bound for type I estimation error follows from the inequality
		\begin{align*}
		&\mathbb{P}\Big(\sup\limits_{h \in \hat{\mathcal{M}}} \left|L_{\mathcal{D}_{N}}(h) - L(h)\right| > \epsilon\Big) \\
		& = \mathbb{P}\Big(\sup\limits_{h \in \mathcal{M}^{\star}} \left|L_{\mathcal{D}_{N}}(h) - L(h)\right| > \epsilon,\hat{\mathcal{M}} = \mathcal{M}^{\star}\Big) + \mathbb{P}\Big(\sup\limits_{h \in \hat{\mathcal{M}}} \left|L_{\mathcal{D}_{N}}(h) - L(h)\right| > \epsilon,\hat{\mathcal{M}} \neq \mathcal{M}^{\star}\Big)\\
		&\leq \mathbb{P}\Big(\sup\limits_{h \in \mathcal{M}^{\star}} \left|L_{\mathcal{D}_{N}}(h) - L(h)\right| > \epsilon\Big) + \mathbb{P}\left(\hat{\mathcal{M}} \neq \mathcal{M}^{\star}\right)
		\end{align*}
		by noting that $B_{N,\epsilon}^{I}(d_{VC}(\mathcal{M}^{\star}))$ is a bound for the first probability. With a similar argument, we have the bound for type II estimation error.
	\end{proof}
	
	\begin{proof}[Proof of Corollary \ref{corollary_reuse}]
		If $\hat{L}(\mathcal{M})$ is estimated by a validation sample or via k-fold cross-validation, the result follows from Theorems \ref{theorem_principal_convergence}, \ref{theorem_tipeIII} and \ref{bound_constant_reusing}, and inequality \eqref{triangle}, since the bounds $B_{N,\epsilon}, \hat{B}_{N,\epsilon}, B^{I}_{N,\epsilon}$ and $B^{II}_{N,\epsilon}$ follow from classical VC Theory applied to the independent training and validation samples, and the whole sample $\mathcal{D}_{N}$ \cite{devroye1996,vapnik1998,vapnik2000}.
	\end{proof}

	\begin{proof}[Proof of Proposition \ref{partition_has_ucurve}]
		We first consider the case $m = 1$, that is when the sample is split into a training and validation sample. We start by establishing that if $\mathcal{H}|_{\pi}$ is a strong local minimum of $\mathbb{L}(\mathcal{H})$ and if $\pi \leq \pi_{i}$ then $\hat{L}(\hat{h}_{\pi}) \leq \hat{L}(\hat{h}_{i})$, in which $\hat{h}_{\pi}$ and $\hat{h}_{i}$ are the ERM hypothesis of $\mathcal{H}|_{\pi}$ and $\mathcal{H}|_{\pi_{i}}$ under the training sample, and $\hat{L}$ is the empirical error under the validation sample. Note that for all $\pi_{j} \leq \pi_{i}$ we have 
		\begin{linenomath}
			\begin{equation*}
			\hat{h}_{j}(x) = \hat{h}_{i}(x) \text{ for all } x \in \bigcup\limits_{a \in \pi_{j} \cap \pi_{i}} a.
			\end{equation*}
		\end{linenomath}
		This is the case because, if $a \in \pi_{j}$ and $b \in \pi_{i}$, then for all $x_{1} \in a$ and $x_{2} \in b$
		\begin{linenomath}
			\begin{align}
			\label{def_h_hat}
			\hat{h}_{j}(x_{1}) = \argmaxA\limits_{y \in \{0,1\}} \sum_{k=1}^{N - V_{N}} \frac{\mathds{1}\{Y_{k} = y,X_{k} \in a\}}{\sum_{k=1}^{N - V_{N}} \mathds{1}\{X_{k} \in a\}} & & \hat{h}_{i}(x_{2}) = \argmaxA\limits_{y \in \{0,1\}} \sum_{k=1}^{N - V_{N}} \frac{\mathds{1}\{Y_{k} = y,X_{k} \in b\}}{\sum_{k=1}^{N - V_{N}} \mathds{1}\{X_{k} \in b\}},
			\end{align}
		\end{linenomath}
		so that if $a = b$, then $\hat{h}_{j}(x) = \hat{h}_{i}(x)$ for all $x \in a = b$. Furthermore, if $\pi_{j} \leq \pi_{i}$ and $|\pi_{j}| = |\pi_{i}| - 1$ then
		\begin{linenomath}
			\begin{align*}
			\pi_{i} \setminus \big(\pi_{j} \cap \pi_{i}\big) = \{a_{1},a_{2}\} & & \pi_{j} = \big(\pi_{j} \cap \pi_{i}\big) \cup \{a_{1} \cup a_{2}\},
			\end{align*}
		\end{linenomath}	
		as $\pi_{i}$ is obtained from $\pi_{j}$ by partitioning a block of it into two blocks $a_{1}, a_{2}$. From (\ref{def_h_hat}) we can establish that
		\begin{linenomath}
			\begin{align*}
			\hat{h}_{i}(x) = \hat{h}_{j}(x), & & \text{ for all } x \in \bigcup\limits_{a \in \pi_{j} \cap \pi_{i}} a \cup a_{1} \text{ or for all } x \in \bigcup\limits_{a \in \pi_{j} \cap \pi_{i}} a \cup a_{2}.
			\end{align*}
		\end{linenomath}	
		Indeed, if 
		\begin{linenomath}
			\begin{align}
			\label{ystar}
			y^{\star} \coloneqq \argmaxA\limits_{y \in \{0,1\}} \sum_{k=1}^{N - V_{N}} \frac{\mathds{1}\{Y_{k} = y,X_{k} \in a_{1} \cup a_{2}\}}{\sum_{k=1}^{N - V_{N}} \mathds{1}\{X_{k} \in a_{1} \cup a_{2}\}}	
			\end{align}
		\end{linenomath}	
		then either
		\begin{linenomath}
			\begin{align}
			\label{ystar_1} \nonumber
			y^{\star}_{1} &\coloneqq \argmaxA\limits_{y \in \{0,1\}} \sum_{k=1}^{N - V_{N}} \frac{\mathds{1}\{Y_{k} = y,X_{k} \in a_{1}\}}{\sum_{k=1}^{N - V_{N}} \mathds{1}\{X_{k} \in a_{1}\}} = y^{\star} \text{ or} \\
			y^{\star}_{2} &\coloneqq \argmaxA\limits_{y \in \{0,1\}} \sum_{k=1}^{N - V_{N}} \frac{\mathds{1}\{Y_{k} = y,X_{k} \in a_{2}\}}{\sum_{k=1}^{N - V_{N}} \mathds{1}\{X_{k} \in a_{2}\}} = y^{\star}
			\end{align}
		\end{linenomath}
		as the ratio in (\ref{ystar}) with $y = y^{\star}$ is a weighted mean of the ratios in (\ref{ystar_1}) with $y = y^{\star}$, so that if it is greater than $1/2$, as is the case when $y = y^{\star}$, the maximum of the ratios in (\ref{ystar_1}) is greater than $1/2$ when $y = y^{\star}$, as the maximum is not lesser than the weighted mean. This establishes that at least one of the $y^{\star}_{1}, y^{\star}_{2}$ is equal to $y^{\star}$.
		
		Suppose that $\mathcal{H}|_{\pi}, \pi \in \mathcal{F}$, is a strong local minimum of $\mathbb{L}(\mathcal{H})$. Then, if $\pi \leq \pi_{i}$ and $|\pi| = |\pi_{i}| - 1$, denoting $A = \bigcup\limits_{a \in \pi \cap \pi_{i}} a$, we have that
		\begin{linenomath}
			\begin{align}
			\label{minimum_integral}
			\nonumber
			\hat{L}(\hat{h}_{\pi}) \leq \hat{L}(\hat{h}_{i}) & = \int_{A \times \{0,1\}} \ell(\hat{h}_{i}(x),y) \ d\hat{P}(x,y) + \int_{A^{c} \times \{0,1\}} \ell(\hat{h}_{i}(x),y) \ d\hat{P}(x,y) \\
			& = \int_{A \times \{0,1\}} \ell(\hat{h}_{\pi}(x),y) \ d\hat{P}(x,y) + \int_{A^{c} \times \{0,1\}} \ell(\hat{h}_{i}(x),y) \ d\hat{P}(x,y)
			\end{align}
		\end{linenomath}
		as $\hat{h}_{i}(x) = \hat{h}_{\pi}(x)$ for $x \in A$, in which $\hat{P}$ is the empirical measure of the validation sample. We have that $A^{c} = a_{1} \cup a_{2}$, with $a_{1} \cap a_{2} = \emptyset$, $a_{1}, a_{2} \in \pi_{i}$, and 
		\begin{linenomath}
			\begin{align}
			\label{partition_integral}
			\int_{A^{c} \times \{0,1\}} &\ell(\hat{h}_{i}(x),y) \ d\hat{P}(x,y) = \int_{a_{1} \times \{0,1\}} \ell(\hat{h}_{i}(x),y) \ d\hat{P}(x,y) + \int_{a_{2} \times \{0,1\}} \ell(\hat{h}_{i}(x),y) \ d\hat{P}(x,y) 
			\end{align}
		\end{linenomath}
		so that, as $\mathcal{H}|_{\pi}$ is a strong local minimum, by substituting (\ref{partition_integral}) in (\ref{minimum_integral}), we have 
		\begin{linenomath}
			\begin{align} 
			\nonumber
			\label{condition_ucurve_partition}
			\int_{a_{1} \times \{0,1\}} \ell(\hat{h}_{i}(x),y) \ d\hat{P}(x,y) & \geq \int_{a_{1} \times \{0,1\}} \ell(\hat{h}_{\pi}(x),y) \ d\hat{P}(x,y) \\
			\int_{a_{2} \times \{0,1\}} \ell(\hat{h}_{i}(x),y) \ d\hat{P}(x,y) & \geq \int_{a_{2} \times \{0,1\}} \ell(\hat{h}_{\pi}(x),y) \ d\hat{P}(x,y)
			\end{align}	
		\end{linenomath}
		with equality holding for at least one of the two inequalities by (\ref{ystar_1}). 
		
		As $\mathcal{H}|_{\pi}$ is a strong local minimum, condition (\ref{condition_ucurve_partition}) holds for any $a_{1} \subset b \in \pi$ by taking $a_{2} = b \setminus a_{1}$. Indeed, let $a_{1} \subset b \in \pi$ be arbitrary and consider $\pi^{\star} = (\pi \setminus \{b\}) \cup \{a_{1},b \setminus a_{1}\}$. Then clearly $\pi \leq \pi^{\star}$ and $|\pi| = |\pi^{\star}| - 1$, so (\ref{condition_ucurve_partition}) follows.
		To end the proof, we note that if $\pi \leq \pi_{i}$ then
		\begin{linenomath}
			\begin{align*}
			&\hat{L}(\hat{h}_{i})  = \int_{A \times \{0,1\}} \ell(\hat{h}_{\pi}(x),y) \ d\hat{P}(x,y) + \sum_{i=1}^{p} \int_{a_{i} \times \{0,1\}} \ell(\hat{h}_{i}(x),y) \ d\hat{P}(x,y) \\
			& \geq \int_{A \times \{0,1\}} \ell(\hat{h}_{\pi}(x),y) \ d\hat{P}(x,y) + \sum_{i=1}^{p} \int_{a_{i} \times \{0,1\}} \ell(\hat{h}_{\pi}(x),y) \ d\hat{P}(x,y) = \hat{L}(\hat{h}_{\pi})
			\end{align*}
		\end{linenomath}
		in which the $a_{1}, \dots, a_{p}$ is a partition of $A^{c}$, as for all $a_{i}$ there exists a $b_{i} \in \pi$ such that $a_{i} \subset b_{i}$, which follows from the fact that $\pi \leq \pi_{i}$, so we may apply inequality (\ref{condition_ucurve_partition}) to each parcel of the sum.
		
		To establish that $\mathcal{H}|_{\pi}$ is a global minimum of every continuous chain that contains it, it is enough to show that $\pi_{i} \leq \pi$ implies $\hat{L}(\hat{h}_{i}) \geq \hat{L}(\hat{h}_{\pi})$. This can be done analogously by supposing that there exists $a_{1}, a_{2} \in \pi$ and $a_{1} \cup a_{2} \in \pi_{i}$, when $|\pi| = |\pi_{i}| + 1$, and proceeding similarly.
		
		To show the result for $m > 1$ we may repeat the proof above to each term in sum \eqref{form_Lhat}, since they each represent a pair of independent training and validation samples, to establish that if $H|_{\pi}$ is a strong local minimum then
		\begin{align}
		\label{faj}
		\hat{L}^{(j)}(\hat{h}^{(j)}_{\pi}) \leq \hat{L}^{(j)}(\hat{h}^{(j)}_{i}), & & \forall j = 1, \dots, m
		\end{align}
		for any $\pi_{i}$ such that $\pi \leq \pi_{i}$ or $\pi_{i} \leq \pi$. Summing \eqref{faj} for $j$ from $1$ to $m$ and dividing by $m$ we have the result.
	\end{proof}

	\begin{proof}[Proof of Theorem \ref{general_theorem}]
		Let $\mathcal{M} \in \mathbb{L}(\mathcal{H})$ be a strong local minimum. We first show that if $\mathcal{M}_{i}, \mathcal{M}_{j} \in \mathcal{C}^{+}(\mathcal{M}), \mathcal{M}_{i} \subset \mathcal{M}_{j},$ then $\hat{L}(\mathcal{M}_{i}) \leq \hat{L}(\mathcal{M}_{j})$, which implies that $\mathcal{M}$ is a global minimum of $\mathcal{C}^{+}(\mathcal{M})$, as it is its least element.	Let $k^{\star}$ be the size of the greatest continuous chain in $\mathcal{C}^{+}(\mathcal{M})$ which contains $\mathcal{M}$ and define for $\mathcal{M}_{1} \subset \mathcal{M}_{2} \in \mathbb{L}(\mathcal{H})$
		\begin{linenomath}
			\begin{equation*}
			D(\mathcal{M}_{1},\mathcal{M}_{2}) = \max\Big\{\text{Length of continuous chain starting in } \mathcal{M}_{1} \text{ and ending in } \mathcal{M}_{2}\Big\}.
			\end{equation*}
		\end{linenomath}
		We may partition $\mathcal{C}^{+}(\mathcal{M})$ by the greatest position each subset occupies in a continuous chain starting in $\mathcal{M}$:
		\begin{linenomath}
			\begin{align*}
			\mathcal{C}^{+}(\mathcal{M},k) \nonumber = \Big\{&\mathcal{M}_{i} \in \mathcal{C}^{+}(\mathcal{M}): D(\mathcal{M}_{i},\mathcal{M}) = k\Big\}
			\end{align*}
		\end{linenomath}
		for $2 \leq k \leq k^{\star}$ and $\mathcal{C}^{+}(\mathcal{M},1) = \{\mathcal{M}\}$. 
		
		By hypothesis, $\hat{L}(\mathcal{M}) \leq \hat{L}(\mathcal{M}_{i})$ for all $\mathcal{M}_{i} \in \mathcal{C}^{+}(\mathcal{M},2)$ as $\mathcal{M}$ is a strong local minimum. We proceed by induction. Assume, for a $k \geq 3$, that $\hat{L}(\mathcal{M}_{j}) \leq \hat{L}(\mathcal{M}_{i})$ for all $\mathcal{M}_{i}, \mathcal{M}_{j} \in \cup_{l=1}^{k-1} \mathcal{C}^{+}(\mathcal{M},l)$ when $\mathcal{M}_{j} \subset \mathcal{M}_{i}$. Fix $\mathcal{M}_{j} \in \mathcal{C}^{+}(\mathcal{M},k)$ and $\mathcal{M}_{j}^{(1)} \in \mathcal{C}^{+}(\mathcal{M},k-1)$ with $\mathcal{M}_{j}^{(1)} \subset \mathcal{M}_{j}$. Note that $\mathcal{M}_{j}^{(1)} \in N^{+}(\mathcal{M}_{j}) \subset \cup_{l=1}^{k-1} \mathcal{C}^{+}(\mathcal{M},l)$ and, as $|N^{+}(\mathcal{M}_{j})| \geq 2$, since $\mathbb{L}(\mathcal{H})$ is U-curve compatible, there exists another $\mathcal{M}_{j}^{(2)} \in N^{+}(\mathcal{M}_{j})$ such that $\mathcal{M}_{j} = \mathcal{M}_{j}^{(1)} \vee \mathcal{M}_{j}^{(2)}$. Therefore,
		\begin{linenomath}
			\begin{align}
			\label{inequality_union}
			\hat{L}(\mathcal{M}_{j}) &\geq \hat{L}(\mathcal{M}_{j}^{(1)}) + \hat{L}(\mathcal{M}_{j}^{(2)}) - \hat{L}(\mathcal{M}_{j}^{(1)} \wedge \mathcal{M}_{j}^{(2)}) \geq \hat{L}(\mathcal{M}_{j}^{(1)})
			\end{align}
		\end{linenomath}
		by hypothesis, as $\hat{L}(\mathcal{M}_{j}^{(2)}) - \hat{L}(\mathcal{M}_{j}^{(1)} \wedge \mathcal{M}_{j}^{(2)}) \geq 0$ since $\mathcal{M}_{j}^{(1)} \wedge \mathcal{M}_{j}^{(2)} \subset \mathcal{M}_{j}^{(2)}$ and both are in $\cup_{l=1}^{k-1}
		\mathcal{C}^{+}(\mathcal{M},l)$. From (\ref{inequality_union}) and the induction hypothesis it follows that $\hat{L}(\mathcal{M}_{i}) \leq \hat{L}(\mathcal{M}_{j})$ for all $\mathcal{M}_{i}, \mathcal{M}_{j} \in \mathcal{C}^{+}(\mathcal{M}), \mathcal{M}_{i} \subset \mathcal{M}_{j},$ as there is a $\mathcal{M}_{j}^{(1)} \in N^{+}(\mathcal{M}_{j})$ such that $\mathcal{M}_{i} \subset \mathcal{M}_{j}^{(1)} \subset \mathcal{M}_{j}$ which implies $\hat{L}(\mathcal{M}_{i}) \leq \hat{L}(\mathcal{M}_{j}^{(1)}) \leq \hat{L}(\mathcal{M}_{j})$.
		
		With an analogous deduction and the inequality
		\begin{linenomath}
			\begin{align*}
			\hat{L}(\mathcal{M}_{1} \wedge \mathcal{M}_{2}) \geq \hat{L}(\mathcal{M}_{1}) + \hat{L}(\mathcal{M}_{2}) - \hat{L}(\mathcal{M}_{1} \vee \mathcal{M}_{2})
			\end{align*}
		\end{linenomath}
		we can show that if $\mathcal{M}_{i}, \mathcal{M}_{j} \in \mathcal{C}^{-}(\mathcal{M}), \mathcal{M}_{i} \subset \mathcal{M}_{j},$ then $\hat{L}(\mathcal{M}_{i}) \geq \hat{L}(\mathcal{M}_{j})$, which implies that $\mathcal{M}$ is also the global minimum of $\mathcal{C}^{-}(\mathcal{M})$, as it is its greatest element.
	\end{proof}
	
	\begin{remark}
		\normalfont
		From the proof of Theorem \ref{general_theorem} one can deduce that its premises may be loosened. Condition (\ref{condS}) need not be satisfied by \textit{all} $\mathcal{M}_{1}, \mathcal{M}_{2} \in \mathbb{L}(\mathcal{H})$ such that $d(\mathcal{M}_{i},\mathcal{M}_{1} \wedge \mathcal{M}_{2}) = d(\mathcal{M}_{i},\mathcal{M}_{1} \vee \mathcal{M}_{2}) = 1, i = 1,2$. It is necessary only that for every $\mathcal{M} \in \mathbb{L}(\mathcal{H})$ and every $\mathcal{M}_{1} \in \mathcal{C}^{+}(\mathcal{M})\setminus\{\mathcal{M}\} \ (\mathcal{C}^{-}(\mathcal{M})\setminus\{\mathcal{M}\})$ with $d(\mathcal{M},\mathcal{M}_{1}) > 1$, there are $\mathcal{M}_{1}^{(i)} \in \mathcal{C}^{+}(\mathcal{M})\setminus\{\mathcal{M}\} \ (\mathcal{C}^{-}(\mathcal{M})\setminus\{\mathcal{M}\}), i = 1,2$, such that
		$\mathcal{M}_{1} = \mathcal{M}_{1}^{(1)} \vee \mathcal{M}_{1}^{(2)} (= \mathcal{M}_{1}^{(1)} \wedge \mathcal{M}_{1}^{(2)})$, and condition (\ref{condS}) is satisfied by $\mathcal{M}_{1}^{(1)}, \mathcal{M}_{1}^{(2)}$. Moreover, Theorem \ref{general_theorem} may be adapted to $\mathbb{L}(\mathcal{H})$ which is not a lattice by adding further constraints to it.
	\end{remark}
	
	\section{Partition Lattice and Lattice Convexity} 
	\label{counter_part}
	
	In this section, we discuss why condition \eqref{condS} is not satisfied by the Partition Lattice Learning Space. As an example of a pair which satisfies the condition, consider the one generated by $\pi_{1} = \{a_{1},a_{2},b,c\}$ and $\pi_{2} = \{a,b_{1},b_{2},c\}$ in which $a = a_{1} \cup a_{2}$ and $b = b_{1} \cup b_{2}$. Now, to get from $\pi_{1} \wedge \pi_{2} = \{a,b,c\}$ to $\pi_{1} \vee \pi_{2} = \{a_{1},a_{2},b_{1},b_{2},c\}$ in this case we perform two partition breaks, each one in distinct partitions and increasing the estimated out-of-sample error by $l_{1}, l_{2} \in \mathbb{R}$, independently of the order of such breaks, by the argument at inequality (\ref{condition_ucurve_partition}). From this, we may conclude that, say,
	\begin{linenomath}
		\begin{align*}
		\hat{L}(\mathcal{H}|_{\pi_{1}}) = \hat{L}(\mathcal{H}|_{\pi_{1} \wedge \pi_{2}}) + l_{1} & & \hat{L}(\mathcal{H}|_{\pi_{2}}) = \hat{L}(\mathcal{H}|_{\pi_{1} \wedge \pi_{2}}) + l_{2} 
		\end{align*}
	\end{linenomath}
	and
	\begin{linenomath}
		\begin{align*}
		\hat{L}(\mathcal{H}|_{\pi_{1} \vee \pi_{2}}) = \hat{L}(\mathcal{H}|_{\pi_{1}}) + l_{2} = \hat{L}(\mathcal{H}|_{\pi_{2}}) + l_{1} = \hat{L}(\mathcal{H}|_{\pi_{1} \wedge \pi_{2}}) + l_{1} + l_{2}.
		\end{align*}
	\end{linenomath}
	Hence
	\begin{linenomath}
		\begin{align*}
		\hat{L}(\mathcal{H}|_{\pi_{1} \wedge \pi_{2}}) + \hat{L}(\mathcal{H}|_{\pi_{1} \vee \pi_{2}}) &= \hat{L}(\mathcal{H}|_{\pi_{1}}) + \hat{L}(\mathcal{H}|_{\pi_{2}})
		\end{align*}
	\end{linenomath}
	which implies the condition of Theorem \ref{condS}.
	
	On the other hand, consider partitions $\pi_{1} = \{a_{1} \cup a_{2},a_{3},b\}$ and $\pi_{2} = \{a_{1}\cup a_{3},a_{2},b\}$, and denote $a = a_{1} \cup a_{2} \cup a_{3}$. To get from $\pi_{1} \wedge \pi_{2} = \{a,b\}$ to $\pi_{1} \vee \pi_{2} = \{a_{1},a_{2},a_{3},b\}$ we perform two partition breaks on the same root partition block $a$. In this case, the order in which $a$ is broken may influence the increment of the estimated out-of-sample error and the condition of Theorem \ref{partition_has_ucurve} may be violated. Consider the empirical distributions obtained from a training and validation samples, presented in Table \ref{tabEmp}. In this example,
	\begin{align*}
	\begin{cases}
	\hat{L}(\mathcal{H}|_{\pi_{1} \wedge \pi_{2}}) = 0.48\\
	\hat{L}(\mathcal{H}|_{\pi_{1}}) = 0.44 \\
	\hat{L}(\mathcal{H}|_{\pi_{2}}) = 0.41 \\  
	\hat{L}(\mathcal{H}|_{\pi_{1} \vee \pi_{2}}) = 0.33\\
	\end{cases} & \implies & 
	\begin{cases}
	\hat{L}(\mathcal{H}|_{\pi_{1} \wedge \pi_{2}}) + \hat{L}(\mathcal{H}|_{\pi_{1} \vee \pi_{2}}) = 0.81  \\
	\hat{L}(\mathcal{H}|_{\pi_{1}}) +  \hat{L}(\mathcal{H}|_{\pi_{2}}) = 0.85
	\end{cases}
	\end{align*}
	so condition \eqref{condS} is not satisfied.
	
	When partition $\{a_{1},a_{2},a_{2}\}$ is broken into $\{a_{1},a_{3}\},a_{2}$ the error decreases $0.07$ while when we perform the break of partition $\{a_{1},a_{2}\}$ into $a_{1},a_{2}$ it decreases $0.11$ hence the variation of the error when we perform a break to free point $a_{2}$ depends on the order we perform the break, that is, if we free $a_{3}$ before or after freeing $a_{2}$. 
	
	\begin{table}[H]
		\centering
		\caption{\footnotesize Empirical distributions of a training and validation samples when $\mathcal{X} = \{a_{1},a_{2},a_{3},b\}$ is a set with four points, and estimated hypothesis for partitions $\pi_{1} \wedge \pi_{2} = \{\{a_{1},a_{2},a_{3}\},b\}, \pi_{1} = \{\{a_{1},a_{2}\},a_{3},b\}, \pi_{2} = \{\{a_{1},a_{3}\},a_{2},b\}$ and $\pi_{1} \vee \pi_{2} = \{a_{1},a_{2},a_{3},b\}$.} \label{tabEmp}
		\begin{tabular}{|c|cc|cc|cccc|}
			\hline
			\multirow{2}{*}{$\mathcal{X}$} & \multicolumn{2}{c|}{Training} & \multicolumn{2}{c|}{Validation} & \multirow{2}{*}{$\hat{h}_{\pi_{1} \wedge \pi_{2}}$} & \multirow{2}{*}{$\hat{h}_{\pi_{1}}$} & \multirow{2}{*}{$\hat{h}_{\pi_{2}}$} & \multirow{2}{*}{$\hat{h}_{\pi_{1} \vee \pi_{2}}$} \\ \cline{2-5}
			& 0 & 1 & 0 & 1 & & & & \\
			\hline
			$a_{1}$ & 0.18 & 0.07 & 0.2 & 0.05 & 1 & 0 & 0 & 0\\
			$a_{2}$ & 0.02 & 0.11 & 0.01 & 0.12 & 1 & 0 & 1 & 1\\
			$a_{3}$ & 0.01 & 0.11 & 0.02 & 0.10 & 1 & 1 & 0 & 1\\
			$b$ & 0.2 & 0.3 & 0.25 & 0.25 & 1 & 1 & 1 & 1 \\
			\hline
		\end{tabular}	
	\end{table}

	\FloatBarrier
	\vskip 0.2in
	\bibliographystyle{plain}
	\bibliography{ref}

\end{document}